\def\year{2021}\relax
\documentclass[letterpaper]{article} % DO NOT CHANGE THIS
\usepackage{aaai21}  % DO NOT CHANGE THIS
\usepackage{times}  % DO NOT CHANGE THIS
\usepackage{helvet} % DO NOT CHANGE THIS
\usepackage{courier}  % DO NOT CHANGE THIS
\usepackage[hyphens]{url}  % DO NOT CHANGE THIS
\usepackage{graphicx} % DO NOT CHANGE THIS
\usepackage{natbib}  % DO NOT CHANGE THIS AND DO NOT ADD ANY OPTIONS TO IT
\usepackage{caption} % DO NOT CHANGE THIS AND DO NOT ADD ANY OPTIONS TO IT
\usepackage{booktabs}       % professional-quality tables
\usepackage{amsfonts}       % blackboard math symbols
\usepackage{nicefrac}       % compact symbols for 1/2, etc.
\usepackage{microtype}      % microtypography
\usepackage{color} 

\usepackage[switch]{lineno}

\usepackage{epsfig}
\usepackage{amsmath}
\usepackage{amssymb}
\usepackage{javen}
\usepackage{bm}
\usepackage{makecell}
\usepackage{multirow}
\usepackage{algorithm}
\usepackage{algorithmic}
\usepackage{adjustbox}
\usepackage{bbm}
\usepackage{epstopdf}
\usepackage{threeparttable}
\usepackage{tablefootnote}
\usepackage[normalem]{ulem}
\usepackage[labelformat=simple]{subcaption}
\usepackage{wrapfig}

\DeclareMathAlphabet{\mathcal}{OMS}{cmsy}{m}{n}

\usepackage{mathtools}

\newcommand{\benr}{\begin{eqnarray}}
\newcommand{\eenr}{\end{eqnarray}}
\newcommand{\benrr}{\begin{eqnarray*}}
\newcommand{\eenrr}{\end{eqnarray*}}
\newcommand{\ben}{\begin{equation}}
\newcommand{\een}{\end{equation}}
\newcommand{\benn}{\begin{equation*}}
\newcommand{\eenn}{\end{equation*}}

\newcommand{\noi}{\noindent}
\newcommand{\nn}{\nonumber}
\newcommand{\vs}{\vskip}
\newcommand{\T}{ {\mathrm{\scriptscriptstyle T}} }

\newcommand{\vep}{\varepsilon}

\newcommand{\bw}{{\bf w}}
\newcommand{\bv}{{\bf v}}

\newcommand{\bA}{{\bf A}}

\newcommand{\bth}{\bm{\theta}}

\newcommand{\cD}{\mathcal D}
\newcommand{\cL}{\mathcal L}
\newcommand{\cV}{P}

\newcommand{\cT}{\mathcal T}

\newcommand{\bbE}{{\mathbb E}}

\title{MetaAugment: Sample-Aware Data Augmentation Policy Learning}
\author{
    Fengwei Zhou\thanks{Equal Contribution}, 
    Jiawei Li\footnotemark[1],
    Chuanlong Xie\footnotemark[1], 
    Fei Chen, 
    Lanqing Hong, 
    Rui Sun, 
    Zhenguo Li\thanks{Corresponding Author}
    \\
}
\affiliations{
    Huawei Noah's Ark Lab\\
    \{zhoufengwei, li.jiawei, xie.chuanlong, chen.f, honglanqing, sun.rui3, li.zhenguo\}@huawei.com

}

\begin{document}

\maketitle

\begin{abstract}
Automated data augmentation has shown superior performance in image recognition. Existing works search for dataset-level augmentation policies without considering individual sample variations, which are likely to be sub-optimal. On the other hand, learning different policies for different samples naively could greatly increase the computing cost. In this paper, we learn a sample-aware data augmentation policy efficiently by formulating it as a sample reweighting problem. Specifically, an augmentation policy network takes a transformation and the corresponding augmented image as inputs, and outputs a weight to adjust the augmented image loss computed by a task network. At training stage, the task network minimizes the weighted losses of augmented training images, while the policy network minimizes the loss of the task network on a validation set via meta-learning. We theoretically prove the convergence of the training procedure and further derive the exact convergence rate. Superior performance is achieved on widely-used benchmarks including CIFAR-10/100, Omniglot, and ImageNet.
\end{abstract}
\section{Introduction}

\begin{figure*}[t]
    \centering
    \includegraphics[width=0.8\textwidth]{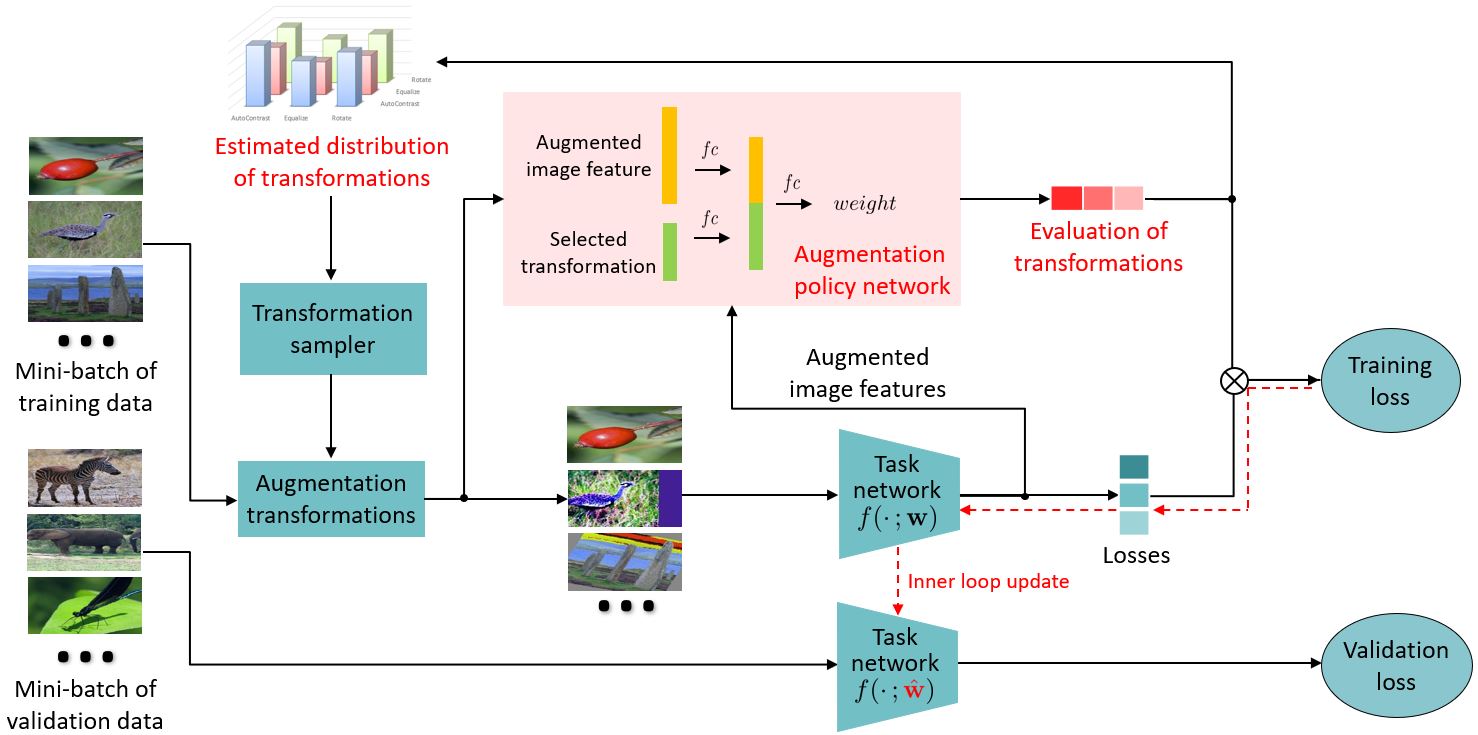}
    \caption{
    An overview of the proposed MetaAugment. The augmentation policy network outputs the weights of the augmented image losses and is learned to evaluate the effectiveness of different transformations for different training images via meta-learning, while the task network is trained to minimize the weighted training loss alternately with the updating of the policy network. For higher training efficiency, the transformation sampler samples transformations according to a distribution refined with the training process of the policy network.
    }
    \label{fig:model}
    \vspace{-10pt}
\end{figure*}

Data augmentation is widely used to increase the diversity of training data in order to improve model generalization~\cite{krizhevsky2012imagenet,srivastava2015training,han2017deep,devries2017improved,zhang2017mixup,yun2019cutmix}. 
Automated data augmentation that searches for data-driven augmentation policies improves the performance of deep models in image recognition compared with the manually designed ones. 
A data augmentation policy is a distribution of transformations, according to which training samples are augmented. 
Reinforcement learning~\cite{cubuk2019autoaugment,zhang2020adversarial}, population-based training~\cite{ho2019population}, and Bayesian optimization~\cite{lim2019fast} have been employed to learn augmentation policies from target datasets. 
Despite the difference of search algorithms, these approaches search for policies at the dataset level, i.e., all samples in the dataset are augmented with the same policy. 
For an image recognition task, left translation may be suitable for the image where the target object is on the right, but may not be suitable for the image where the target object is on the left (see Figure~\ref{fig:weights}). 
According to this observation, dataset-level polices may give rise to various noises such as noisy labels, misalignment, or image distortion, since different samples vary greatly in object scale, position, color, illumination, etc.

To increase data diversity while avoiding noises, it is appealing to learn a sample-aware data augmentation policy, i.e., learning different distributions of transformations for different samples. However, it is time-consuming to evaluate a large number of distributions and non-trivial to determine the relation among the distributions. Augmenting training samples with the corresponding policies, we consider the augmented sample loss as a random variable and train a task network to minimize the expectation of the augmented sample loss. From this perspective, learning a sample-aware policy can be regarded as reweighting the augmented sample losses and the computing cost can be greatly reduced.

In this paper, we propose an efficient method, called MetaAugment, to learn a sample-aware data augmentation policy by formulating it as a sample reweighting problem. An overview of the proposed method is illustrated in Figure~\ref{fig:model}. 
Given a transformation and the corresponding augmented image feature, extracted by a task network, an augmentation policy network outputs the weight of the augmented image loss. The task network is optimized by minimizing the weighted training loss, while the goal of the policy network is to improve the performance of the task network on a validation set via adjusting the weights of the losses. This is a bilevel optimization problem~\cite{colson2007overview} which is hard to be optimized. 
We leverage the mechanism of meta-learning~\cite{finn2017model,li2017meta,ren2018learning,wu2018learning,liu2018darts,shu2019meta} to solve this problem. The motivation is based on the ability of meta-learning to extract useful knowledge from related tasks. 
During training, classification for each batch of samples is treated as a task. The policy network acts as a meta-learner to adapt the task network with the augmented samples such that it can perform well on a batch of validation samples. 
Instead of learning an initialization for fast adaptation in downstream tasks, the policy network learns to augment while guiding the actual training process of the task network. 
We also propose a novel transformation sampler that samples transformations according to a distribution estimated by the outputs of the policy network. In principle, the distribution reflects the overall effectiveness of the transformations for the whole dataset and the transformation sampler can avoid invalid ones to improve the training efficiency. Furthermore, we theoretically show the convergence guarantee of our algorithm.

Our main contributions can be summarized as follows:

1) We propose MetaAugment to learn a sample-aware augmentation policy network that captures the variability of training samples and evaluates the effectiveness of transformations for different samples. 

2) We systematically investigate the convergence properties under two cases: (i) the policy network has its own feature extractor; (ii) the policy network depends on the parameters of the task network. We also point out the exact convergence rate and the optimization bias of our algorithm.

3) Extensive experimental results show that our method consistently improves the performance of various deep networks and 
outperforms previous automated data augmentation methods on CIFAR-10/100, Omniglot, and ImageNet.

\section{Related Work}

\noindent\textbf{Automated Data Augmentation.} There are rich studies on data augmentation in the past few decades, while automated data augmentation is a relatively new topic. Inspired by neural architecture search, AutoAugment~\cite{cubuk2019autoaugment} adopts reinforcement learning to train a controller to generate augmentation policies such that a task network trained along with the policies may have the highest validation accuracy. Adversarial AutoAugment~\cite{zhang2020adversarial} trains a controller to generate adversarial augmentation policies that increase the training loss of a task network. 
Inspired by hyper-parameter optimization, PBA~\cite{ho2019population} learns an epoch-aware augmentation schedule instead of a fixed policy for all training epochs. Following Bayesian optimization, FAA~\cite{lim2019fast} searches for policies that match the distribution of augmented data with that of unaugmented data. DADA~\cite{li2020dada} proposes to relax the discrete selection of augmentation policies to be differentiable and uses gradient-based optimization to do policy search.
These methods overlook the variability of training samples and adopt the same policy for all samples. RandAugment~\cite{cubuk2019randaugment} shows that hyper-parameters in such policies do not affect the results a lot. Our method learns a sample-aware policy network that associates different pairs of transformations and augmented samples with different weights. 

\noindent\textbf{Sample Reweighting.}  
There are many studies on sample reweighting for specific issues, e.g., class imbalance~\cite{johnson2019survey} and label noise~\cite{zhang2018generalized}. 
Among them, there are mainly two types of weighting functions. 
The first one, suitable for class imbalance, is to increase the weights of hard samples~\cite{freund1995desicion, johnson2019survey, malisiewicz2011ensemble, lin2017focal}, while the second one, suitable for noise label, is to increase the weights of easy samples~\cite{kumar2010self, jiang2014easy, jiang2014self, zhang2018generalized}. Instead of manually designing the weight functions, \citet{ren2018learning} propose an online reweighting method that learns sample weights directly from data via meta-learning. 
Meta-Weight-Net~\cite{shu2019meta} adopts a neural network to learn the mapping from sample loss to sample weight, which stabilize the weighting behavior. \citet{wang2019optimizing} train a scorer network to up-weight training data that have similar loss gradients with validation data via reinforcement learning. 
Different from these works, our policy network aims to evaluate different transformations for different samples and assign weights to augmented samples. 

\section{Methodology}

\subsection{Sample-Aware Data Augmentation}

Consider an image recognition task with the training set 
$\cD^{tr} = \{(x_i, y_i)\}_{i=1}^{N^{tr}}$, where $y_i$ is the label of the image $x_i$, and $N^{tr}$ is the sample size. Training samples are augmented by various transformations. Each transformation consists of two image processing functions, such as rotation, translation, coloring, etc., to be applied in sequence. Each function is associated with a magnitude that is rescaled to and sampled uniformly from $[0, 10]$. Given $K$ image processing functions in order, let $\cT_{j,k}^{m_1,m_2}(x_i)$ be a transformation applied on an image $x_i$ with $j$-th and $k$-th functions in order and the magnitudes are $m_1$ and $m_2$, respectively. 

Intuitively, not all of the augmented samples may help to improve the performance of a task network, and thus, an augmentation policy network is proposed to learn the effectiveness of different transformations for different training samples. Let $f(x_i;\bw)$ be the task network with parameters $\bw$. By abuse of notation, the deep feature of $x_i$ extracted by the task network is also denoted by $f(x_i;\bw)$. For each pair of augmented sample feature $f(\cT_{j,k}^{m_1,m_2}(x_i); \bw)$ and the embedding of the applied transformation $e(\cT_{j,k}^{m_1,m_2})$, the policy network $\cV(\cdot,\cdot\,;\bth)$ with parameters $\bth$ takes the pair as input and outputs a weight that is imposed on the augmented sample loss $L_{i,j,k}(m_1,m_2;\bw)=\ell(f(\cT_{j,k}^{m_1,m_2}(x_i); \bw),y_i)$. The task network is trained to minimize the following weighted training loss: 
\small
\begin{linenomath}
\begin{multline*}
\cL^{tr}(\bw,\bth)=\frac{1}{N^{tr}} \sum_{i=1}^{N^{tr}} \frac{1}{K^2} \sum_{j,k=1}^K 
\bbE_{m_1,m_2\sim U(0,10)}\Big[\\ \cV_{i,j,k}(m_1,m_2;\bw,\bth)L_{i,j,k}(m_1,m_2;\bw)\Big],
\end{multline*}
\end{linenomath}
\normalsize
where 
\small
\begin{linenomath}
\begin{equation*}
\cV_{i,j,k}(m_1,m_2;\bw,\bth) = \cV(f(\cT_{j,k}^{m_1,m_2}(x_i);\bw),e(\cT_{j,k}^{m_1,m_2});\bth)
\end{equation*}
\end{linenomath}
\normalsize
and $U(0,10)$ denotes the uniform distribution over $[0, 10]$. 
The objective of the policy network is to output the accurate sample weights such that the task network has the best performance on a validation set $\cD^{val} = \{(x_{i'}^{val}, y_{i'}^{val})\}_{i'=1}^{N^{val}}$ via minimizing $\cL^{tr}(\bw,\bth)$. 
Mathematically, we formulate the following optimization problem:
\small
\begin{linenomath}
\begin{equation}\label{euq:bilevel}
\begin{aligned}
& \min_{\bth} 
& & \cL^{val}(\bw^*(\bth)) = \frac{1}{N^{val}}\sum_{i'=1}^{N^{val}} L_{i'}^{val}(\bw^*(\bth))  \\
& \text{subject to} 
& & \bw^*(\bth) = \mathop{\arg\min}_{\bw} \cL^{tr}(\bw, \bth),
\end{aligned}
\end{equation}
\end{linenomath}
\normalsize
where $L_{i'}^{val}(\bw^*(\bth))=\ell(f(x_{i'}^{val};\bw^*(\bth)), y_{i'}^{val})$. This is a bilevel optimization problem~\cite{colson2007overview}, which is hard to solve since as the updating of $\bth$, the parameters of the task network are required to be optimized accordingly. Recent works~\cite{ren2018learning,wu2018learning,liu2018darts,shu2019meta} use meta-learning techniques to get approximate optimal solutions for such bilevel optimization problems. We also leverage meta-learning and employ the updating rules proposed in~\cite{shu2019meta,li2017meta,antoniou2018how} to solve problem~\eqref{euq:bilevel}.

\subsection{Proposed MetaAugment Algorithm}

The policy and task networks are trained alternately. For each iteration, a mini-batch of training data $\cD^{tr}_{mi}=\{(x_i, y_i)\}_{i=1}^{n^{tr}}$ with batch size $n^{tr}$ is sampled and for each $x_i$, a transformation $\cT_{j_i,k_i}^{m_1,m_2}$ is sampled to augment $x_i$. For notation simplicity, let $\cV_{i}(\bw,\bth)=\cV(f(\cT_{j_i,k_i}^{m_1,m_2}(x_i);\bw),e(\cT_{j_i,k_i}^{m_1,m_2});\bth)$ and $L_{i}(\bw)=\ell(f(\cT_{j_i,k_i}^{m_1,m_2}(x_i); \bw),y_i)$. 
Then the inner loop update of $\bw$ in iteration $t+1$ is 
\small
\begin{linenomath}
\begin{equation}\label{inner}
\hat \bw^{(t)}(\bth, \alpha) = \bw^{(t)} - \alpha
\frac{1}{n^{tr}}\sum_{i=1}^{n^{tr}} \cV_{i}(\bw^{(t)},\bth) \nabla_{\bw} L_{i}(\bw^{(t)}),
\end{equation}
\end{linenomath}
\normalsize
where $\alpha$ is a learnable learning rate~\cite{li2017meta,antoniou2018how} and $\nabla_{\bw} L_{i}(\bw^{(t)}) = \nabla_{\bw} L_{i}(\bw) \big|_{\bw^{(t)}}$. We adopt a learnable $\alpha$ because 
it is unclear how to set the learning rate schedule manually for this inner loop update and proper schedules may vary for different training datasets. We regard $\cV_{i}(\bw,\bth)$ as a function of $\bth$ and do not take derivative of $\cV_{i}(\bw,\bth)$ with respect to $\bw$ in Eq.~\eqref{inner}. This is because $\cV_{i}(\bw,\bth)$ shall be fixed when updating $\bf{w}$ and the weighted training loss shall not be minimized via minimizing $\cV_{i}(\bw,\bth)$.
It can also avoid a second-order derivative when updating the policy network, which otherwise will substantially increase the computational complexity.

The formulation $\hat \bw^{(t)}(\bth, \alpha)$ is regarded as a function of $\bth$ and $\alpha$, and then $\bth$ and $\alpha$ can be updated via the validation loss computed by $\hat \bw^{(t)}(\bth, \alpha)$ on a mini-batch of validation samples $\cD^{val}_{mi}= \{(x^{val}_{i'}, y^{val}_{i'})\}_{i'=1}^{n^{val}}$ with batch size $n^{val}$. The outer loop updates of $\bth$ and $\alpha$ are formulated by 
\small
\begin{linenomath}
\begin{multline}\label{outer1}
(\bth^{(t+1)},\alpha^{(t+1)}) = (\bth^{(t)},\alpha^{(t)})\\ - \beta \frac{1}{n^{val}} \sum_{i'=1}^{n^{val}} \nabla_{(\bth,\alpha)} L^{val}_{i'}(\hat\bw^{(t)}(\bth^{(t)},\alpha^{(t)})),
\end{multline}
\end{linenomath}
\normalsize
where $\beta$ is a learning rate and
$\nabla_{(\bth,\alpha)} L^{val}_{i'}(\hat\bw^{(t)}(\bth^{(t)},\alpha^{(t)}))$  $=\nabla_{(\bth,\alpha)} L^{val}_{i'}(\hat\bw^{(t)}(\bth,\alpha))\big|_{(\bth^{(t)},\alpha^{(t)})}$.
The third step in iteration $t+1$ is the outer loop update of $\bw^{(t)}$ with the updated $\bth^{(t+1)}$:
\small
\begin{linenomath}
\benr\label{outer2}
\bw^{(t+1)} = \bw^{(t)} - \gamma \frac{1}{n^{tr}} \sum_{i=1}^{n^{tr}} \cV_{i}(\bw^{(t)},\bth^{(t+1)}) \nabla_{\bw} L_{i}(\bw^{(t)}),
\eenr
\end{linenomath}
\normalsize
where $\gamma$ is a learning rate. With these updating rules, the two networks can be trained efficiently.

Although the policy network outputs the weights that evaluate the importance of the augmented samples, sampling invalid transformations constantly may lead to poor training efficiency. We propose a novel transformation sampler that sample transformations according to a probability distribution estimated by the outputs of the policy network and refined with the training process of the policy network. Specifically, let $\{\cV(f(\cT_{j_i,k_i}^{m_1,m_2}(x_i);\bw),e(\cT_{j_i,k_i}^{m_1,m_2});\bth)\}_{i=1}^{r\cdot n^{tr}}$ denote the collection of the policy network outputs in the last $r$ iterations. Then the average value of the outputs corresponding to the transformation with $j$-th and $k$-th functions in order (without magnitude) is estimated by 
\small
\begin{linenomath}
\[
v_{j,k} = \frac{1}{c_{j,k}} \sum_{i=1}^{r\cdot n^{tr}}\sum_{j_i=j,k_i=k} \cV(f(\cT_{j_i,k_i}^{m_1,m_2}(x_i);\bw),e(\cT_{j_i,k_i}^{m_1,m_2});\bth),
\]
\end{linenomath}
\normalsize
where $c_{j,k}$ is the number of terms in the summation. In our implementation, the output of the policy network is with the Sigmoid function to ensure the output is positive. To balance exploration and exploitation, and to avoid the biases caused by underfitting of the policy network, the sampler samples each transformation according to the following distribution:
\small
\begin{linenomath}
\begin{equation}\label{probability}
p_{j,k} = (1-\epsilon)\cdot\frac{v_{j,k}}{\sum_{l,m=1}^K v_{l,m}} + \epsilon\cdot\frac{1}{K^2},
\end{equation}
\end{linenomath}
\normalsize
where $\epsilon$ is a hyper-parameter, and the corresponding magnitudes are sampled uniformly from $[0,10]$. The probability $p_{j,k}$ is updated every $s$ iterations. This estimated distribution reflects the overall effectiveness of the transformations for the whole dataset and evolves synergistically with the policy network. Dataset-level and sample-level augmentation policies are combined together by these two modules. The MetaAugment algorithm is summarized in Algorithm~\ref{algo1}.

\begin{algorithm}[t]\small
	\caption{MetaAugment: Sample-Aware Data Augmentation Policy Learning}
	\label{algo1}
	\begin{algorithmic}[1]
		\REQUIRE Training data $\cD^{tr}$, validation data $\cD^{val}$, $K$ image processing functions, batch sizes $n^{tr}$, $n^{val}$, learning rate $\beta$, $\gamma$, sampler hyper-parameters $r$, $s$, $\epsilon$, iteration number $T$
		\ENSURE $\bw^{(T)}$, $\bth^{(T)}$, $\{p_{j,k}\}_{j,k=1}^K$
		\STATE Initialize $\bw^{(0)}$, $\bth^{(0)}$, $\alpha^{(0)}$, $\{p_{j,k}=\frac{1}{K^2}\}_{j,k=1}^K$;
		\FOR {$0 \leq t \leq T-1$} 
		\STATE Sample a mini-batch of training samples $\cD^{tr}_{mi}$ with batch size $n^{tr}$;
		\STATE For each sample in the mini-batch, sample a transformation according to $p_{j,k}$ and the corresponding magnitudes uniformly from $[0,10]$;
		\STATE Augment the batch data with the sampled transformations;
		\STATE Sample a mini-batch of validation samples $\cD^{val}_{mi}$ with batch size $n^{val}$;
		\STATE Compute $\hat \bw^{(t)}(\bth,\alpha)$ according to Eq.~\eqref{inner};
		\STATE Update $(\bth^{(t+1)},\alpha^{(t+1)})$ according to Eq.~\eqref{outer1};
		\STATE Update $\bw^{(t+1)}$ according to Eq.~\eqref{outer2};
		\IF {$(t+1) \bmod{s} = 0$}
		    \STATE Update $p_{j,k}$ according to Eq.~\eqref{probability} with the policy network outputs in the last $\min(t+1,r)$ iterations;
		\ENDIF
		\ENDFOR
	\end{algorithmic}
%	\vspace{-5pt}	
\end{algorithm}
%\vspace{-5pt}

In each iteration, MetaAugment requires three forward and backward passes of the task network, which makes it take $3\times$ training time than a standard training scheme. However, once trained, the policy network, together with the task network and the estimated distribution $\{p_{j,k}\}_{j,k=1}^{K}$ can be transferred to train different networks on the same dataset efficiently. More details are provided in Appendix.

\subsection{Convergence Analysis}
Motivated by Meta-Weight-Net~\cite{shu2019meta}, we analyze the convergence of the proposed algorithm. 
In technical details, we release the assumptions of Meta-Weight-Net, e.g. $\sum_{t=1}^{\infty} \beta_t \leq \infty$ and $\sum_{t=1}^{\infty} \beta_t^2 \leq \infty$, which are invalid in many cases.
We find a proper trade-off between the training and validation convergence and exactly point out the convergence rate and the optimization bias.
Furthermore, we systematically investigate two situations: (i) the policy network has its own feature extractor; (ii) the policy network depends on the feature extractor of the task network.
For the case (i), the convergence is guaranteed on both validation and training data, while for the case (ii), the conclusion on the validation data still holds, but the convergence is not ensured on the training data.
However, if the policy network is also a deep network, it will take nearly $4.5\times$ training time than a standard training scheme. Also, with limited validation data, it may overfit and thus make the task network overfit the validation data. Hence, we choose the latter case in our algorithm. 
We assume $\alpha$ is fixed during training and postpone the proof into Appendix.

\begin{theorem}\label{theorem1}
Suppose that the loss function $\ell$ has $\rho_1$-bounded gradients with respect to $\bw$ under both (augmented) training data and validation data, $\ell$ is Lipschitz smooth with constant $\rho_2$, the policy network $\cV$ is differential with a $\delta_1$-bounded gradient and twice differential with its Hessian bounded by $\delta_2$ with respect to $\bth$, and the absolute values of $\cV$ and $\ell$ are bounded above by $C_1$ and $C_2$, respectively. Furthermore, for any iteration $0 \leq t \leq T-1$, the variance of the weighted training loss (validation loss) gradient on a mini-batch of training (validation) samples is bounded above. Let
{\small
\begin{linenomath}
\benrr
\alpha = \frac{c \log T}{T}, \quad \beta = \sqrt{\frac{c' \log\log T}{T}}, \quad \gamma =  \frac{c''\log T}{T},
\eenrr
\end{linenomath}
}
for some positive constants $c$, $c'$ and $c''$. The number of iterations $T$ is sufficiently large such that $\alpha \beta \rho_1^2(\alpha \delta_1^2 \rho_2 + \delta_2)<1$ and $\gamma C_1 \rho_2 <1.$ If 
the policy network has its own feature extractor, we have
{\small
\begin{linenomath}
\benr
\frac{1}{T} \sum_{t=0}^{T-1}\bbE \Big[  \big\|\nabla_{\bth} \cL^{val}(\hat \bw^{(t)}(\bth^{(t)}))\big\|^2 \Big] \leq O(\frac{\log T}{\sqrt{T\log\log T}}), \label{conv1}\\
\lim_{T \rightarrow \infty} 
\frac{1}{T}\sum_{t=0}^{T-1} \bbE \Big[\|\nabla_{\bw} \cL^{tr}(\bw^{(t)}, \bth^{(t+1)})\|^2 \Big] = 0. \label{conv2}
\eenr
\end{linenomath}
}
\end{theorem}
If the policy network uses the feature extractor of the task network, the weights in the training loss will change when $\bw$ updates. Since we regard $\cV$ as a fixed weight when updating $\bw$, the weighted training loss at the end of the last iteration is different from the weighted training loss at the beginning of the current iteration. The discontinuity leads to a bias term in the convergence of the weighted training loss. 

\begin{theorem}\label{theorem2}
Suppose the assumptions of  Theorem~\ref{theorem1} hold. Further assume that the policy network $\cV$ depends on $\bw$ and is differential with a $\tilde \delta_1$-bounded gradient 
with respect to $\bw.$ Then we have that (\ref{conv1}) still holds and
{\small
\begin{linenomath}
\begin{equation}\label{conv3}
\frac{1}{T}\sum_{t=0}^{T-1} \bbE \Big[\|\nabla_{\bw} \cL^{tr}(\bw^{(t)}, \bth^{(t+1)})\|^2 \Big] - 2\rho_1 \tilde \delta_1 C_1 C_2 \leq o(1).
\end{equation}
\end{linenomath}
}
\end{theorem}
According to the proof of Theorem~\ref{theorem2}, one can find that under certain conditions, (\ref{conv2}) can still hold even if the policy network depends on the feature extractor of the task network.

\section{Experimental Results}\label{sec:experiment}

\begin{table*}[t]
	\centering
	%\vspace{-10pt}
	\caption{Top-1 test accuracy (\%) on CIFAR-10 and CIFAR-100. 
	}
	\label{table:cifar}
	\vspace{-0.2cm}
	\begin{adjustbox}{max width=0.95\textwidth}
		\begin{tabular}{llcccccccc}
			\toprule
			\toprule
			Dataset & Model & Baseline & AA & FAA & PBA & DADA & RA & AdvAA & MetaAugment \\
			\midrule
			CIFAR-10 & WRN-28-10 & 96.1 & 97.4 & 97.3 & 97.42 & 97.3 & 97.3 & \textbf{98.10} & 97.76$\pm$0.04 \\
			& WRN-40-2 & 94.7 & 96.3 & 96.4 & - & 96.4 & - & - & \textbf{96.79$\pm$0.06} \\
			& Shake-Shake (26 2x96d) & 97.1 & 98.0 & 98.0 & 97.97 & 98.0 & 98.0 & 98.15 & \textbf{98.29$\pm$0.03} \\
			& Shake-Shake (26 2x112d) & 97.2 & 98.1 & 98.1 & 97.97 & 98.0 & - & 98.22 & \textbf{98.28$\pm$0.01} \\
			& PyramidNet+ShakeDrop & 97.3 & 98.5 & 98.3 & 98.54 & 98.3 & 98.5 & \textbf{98.64} & 98.57$\pm$0.02 \\
			\midrule
			CIFAR-100 & WRN-28-10 & 81.2 & 82.9 & 82.8 & 83.27 & 82.5 & 83.3 & \textbf{84.51} & 83.79$\pm$0.11 \\
			& WRN-40-2 & 74.0 & 79.3 & 79.4 & - & 79.1 & - & - & \textbf{80.60$\pm$0.16} \\
			& Shake-Shake (26 2x96d) & 82.9 & 85.7 & 85.4 & 84.69 & 84.7 & - & 85.90 & \textbf{85.97$\pm$0.09} \\
			& PyramidNet+ShakeDrop & 86.0 & 89.3 & 88.3 & 89.06 & 88.8 & - & \textbf{89.58} & 89.46$\pm$0.11 \\
			\bottomrule
			\bottomrule
		\end{tabular}
	\end{adjustbox}
	%\vspace{-0.2cm}
\end{table*}

In this section, we evaluate MetaAugment for image recognition tasks on CIFAR-10/100~\cite{krizhevsky2009learning}, Omniglot~\cite{Lake_oneshot}, and ImageNet~\cite{deng2009imagenet}. We show the effectiveness of MetaAugment with different task network architectures and visualize the learned augmentation policies to illustrate the necessity of sample-aware data augmentation. 

\begin{table}[t]
	\centering
	\caption{Top-1 test accuracy (\%) on CIFAR using Multiple Transformations (MT) for each sample in a mini-batch.}
	\label{table:advAAcifar}
	\vspace{-0.2cm}
	\begin{adjustbox}{max width=0.47\textwidth}
	    \begin{tabular}{llccccc}
			\toprule
			\toprule
			Dataset & Model & AdvAA & MetaAugment+MT \\
			\midrule
			CIFAR-10 & WRN-28-10 & 98.10 & \textbf{98.26$\pm$0.02} \\
			CIFAR-100 & WRN-28-10 & 84.51 & \textbf{85.21$\pm$0.09} \\
			\bottomrule
			\bottomrule
		\end{tabular}
	\end{adjustbox}
	\vspace{-5pt}
\end{table}

In our implementation, we use $K=14$ image processing functions: AutoContrast, Equalize, Rotate, Posterize, Solarize, Color, Contrast, Brightness, Sharpness, ShearX/Y, TranslateX/Y, Identity~\cite{cubuk2019randaugment,cubuk2019autoaugment,ho2019population,lim2019fast,zhang2020adversarial}. 
The embedding of a particular transformation $\cT_{j,k}^{m_1,m_2}$ is a $28$-dimensional vector with $m_1 + 1$ in {\small $(2j-1)$}-th position, $m_2 + 1$ in {\small $(2k)$}-th position, and $0$ elsewhere. For AutoContrast, Equalize, and Identity that do not use magnitude, we let $11$ be in their positions. 
The augmentation policy network is an MLP that takes the embedding of the transformation and the corresponding augmented image feature as inputs, each followed by a fully-connected layer of size 100 with ReLU nonlinearities. The two intermediate features are then concatenated together, followed by a fully-connected output layer of size 1. The Sigmoid function is applied to the output. We also normalize the output weights of training samples in each mini-batch, i.e., each weight is divided by the sum of all weights in the mini-batch. 
More implementation details and the hyper-parameters we used are provided in Appendix. 
All of the reported results are averaged over five runs with different random seeds.

\subsection{Results on CIFAR, Omniglot, and ImageNet}

\noindent\textbf{CIFAR.} CIFAR-10 and CIFAR-100 consist of 50,000 images for training and 10,000 images for testing. For our method, we hold out 1,000 training images as the validation data. We compare MetaAugment with Baseline, AutoAugment (AA)~\cite{cubuk2019autoaugment}, FAA~\cite{lim2019fast}, PBA~\cite{ho2019population}, DADA~\cite{li2020dada}, RandAugment (RA)~\cite{cubuk2019randaugment}, and Adversarial AutoAugment (AdvAA)~\cite{zhang2020adversarial} on Wide-ResNet (WRN)~\cite{zagoruyko2016wide}, Shake-Shake~\cite{gastaldi2017shake}, and PyramidNet+ShakeDrop~\cite{han2017deep,yamada2018shakedrop}. 
The Baseline adopts the standard data augmentation: horizontal flipping with 50\% probability, zero-padding and random cropping. For MetaAugment, the transformation is applied after horizontal flipping, and then Cutout~\cite{devries2017improved} with $16\times16$ pixels is applied. 

\begin{figure}[t]
	\begin{subfigure}{.23\textwidth}
		\centering
		\includegraphics[width=.9\linewidth]{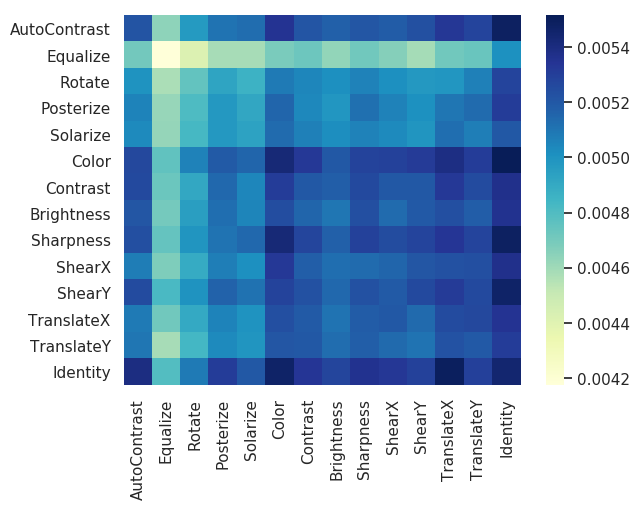}
		\caption{CIFAR-10}
		\label{fig:cifar10}
	\end{subfigure}
	\begin{subfigure}{.23\textwidth}
		\centering
		\includegraphics[width=.9\linewidth]{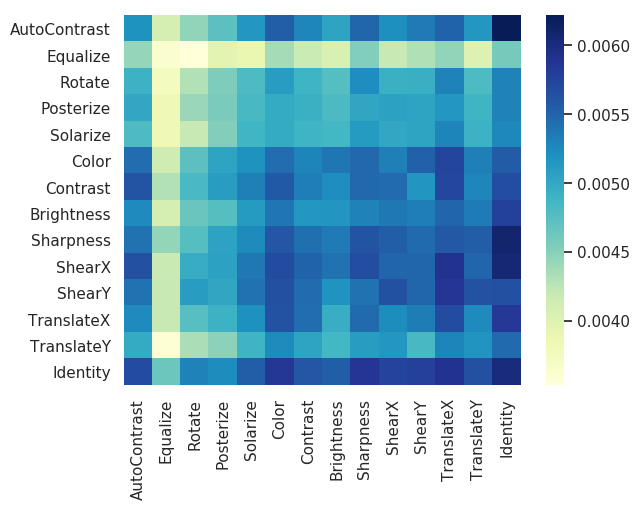}
		\caption{CIFAR-100}
		\label{fig:cifar100}
	\end{subfigure}
	\begin{subfigure}{.23\textwidth}
		%\vspace{5pt}
		\centering
		\includegraphics[width=.9\linewidth]{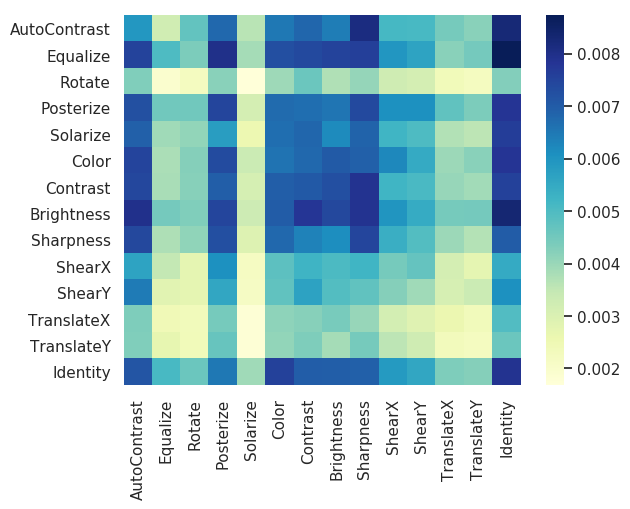}
		\caption{Omniglot}
		\label{fig:omniglot}
	\end{subfigure}
	\begin{subfigure}{.23\textwidth}
		%\vspace{5pt}
		\centering
		\includegraphics[width=.9\linewidth]{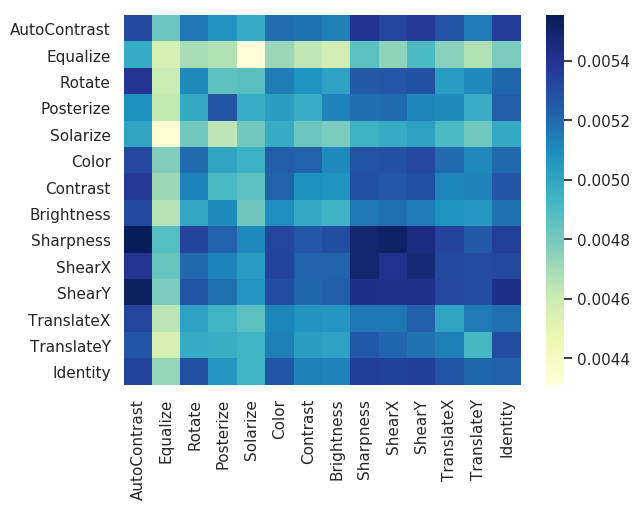}
		\caption{ImageNet}
		\label{fig:imagenet}
	\end{subfigure}
	\caption{Estimated distributions of transformations on (a) CIFAR-10, (b) CIFAR-100, (c) Omniglot, and (d) ImageNet.}
	\label{fig:heat_maps}
	\vspace{-15pt}
\end{figure}

The mean test accuracy and Standard Deviation (Std Dev) of MetaAugment, together with the results of other competitors, are reported in Table~\ref{table:cifar}. On both of CIFAR-10 and CIFAR-100, our method outperforms AA, FAA, PBA, DADA, and RA on all of the models. Compared with AdvAA, MetaAugment shows slightly worse results on WRN-28-10 and PyramidNet+ShakeDrop, and better results on Shake-Shake. 
However, AdvAA trains a task network with a large batch consisting of samples augmented by 8 augmentation policies. The Multiple-Transformation-per-sample (MT) trick leads to better performance but $8\times$ more computing cost than the regular training. We also compare MetaAugment with AdvAA in the MT setting. Each training sample in a mini-batch is augmented by 4 transformations and all the augmented samples are used to train the task network. The results are illustrated in Table~\ref{table:advAAcifar}. It can be seen that MetaAugment outperforms AdvAA in this setting.
Moreover, AdvAA assumes all transformations do not change the labels of data, which may not be valid in challenging cases. More details can be found in Figure~\ref{fig:omniglot_visual}. 
We visualize the estimated distributions of transformations in Figure~\ref{fig:heat_maps}. The difference in probability values is greater on CIFAR-100 than that on CIFAR-10, which shows the effectiveness of different transformations varies more on CIFAR-100. 

We train the policy network to assign proper weights to the augmented samples and use all of them to train the task network instead of rejecting the augmented samples with low weights. We also conduct experiment on the case that the policy network rejects the augmented samples with weights less than the mean of all the weights in a mini-batch. The results on CIFAR-100 with task networks WRN-28-10 and WRN-40-2 are 82.57\% and 79.01\% respectively, which are worse than the original case. It implies that samples with small weights are still useful. Ideally, the policy network can automatically assign very small weights to augmented samples that hurt the validation accuracy and we need no additional zeroing. Intuitively, rejecting augmented samples using a carefully selected threshold number may be helpful, but it is a bit far from the main idea of this paper.

\begin{table}[t]
	\centering
	%\vspace{-10pt}
	\caption{Top-1 test accuracy (\%) on Omniglot.}
	\label{table:omniglot}
	\vspace{-0.2cm}
	\begin{adjustbox}{max width=0.47\textwidth}
	    \begin{tabular}{lccccc}
			\toprule
			\toprule
			Model & Baseline & FAA & PBA & RA & MetaAugment \\
			\midrule
			WRN-28-10 & 87.89 & 89.24 & 89.25 & 87.86 & \textbf{89.61$\pm$0.05} \\
			WRN-40-2 & 85.86 & 88.72 & 88.30 & 88.10 & \textbf{89.12$\pm$0.10} \\
			\bottomrule
			\bottomrule
		\end{tabular}
	\end{adjustbox}
	\vspace{-15pt}
\end{table}

\noindent\textbf{Omniglot.} To investigate the universality of our method, we conduct experiments on Omniglot which contains images of 1,623 characters instead of natural objects.
For each character, we select 15, 2, and 3 images as training, validation, and test data. We compare MetaAugment with Baseline, FAA, PBA, and RA on WRN. 
The Baseline models are trained without data augmentation. For MetaAugment, transformations are applied to training samples directly with no Cutout added. For FAA and PBA, we do experiments with their open-source codes. For RA, we use our own implementation that randomly samples transformations and adopts the same weight for augmented samples. Implementation details are provided in Appendix. 

The results are reported in Table~\ref{table:omniglot}. It can be seen that MetaAugment outperforms the Baseline and RA by a wide margin and still achieves better results than FAA and PBA. 
We also visualize the estimated distribution in Figure~\ref{fig:heat_maps}. Different from CIFAR, geometric transformations have low probability values. 
This is because the geometric structure is the key feature of characters and should not be changed a lot as shown in Figure~\ref{fig:omniglot_visual}. In contrast, natural images in CIFAR contain rich texture and color information and less depend on geometric structure. The results indicate the robustness of our policy network when dealing with bad transformations. 
To compare with adversarial strategy in AdvAA, we visualize samples selected by adversarial strategy and our strategy, i.e., samples with high losses but low weights and those with low losses but high weights, in Figure~\ref{fig:omniglot_visual}.
In the first two rows, we observe that geometric transformations with large magnitudes may not preserve the labels and make the characters look like samples of different classes (the hard negatives). In this case, AdvAA that prefers the transformations leading to large sample losses may harm the performance.
In the last two rows, we observe that our method prefers the transformations that preserve the labels and key features of the augmented samples. 
Our method is more robust when many bad augmentation transformations are introduced in the search space.

\begin{figure}[t]
    \centering
    \includegraphics[width=.85\linewidth]{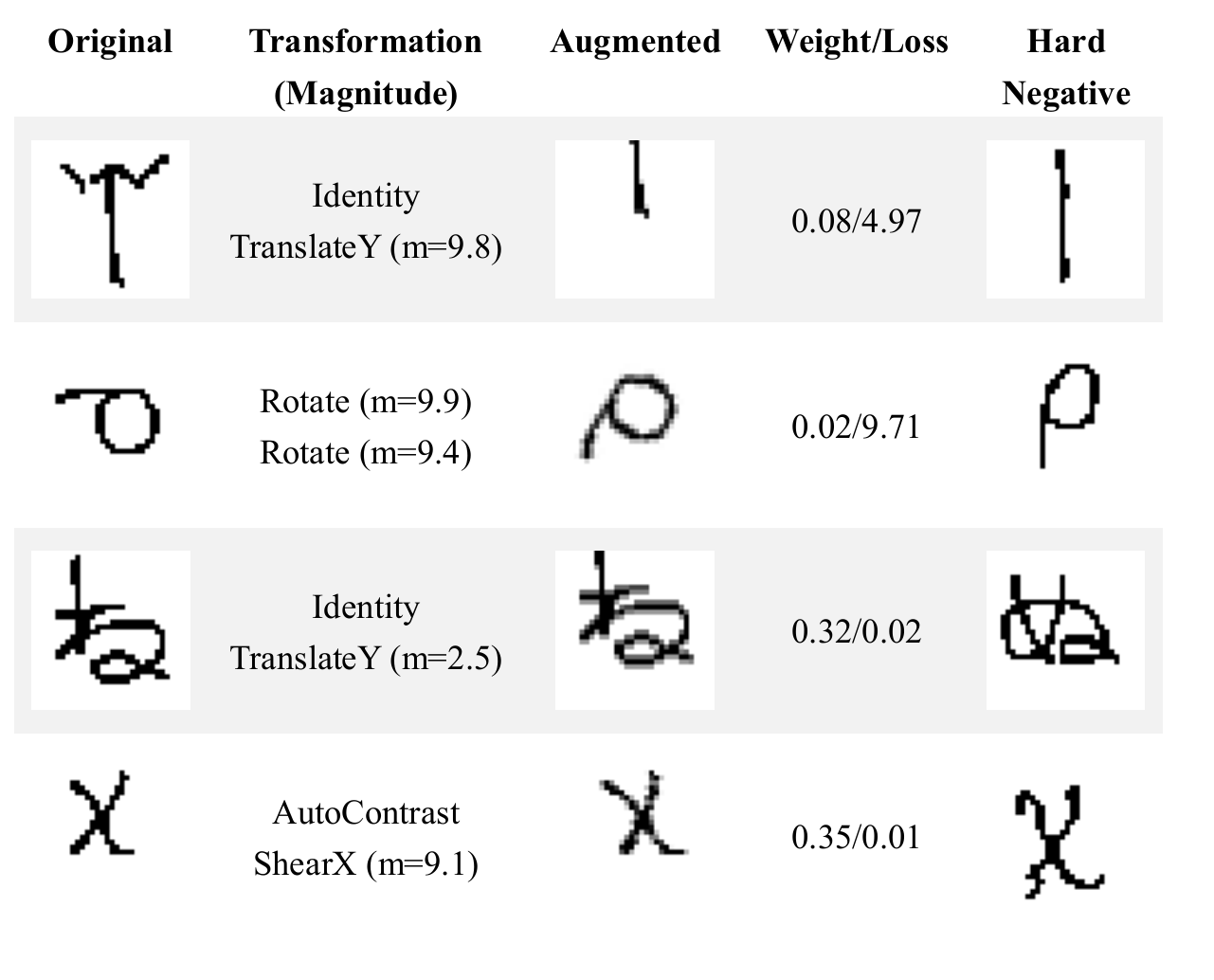}
    \vspace{-10pt}
    \caption{Examples of augmented samples on Omniglot. Here, hard negative means a validation sample w.r.t. similar feature map but different label.}
    \label{fig:omniglot_visual}
    %\vspace{-15pt}
\end{figure}

\begin{table*}[t]
	\centering
	%\vspace{-10pt}
	\caption{Top-1 / Top-5 test accuracy (\%) on ImageNet. 
	}
	\label{table:imagenet}
	\vspace{-0.2cm}
	\begin{adjustbox}{max width=0.9\textwidth}
		\begin{tabular}{lccccccc}
			\toprule
			\toprule
			Model & Baseline & AA & FAA & DADA & RA & AdvAA & MetaAugment \\
			\midrule
			ResNet-50 & 76.3 / 93.1 & 77.6 / 93.8 & 77.6 / 93.7 & 77.5 / 93.5 & 77.6 / 93.8 & 79.40 / 94.47 & \textbf{79.74$\pm$0.08 / 94.64$\pm$0.03} \\
			ResNet-200 & 78.5 / 94.2 & 80.0 / 95.0 & 80.6 / 95.3 & - & - & 81.32 / 95.30 & \textbf{81.43$\pm$0.08 / 95.52$\pm$0.04} \\
			\bottomrule
			\bottomrule
		\end{tabular}
	\end{adjustbox}
	\vspace{-0.1cm}
\end{table*}

\begin{figure}[t]
	\begin{subfigure}{.45\textwidth}
		\centering
		\includegraphics[width=.85\linewidth]{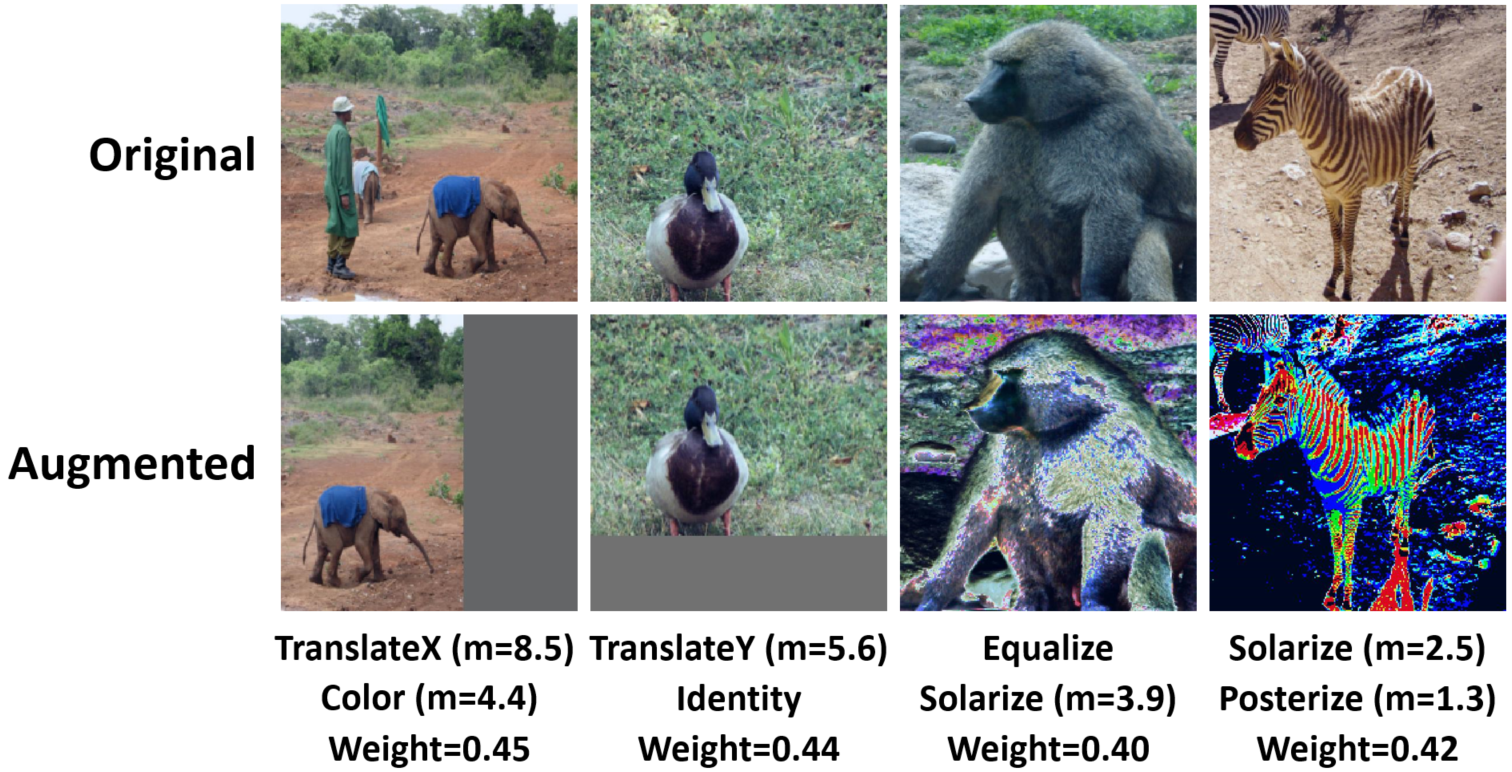}
		\caption{Augmented samples with high weights}
		\label{fig:high_weights}
	\end{subfigure}
	\begin{subfigure}{.45\textwidth}
		\centering
		\includegraphics[width=.85\linewidth]{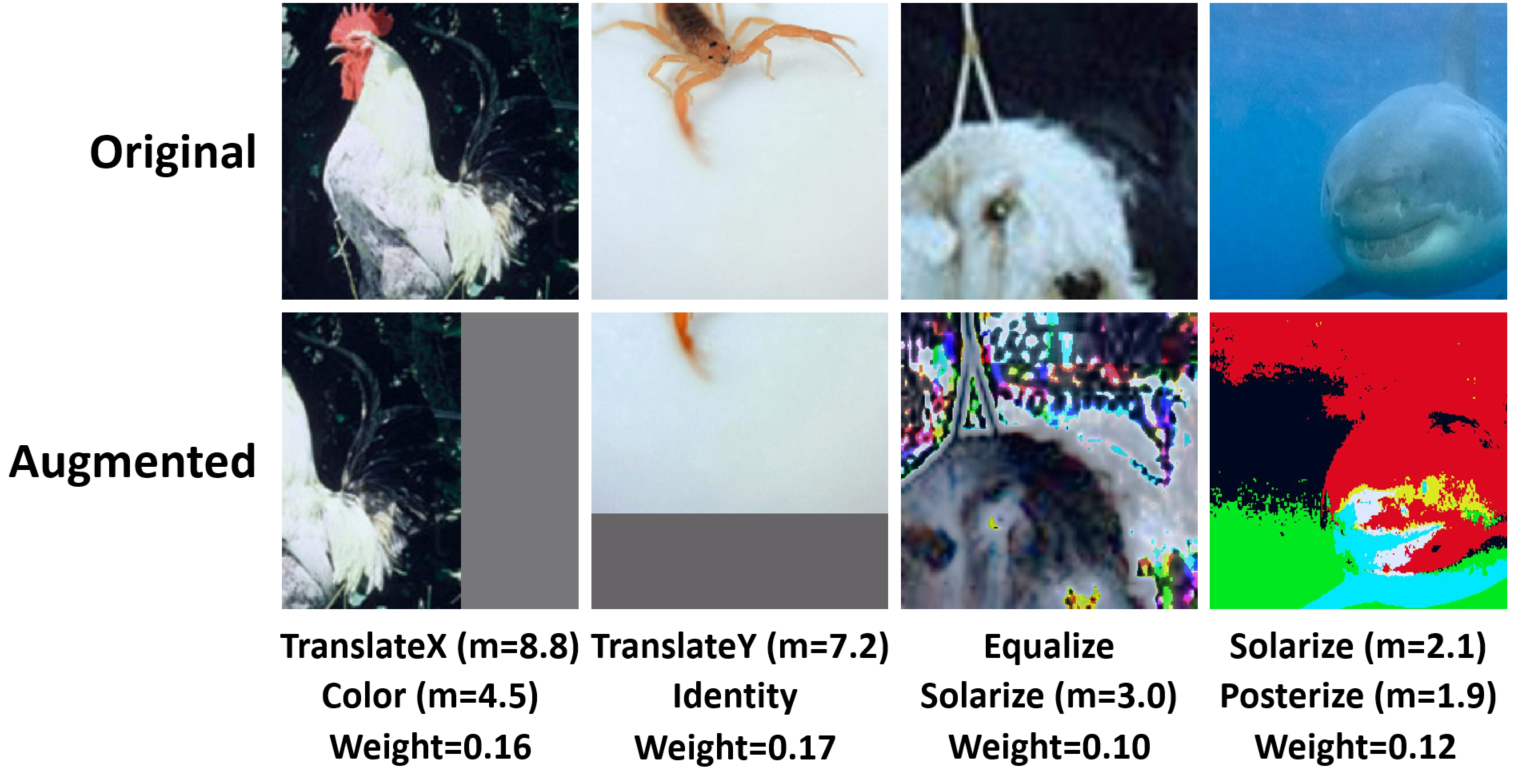}
		\caption{Augmented samples with low weights}
		\label{fig:low_weights}
	\end{subfigure}
	%\vspace{-5pt}
	\caption{Examples of augmented samples with (a) high and (b) low weights on ImageNet.}
	\label{fig:weights}
	\vspace{-5pt}
\end{figure}

\noindent\textbf{ImageNet.} ImageNet consists of colored images in 1,000 classes, with about 1.2 million images for training. For each class, we hold out 2\% of training images for validation. We compare MetaAugment with Baseline, AA, FAA, DADA, RA, and AdvAA on ResNet-50~\cite{he2016deep} and ResNet-200~\cite{he2016identity}. The Baseline models are trained with the standard Inception-style pre-processing~\cite{szegedy2015going}. For MetaAugment, the transformation is applied after random cropping, resizing to $224\times224$, and horizontal flipping with 50\% probability. 

The results are presented in Table~\ref{table:imagenet}. MetaAugment outperforms all the other automated data augmentation methods. The model ResNet-50 is trained with Multiple-Transformation-per-sample trick, i.e., each training sample in a mini-batch is augmented by 4 transformations. By assigning proper weights to the augmented samples, MetaAugment achieves superior performance. 
The estimated distribution of transformations is visualized in Figure~\ref{fig:heat_maps}. Transformations with Sharpness, ShearX, and ShearY have high probability values, while transformations with Equalize, Solarize, and Posterize have low probability values. To illustrate the necessity of sample-aware data augmentation, we display some augmented samples with high and low learned weights in Figure~\ref{fig:weights}. Similar transformations may have very different effects on different images. The policy network imposes high weights on the augmented images with elephant and duck that increase the diversity of training data, and imposes low weights on the augmented images with cock and scorpion that lose semantic information caused by the translation. Even for transformations with Equalize, Solarize, and Posterize that have low priority at the dataset level, the policy network is learned to assign high weights to informative images augmented by such transformations, as shown in Figure~\ref{fig:high_weights}.

\subsection{Ablation Studies}

\noindent\textbf{Transformation Sampler.} 
In the transformation sampler module, 
the hyper-parameter $\epsilon$ in Eq.~\eqref{probability} determines the probability of random sampling transformations. To investigate the influence of $\epsilon$, we conduct experiments on Omniglot with task network WRN-28-10. The mean test accuracy and Std Dev over five runs with different values of $\epsilon$ are depicted in Figure~\ref{fig:epsilon}. As expected, sampling transformations according to the estimated distribution with a certain randomness ($\epsilon=0.1$) outperforms random sampling ($\epsilon=1.0$). 

\begin{figure}[t]
    \centering
    \includegraphics[width=.6\linewidth]{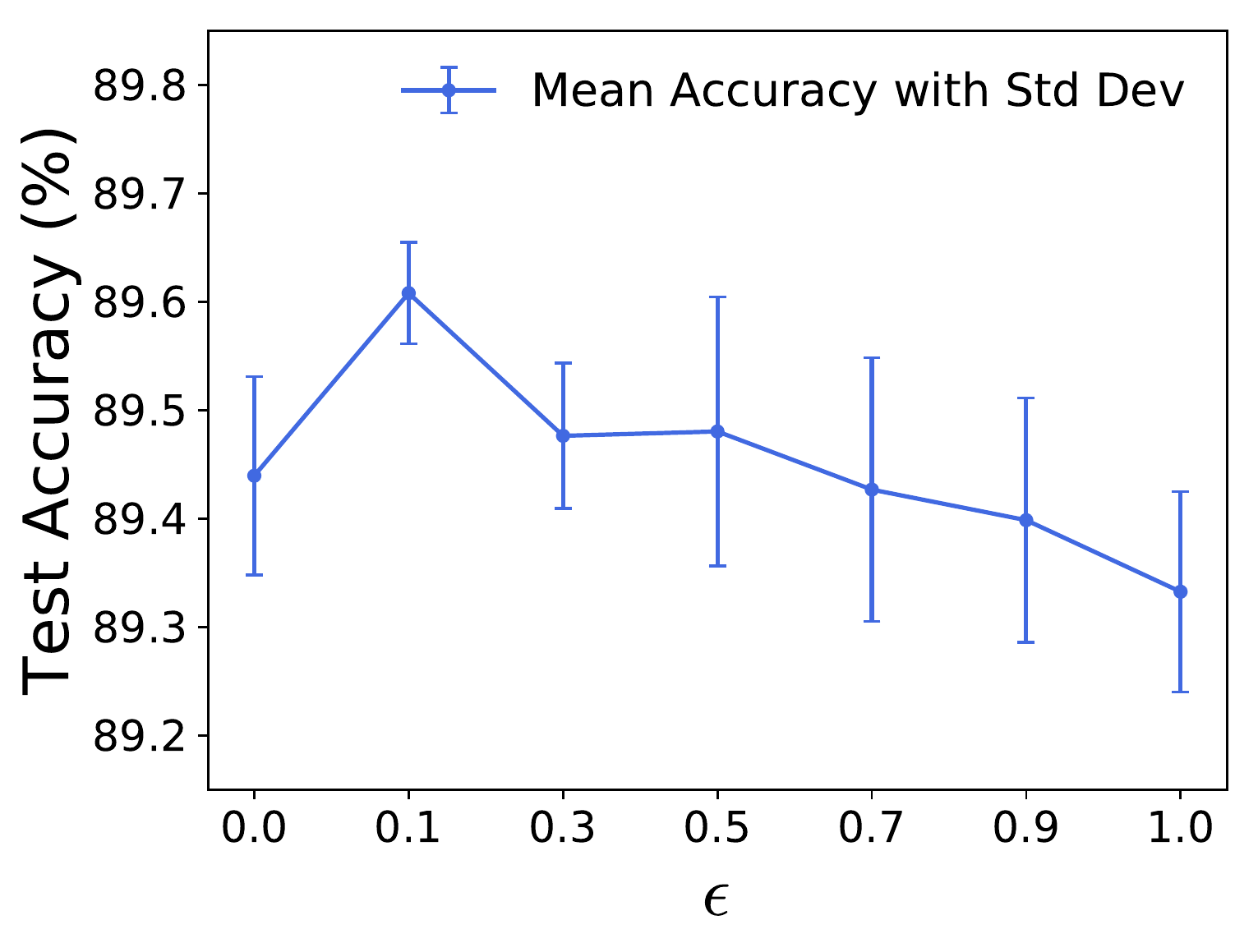}
    \vspace{-10pt}
    \caption{Test accuracy (averaged over five runs) of WRN-28-10 trained on Omniglot with different values of $\epsilon$.}
    \label{fig:epsilon}
    \vspace{-5pt}
\end{figure}

\noindent\textbf{Augmentation Policy Network.} 
To demonstration the contributions of all the components in the policy network, we compare different designs of the policy network. We conduct experiments on the cases that the policy network does not take the transformation embedding as input and the policy network has its own feature extractor. The comparison results of WRN-28-10 trained on CIFAR and Omniglot are shown in Table~\ref{table:ablation}.

First, we observe that the policy network with Transformation Embedding (w.TE)  as input achieves $0.3\%$ higher accuracy than that without TE (o.TE) in average. That means TE contains additional information beyond the images. For example, both the augmented sample and the hard negative in the first row of Figure~\ref{fig:omniglot_visual} look like vertical lines, but can be generated by different transformations (TranslateY and Identity, respectively) and have different labels. With TE as input, the policy network is learned to impose different weights on them. On the other hand, the dimension of TE (28 in our setting) is much lower than that of the image feature (640 in WRN-28-10), so the TE branch hardly increases the computing cost.

Secondly, we evaluate the performance of the policy network with its own feature extractor (own FE) and that shared a common one with the task network (share FE). The latter one performs consistently better than the former one. Also, the former one takes more training time ($1.2\times$ more real running-time) since the feature extraction is repeated twice for the policy network and the task network, respectively.

\begin{table}[h]
    \centering
    \caption{Top-1 test accuracy (\%) of WRN-28-10 with different designs of the policy network.}
    \label{table:ablation}
    %\vspace{-0.2cm}
    \begin{adjustbox}{max width=.47\textwidth}
        \begin{tabular}{lcc|cc}
            \toprule
    		\toprule
    		Dataset & o.TE & w.TE & own FE & share FE \\
    		\midrule
    		CIFAR-10 & 97.58 & \textbf{97.76} & 97.59 & \textbf{97.76} \\
    		CIFAR-100 & 83.49 & \textbf{83.79} & 83.68 & \textbf{83.79} \\
    		Omniglot & 89.29 & \textbf{89.61} & 89.29 & \textbf{89.61} \\
    		\bottomrule
    		\bottomrule
        \end{tabular}
    \end{adjustbox}
    \vspace{-13pt}
\end{table}

\section{Conclusions}

In this paper, a sample-aware augmentation policy network is proposed to reweight augmented samples. We leverage the mechanism of meta-learning and use gradient-based optimization instead of non-differentiable approaches or reinforcement learning, which can balance the learning efficiency and model performance. 
As expected, the learned policy network can distinguish informative augmented images from the junks and thus greatly reduce the noises caused by intensive data augmentation. Extensive experiments demonstrate the superiority of the proposed method to the existing methods using dataset-level augmentation policies.

{%\small
	\bibliography{reference}
}

\clearpage

\onecolumn

\appendix

% \renewcommand\thesubsection{\Alph{subsection}}
% \newtagform{test}{(A.}{)}
% \usetagform{test}
% \graphicspath{{figs/}}
\setcounter{secnumdepth}{10}
\setcounter{theorem}{0}

\section*{MetaAugment: Sample-Aware Data Augmentation Policy Learning \\ Appendix}

\vspace{30pt}

\section{Transferability of MetaAugment}

According to the updating rules, MetaAugment requires three forward and backward passes of the task network in each iteration, which makes it take $3\times$ training time than a standard training of a task network. However, once trained, the policy network $\cV(\cdot,\cdot\,;\bth^{(T)})$, together with the task network $f(\cdot\,;\bw^{(T)})$ and the estimated distribution $\{p_{j,k}\}_{j,k=1}^{K}$ can be transferred to train different networks with data augmentation on the same dataset. The transfer training of MetaAugment takes almost the same computing cost as the regular training and thus is $3\times$ less computing cost than MetaAugment with joint training.
Specifically, let $g(\cdot\,;\bv)$ be a new task network with parameters $\bv$. For each iteration, a mini-batch of training data $\{(x_i, y_i)\}_{i=1}^{n^{tr}}$ is sampled. Also, a mini-batch of transformations $\{\cT_{j_i,k_i}^{m_1,m_2}\}_{i=1}^{n^{tr}}$ is sampled according to $\{p_{j,k}\}_{j,k=1}^{K}$ and the batch data are augmented by the transformations.
Then the update of $\bv$ in iteration $t+1$ is 
%\small
\begin{linenomath}
\begin{equation*}
\bv^{(t+1)} = \bv^{(t)} - \gamma \frac{1}{n^{tr}} \sum_{i=1}^{n^{tr}} \cV_{i}(\bw^{(T)},\bth^{(T)}) \nabla_{\bv} L_{i}(\bv)\big|_{\bv^{(t)}},
\end{equation*}
\end{linenomath}
%\normalsize
where $\cV_{i}(\bw^{(T)},\bth^{(T)})=\cV(f(\cT_{j_i,k_i}^{m_1,m_2}(x_i);\bw^{(T)}),e(\cT_{j_i,k_i}^{m_1,m_2});\bth^{(T)})$ and $L_{i}(\bv)=\ell(g(\cT_{j_i,k_i}^{m_1,m_2}(x_i); \bv),y_i)$. 
Different from~\cite{cubuk2019autoaugment,ho2019population,lim2019fast,zhang2020adversarial}, in which the transferred policies are combinations of image processing functions with fixed magnitudes, MetaAugment can sample transformations with all possible magnitudes and evaluate the effectiveness of different transformations for different samples when training a new network.

\begin{wraptable}{R}{0.45\textwidth}
	\centering
	\vspace{-5pt}
	\caption{Top-1 test accuracy (\%) of WRN-40-2 on CIFAR.}
	\label{table:transfer}
	\vspace{-0.2cm}
	\begin{adjustbox}{max width=0.45\textwidth}
		\begin{tabular}{lcc}
			\toprule
            \toprule
			Dataset & MetaAugment (joint) & MetaAugment (transfer)  \\
			\midrule
			CIFAR-10 & 96.79$\pm$0.06 & 96.82$\pm$0.10  \\
			CIFAR-100 & 80.60$\pm$0.16 & 80.15$\pm$0.09  \\
			\bottomrule
			\bottomrule
		\end{tabular}
	\end{adjustbox}
	\vspace{-0.2cm}
\end{wraptable}

To demonstrate that the learned augmentation policies can be transferred across different task networks, we visualize the distributions of transformations learned with WRN-28-10 and WRN-40-2 on CIFAR-10, CIFAR-100, and Omniglot in Figure~\ref{fig:distribution_appendix}. It can be seen that the learned policies are not sensitive to network architectures. We also compare the results of joint training with the results of transfer training using WRN-40-2 as the task network in Table~\ref{table:transfer}. For the transfer training setting, the transferred policy network is jointly trained with WRN-28-10. Compared with joint training, MetaAugment with transfer training shows similar result on CIFAR-10 and slightly worse result on CIFAR-100. Considering the efficiency of transfer training, we regard it as a good alternative to MetaAugment with joint training. 
In our experiments, Shake-Shake and PyramidNet+ShakeDrop in Table~\ref{table:cifar} are trained with the transferred policy network jointly learned with WRN-28-10 and ResNet-200 in Table~\ref{table:imagenet} is trained with the transferred policy network jointly learned with ResNet-50. The results show the strong transferability of MetaAugment.

\begin{figure}[htbp]
	\begin{subfigure}{.33\textwidth}
		\centering
		\includegraphics[width=.9\linewidth]{figs/heats_cifar10_28_10_bs128_epoch_599.png}
		\caption{WRN-28-10 on CIFAR-10}
		\label{fig:distribution_28_10_cifar10_appendix}
	\end{subfigure}
	\begin{subfigure}{.33\textwidth}
		\centering
		\includegraphics[width=.9\linewidth]{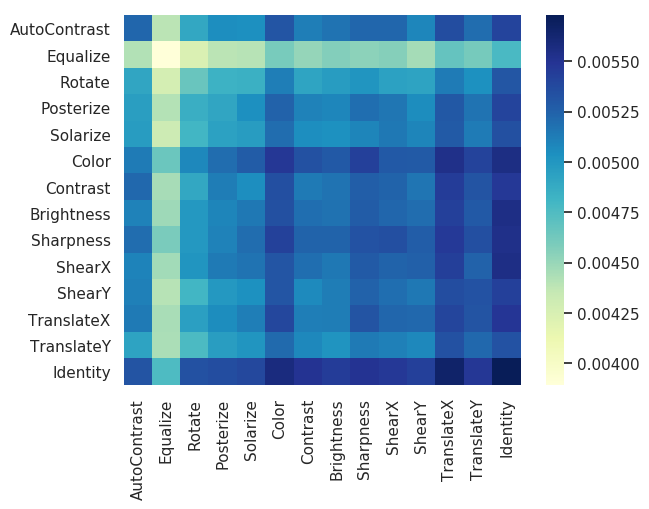}
		\caption{WRN-40-2 on CIFAR-10}
		\label{fig:distribution_40_2_cifar10_appendix}
	\end{subfigure}
	\begin{subfigure}{.33\textwidth}
		%\vspace{5pt}
		\centering
		\includegraphics[width=.9\linewidth]{figs/heats_cifar100_28_10_bs128_epoch_199.png}
		\caption{WRN-28-10 on CIFAR-100}
		\label{fig:distribution_28_10_cifar100_appendix}
	\end{subfigure}
	\vspace{8pt}
	
	\begin{subfigure}{.33\textwidth}
		%\vspace{5pt}
		\centering
		\includegraphics[width=.9\linewidth]{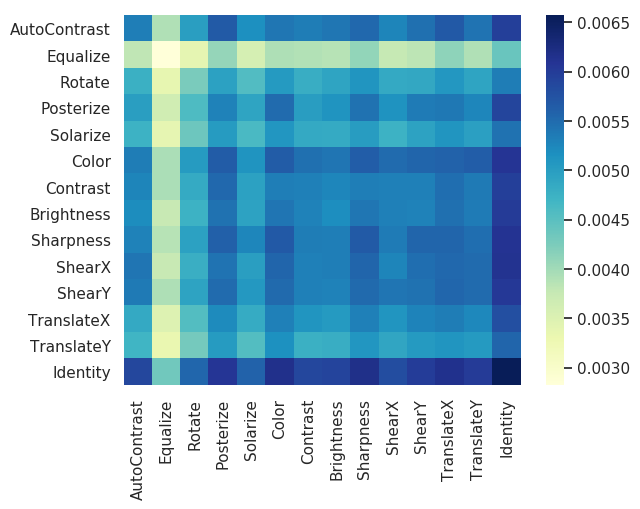}
		\caption{WRN-40-2 on CIFAR-100}
		\label{fig:distribution_40_2_cifar100_appendix}
	\end{subfigure}
	\begin{subfigure}{.33\textwidth}
		%\vspace{5pt}
		\centering
		\includegraphics[width=.9\linewidth]{figs/heats_om_28_10_bs128_ep01_epoch_199.png}
		\caption{WRN-28-10 on Omniglot}
		\label{fig:distribution_28_10_omniglot_appendix}
	\end{subfigure}
	\begin{subfigure}{.33\textwidth}
		%\vspace{5pt}
		\centering
		\includegraphics[width=.9\linewidth]{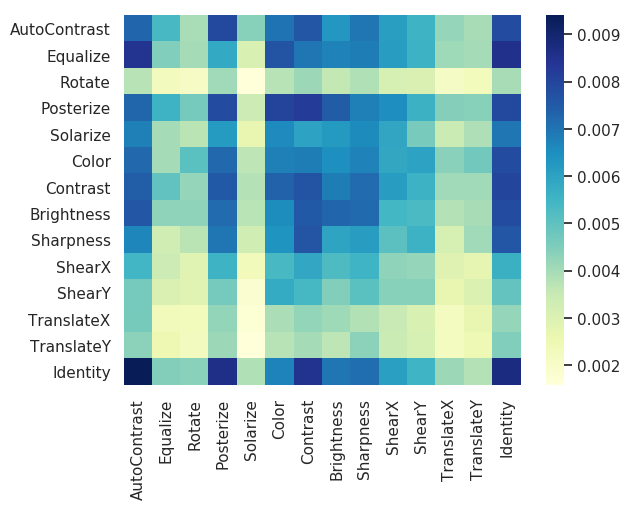}
		\caption{WRN-40-2 on Omniglot}
		\label{fig:distribution_40_2_omniglot_appendix}
	\end{subfigure}
	\caption{Distributions of transformations learned with WRN-28-10 and WRN-40-2 on CIFAR-10, CIFAR-100, and Omniglot.}
	\label{fig:distribution_appendix}
	%\vspace{-13pt}
\end{figure}

\section{Implementation Details of MetaAugment}

\noindent\textbf{CIFAR-10.} On CIFAR-10, WRN-28-10 and WRN-40-2 are jointly trained with the policy network. Both of them are trained for 600 epochs by SGD with a momentum 0.9, weight decay of $5\times 10^{-4}$, batch size of 128, initial learning rate of 0.1, and cosine learning rate decay. The policy network is trained for the same number of iterations as the task network by SGD with a momentum 0.9, weight decay of $5\times 10^{-4}$, batch size of 128, and a fixed learning rate $1\times 10^{-3}$. The distribution of transformations is updated every one epoch with the outputs of the policy network in the last 50 epochs. The hyper-parameter $\epsilon$ that determines the probability of random sampling transformations is set to 0.1. We adopt the transfer training of MetaAugment for Shake-Shake (26 2x96d), Shake-Shake (26 2x112d), and PyramidNet+ShakeDrop. Shake-Shake (26 2x96d) and Shake-Shake (26 2x112d) are trained for 1,800 epochs by SGD with a momentum 0.9, weight decay of $1\times 10^{-3}$, batch size of 128, initial learning rate of 0.01, and cosine learning rate decay. PyramidNet+ShakeDrop is trained for 1,800 epochs by SGD with a momentum 0.9, weight decay of $1\times 10^{-4}$, batch size of 128, initial learning rate of 0.1, and cosine learning rate decay. 

\noindent\textbf{CIFAR-100.} On CIFAR-100, WRN-28-10 and WRN-40-2 are also jointly trained with the policy network. They are trained with the same hyper-parameters as those used on CIFAR-10 except that WRN-28-10 is trained for 200 epochs and the corresponding distribution of transformations is updated with the outputs of the policy network in the last 20 epochs. For the transfer training of MetaAugment, Shake-Shake (26 2x96d) is trained for 1,800 epochs by SGD with a momentum 0.9, weight decay of $2.5\times 10^{-3}$, batch size of 128, initial learning rate of 0.01, and cosine learning rate decay. PyramidNet+ShakeDrop is trained for 1,800 epochs by SGD with a momentum 0.9, weight decay of $5\times 10^{-4}$, batch size of 128, initial learning rate of 0.05, and cosine learning rate decay.

\noindent\textbf{Omniglot.} For MetaAugment, WRN-28-10 and WRN-40-2 are trained for 200 epochs and the distribution of transformations is updated with the outputs of the policy network in the last 20 epochs. Other hyper-parameters are the same as those used for WRN-28-10 and WRN-40-2 on CIFAR-10. For RA, we use the same implementation as MetaAugment except that we randomly sample transformations and adopt the same weight for augmented samples. For FAA, we follow the setting that the augmentation policy is searched directly on the full dataset (without the test data) given the task network. We use the same image processing functions as MetaAugment and follow the hyper-parameter setting in FAA when searching for the policy. For PBA, we follow the setting that the augmentation schedules are searched on a reduced Omniglot of 3,246 training images, two images for each character. We also use the same image processing functions as MetaAugment and run 16 trials on WRN-40-2 to generate the augmentation schedules. When using the searched policies to finally train WRN-28-10 and WRN-40-2, we adopt the same hyper-parameters as those used by MetaAugment for FAA and PBA.

\noindent\textbf{ImageNet.} On ImageNet, ResNet-50 is jointly trained with the policy network. It is trained for 120 epochs by SGD with a momentum 0.9, weight decay of $1\times 10^{-4}$, batch size of $1024 \cdot 4$ (each training sample in a mini-batch is augmented by 4 transformations), initial learning rate of 0.4, and cosine learning rate decay. We also use a gradual warmup strategy that increases the learning rate from 0.08 to 0.4 linearly in the first 5 epochs, label smoothing with magnitude 0.1, and gradient clipping with magnitude 5. Since ImageNet is a very challenging dataset, we first train ResNet-50 with random data augmentation for 40 epochs to get a pre-trained feature extractor and then train it with the policy network jointly for 40 epochs. Finally, we fix the policy network and train ResNet-50 as the transfer training of MetaAugment for 40 epochs. The policy network is trained by SGD with a momentum 0.9, weight decay of $5\times 10^{-4}$, batch size of 2048, a fixed learning rate $4\times 10^{-3}$, and gradient clipping with magnitude 5. The distribution of transformations is updated every one epoch with the outputs of the policy network in the last 20 epochs. The hyper-parameter $\epsilon$ is set to 0.1. For the transfer training of MetaAugment, ResNet-200 is trained with the same setting as that of ResNet-50 except that it is trained for 150 epochs with batch size of 512 and initial learning rate of 0.2.

\section{Convergence Tendency of Loss Curve}
To illustrate the convergence properties of our algorithm, we plot the weighted training and validation loss curves of WRN-28-10 and WRN-40-2 trained on CIFAR-10, CIFAR-100, and Omniglot in Figure~\ref{fig:curves_appendix}. It can be seen that the validation loss curves converge in all cases, which is consistent with our theorem. In addition, the training loss curves of WRN-28-10 and WRN-40-2 trained on CIFAR-10 and CIFAR-100 also converge, while the training loss curves fluctuate on Omniglot. As stated in Theorem~\ref{theorem2}, the convergence of the training loss is not ensured since the policy network depends on the feature extractor of the task network. These empirical results further confirm the theoretical results.

\begin{figure}[htbp]
	\begin{subfigure}{.33\textwidth}
		\centering
		\includegraphics[width=.9\linewidth]{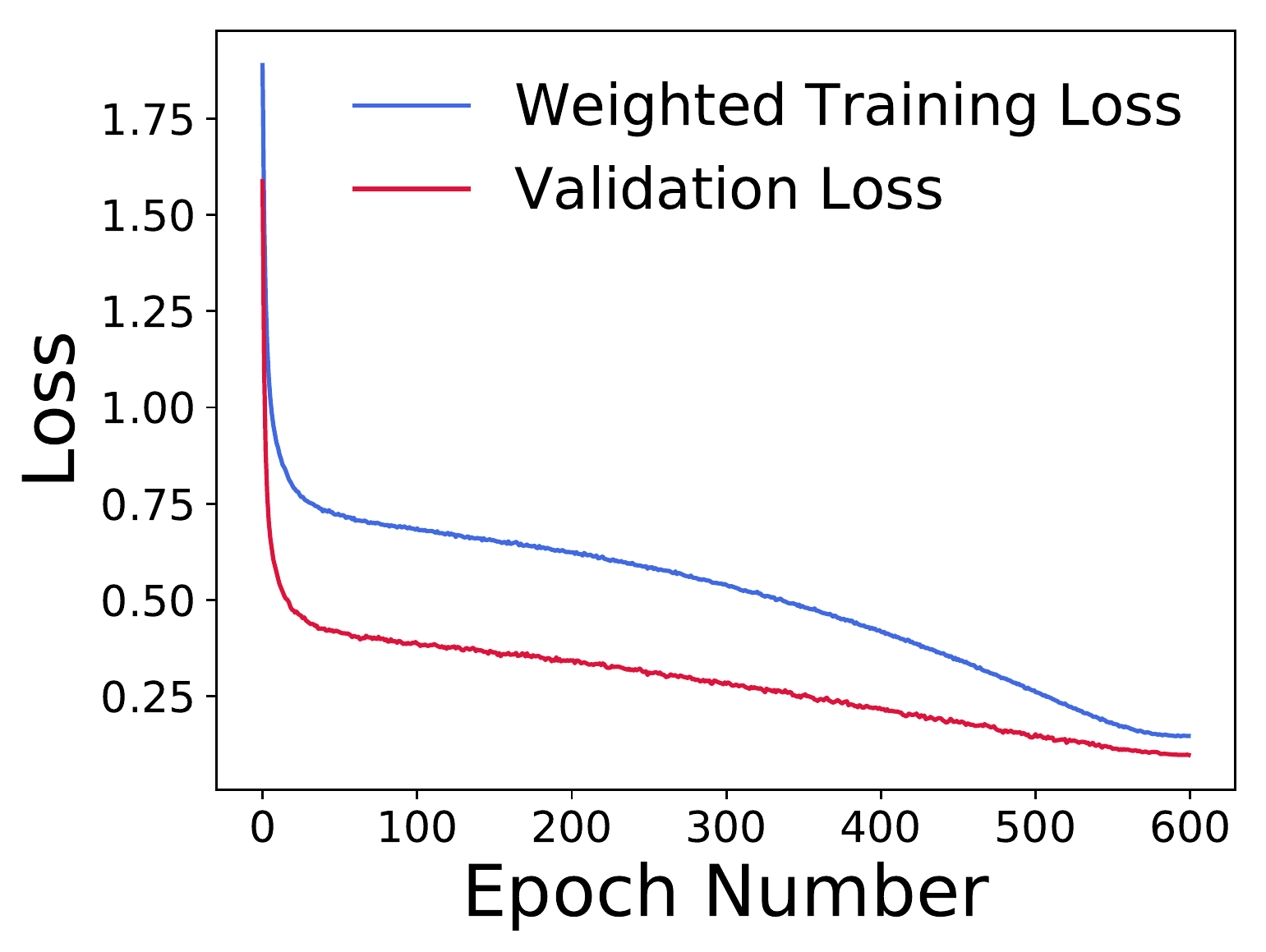}
		\caption{WRN-28-10 on CIFAR-10}
		\label{fig:curve_28_10_cifar10_appendix}
	\end{subfigure}
	\begin{subfigure}{.33\textwidth}
		\centering
		\includegraphics[width=.9\linewidth]{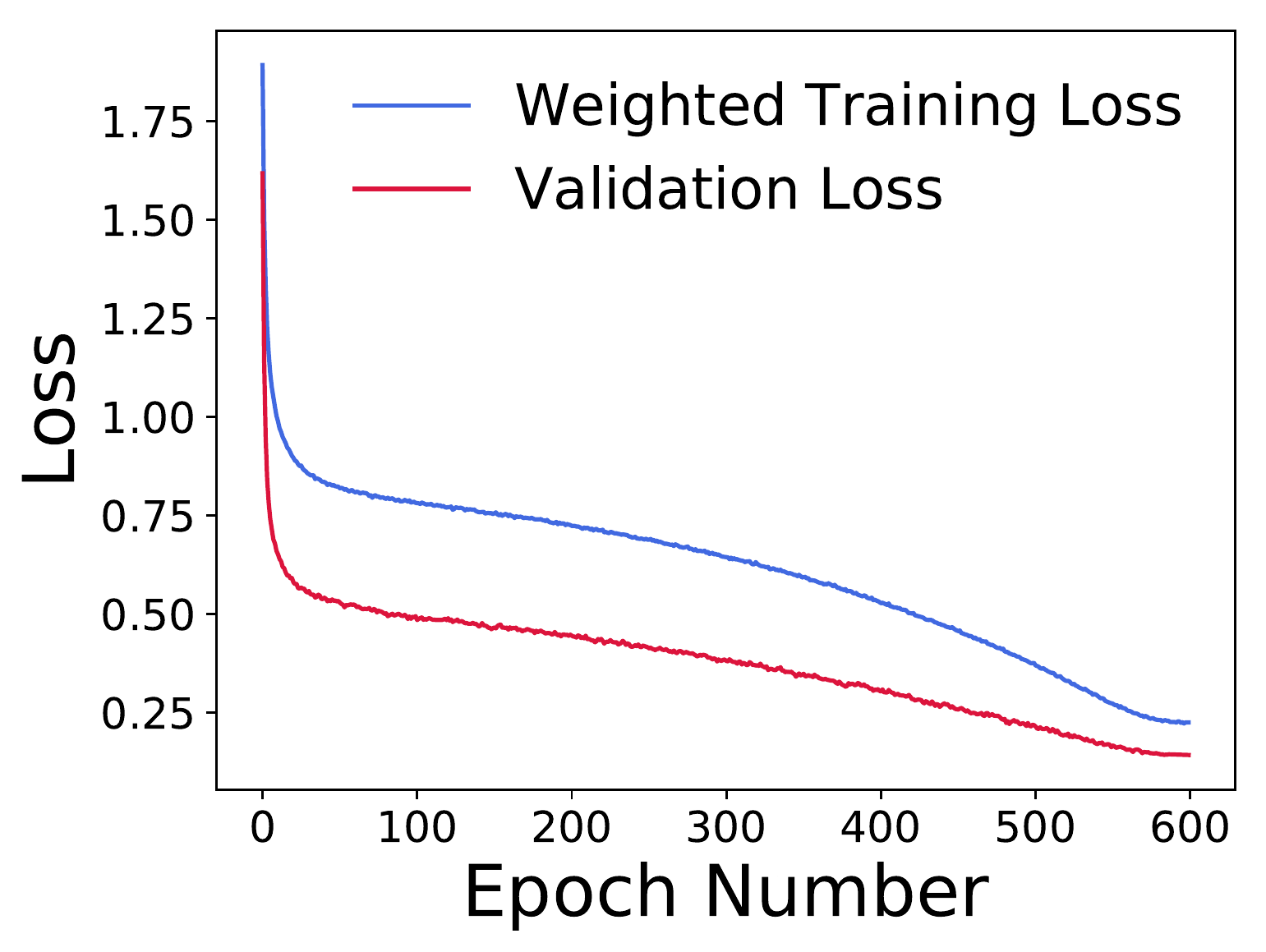}
		\caption{WRN-40-2 on CIFAR-10}
		\label{fig:curve_40_2_cifar10_appendix}
	\end{subfigure}
	\begin{subfigure}{.33\textwidth}
		%\vspace{5pt}
		\centering
		\includegraphics[width=.9\linewidth]{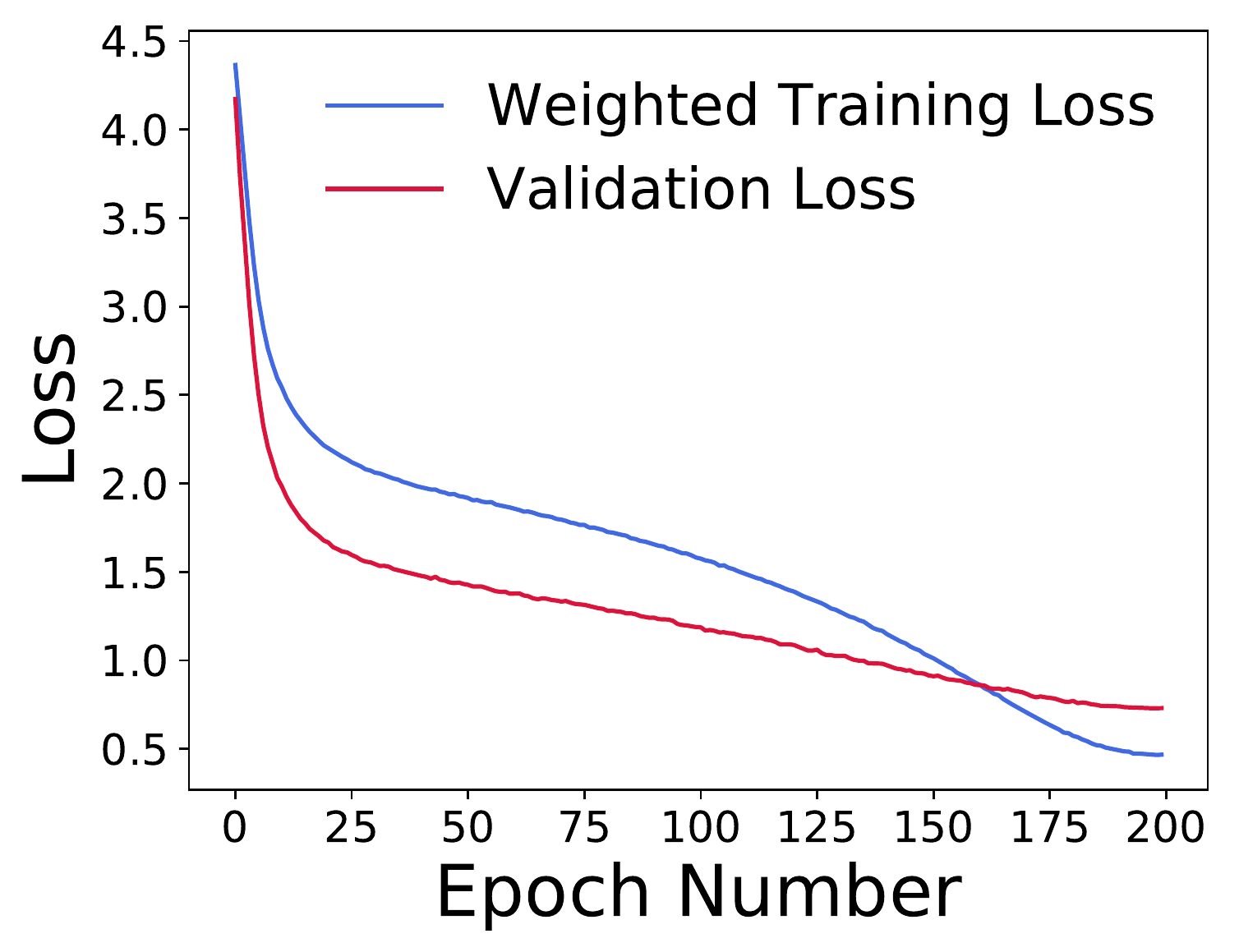}
		\caption{WRN-28-10 on CIFAR-100}
		\label{fig:curve_28_10_cifar100_appendix}
	\end{subfigure}
	\vspace{8pt}
	
	\begin{subfigure}{.33\textwidth}
		%\vspace{5pt}
		\centering
		\includegraphics[width=.9\linewidth]{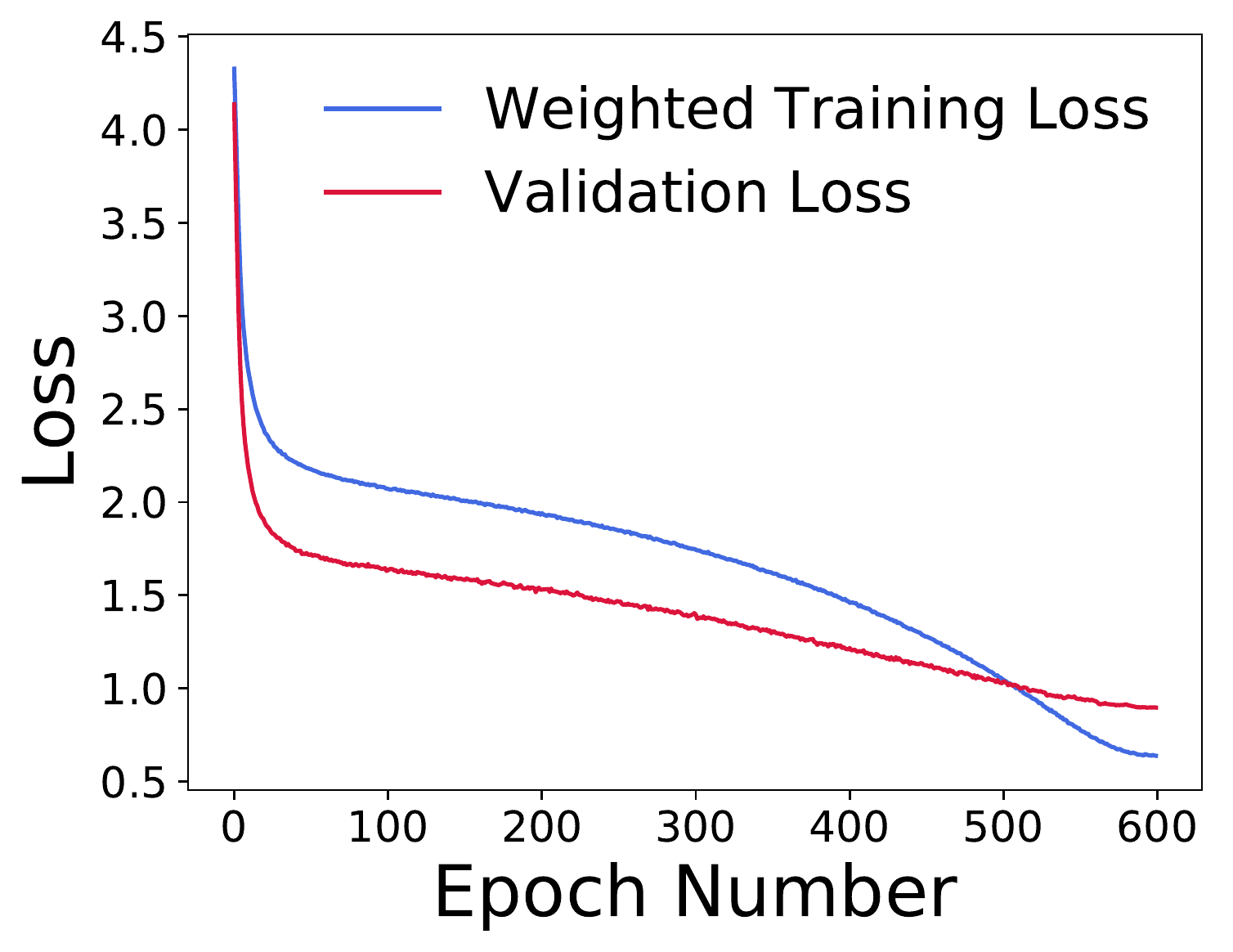}
		\caption{WRN-40-2 on CIFAR-100}
		\label{fig:curve_40_2_cifar100_appendix}
	\end{subfigure}
	\begin{subfigure}{.33\textwidth}
		%\vspace{5pt}
		\centering
		\includegraphics[width=.9\linewidth]{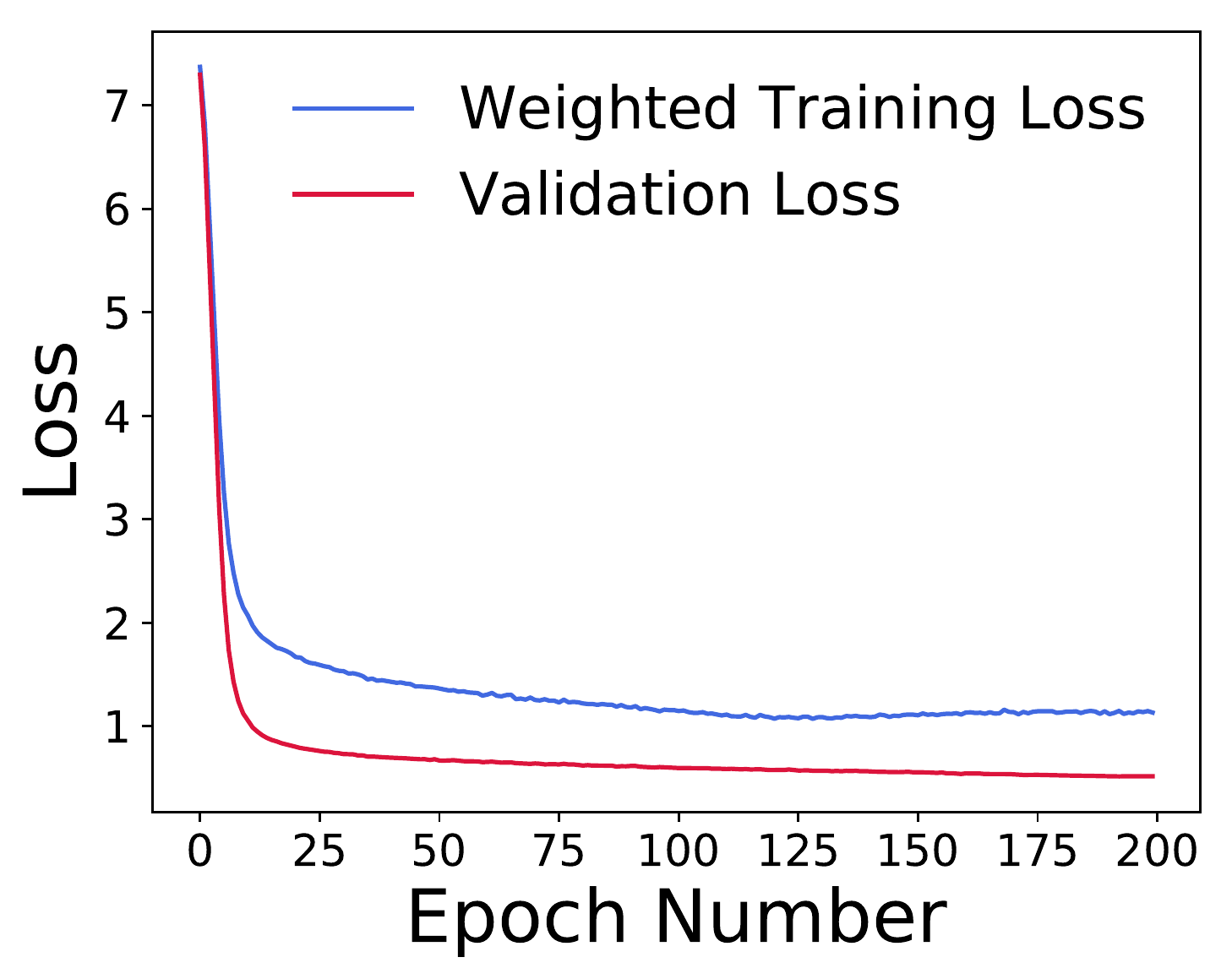}
		\caption{WRN-28-10 on Omniglot}
		\label{fig:curve_28_10_omniglot_appendix}
	\end{subfigure}
	\begin{subfigure}{.33\textwidth}
		%\vspace{5pt}
		\centering
		\includegraphics[width=.9\linewidth]{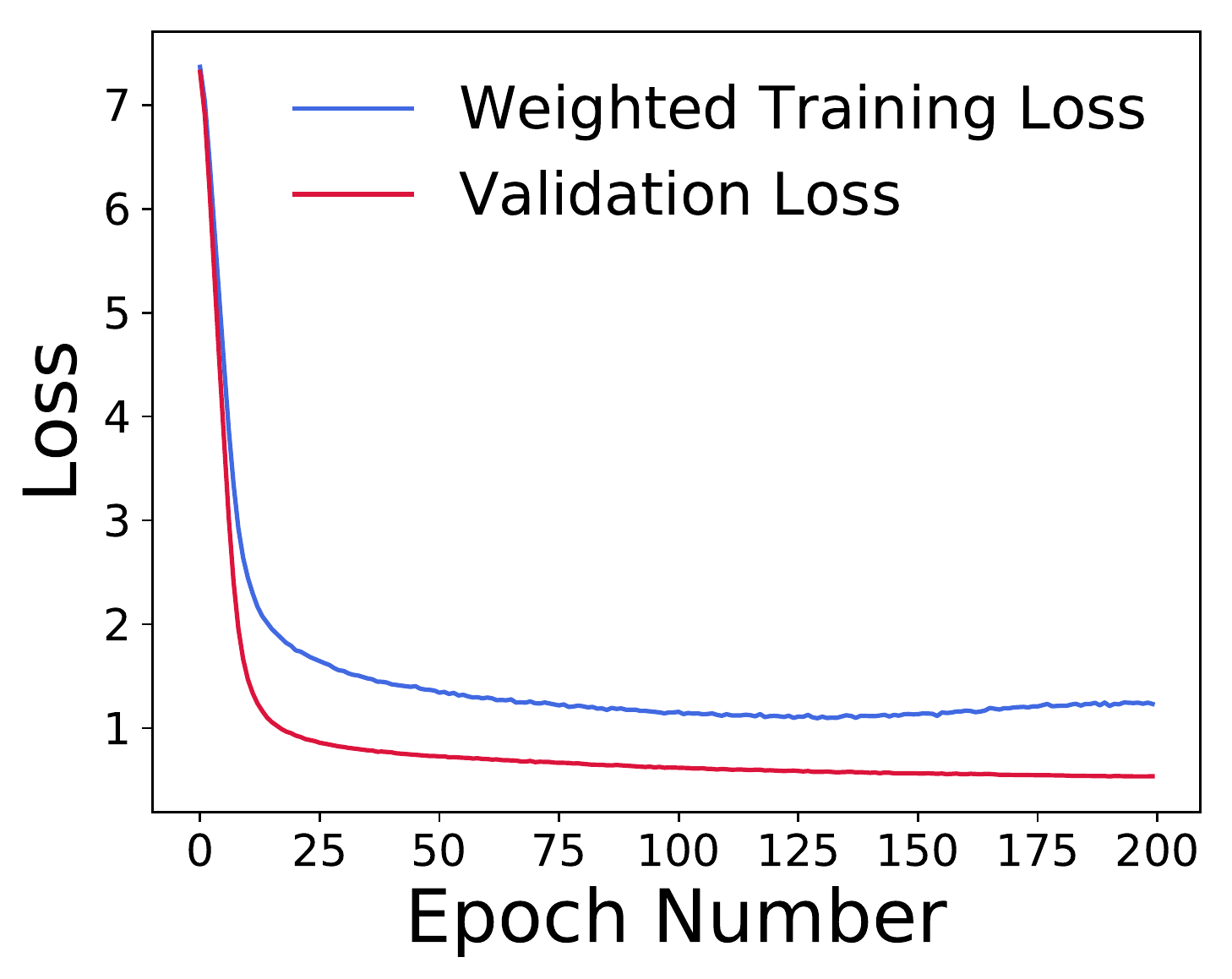}
		\caption{WRN-40-2 on Omniglot}
		\label{fig:curve_40_2_omniglot_appendix}
	\end{subfigure}
	\caption{Weighted training and validation loss curves of WRN-28-10 and WRN-40-2 trained on CIFAR-10, CIFAR-100, and Omniglot.}
	\label{fig:curves_appendix}
	%\vspace{-13pt}
\end{figure}

\section{More Examples of Augmented Samples}
We display more augmented samples with high and low learned weights in Figure~\ref{fig:high_weights_appendix} and in Figure~\ref{fig:low_weights_appendix}, respectively. These images further illustrate the necessity and effectiveness of the learned policy network.

\begin{figure}[t]
\centering
\begin{tabular}{cm{1.67cm}m{1.67cm}m{1.67cm}m{1.67cm}m{1.67cm}m{1.67cm}}
{\tiny Original} & {\includegraphics[scale=0.093]{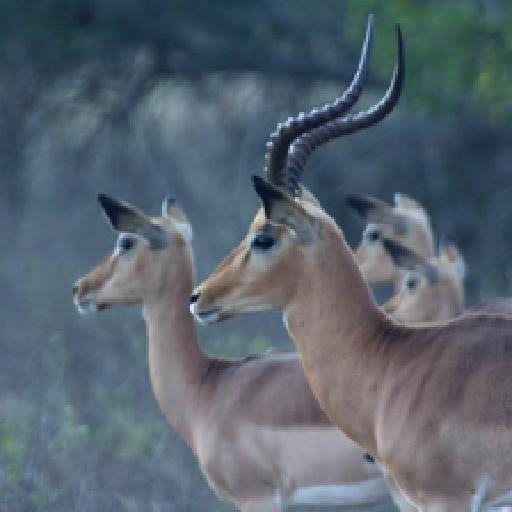}} & {\includegraphics[scale=0.093]{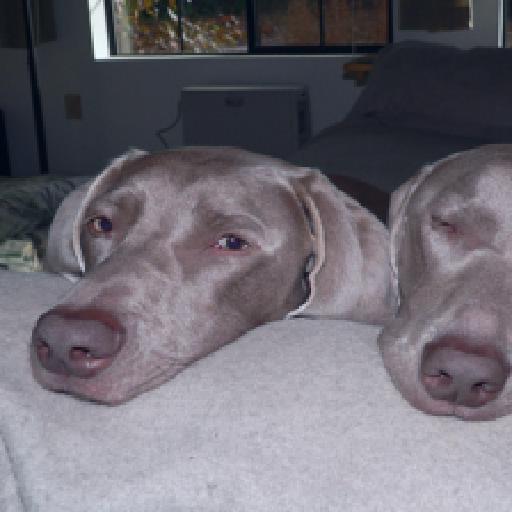}} & {\includegraphics[scale=0.093]{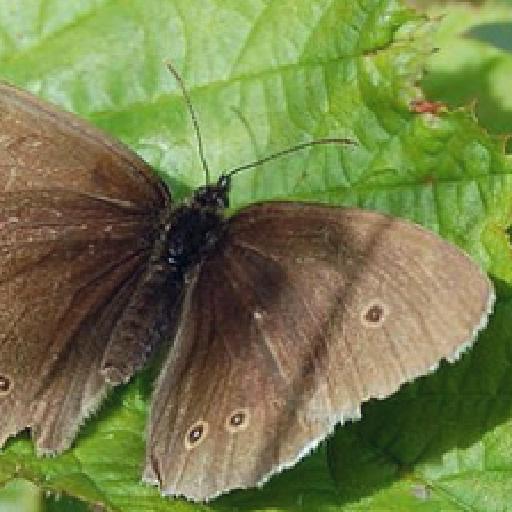}} & {\includegraphics[scale=0.093]{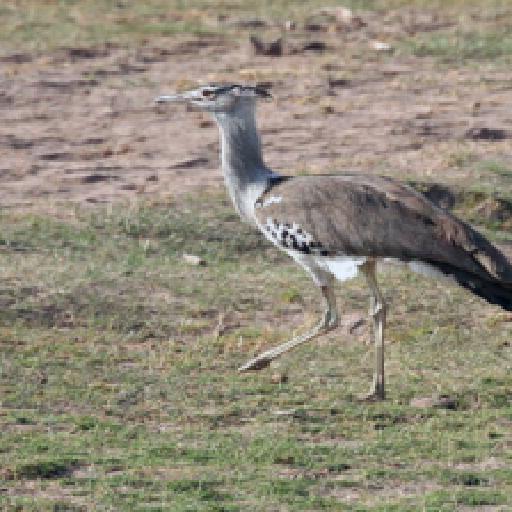}} & {\includegraphics[scale=0.093]{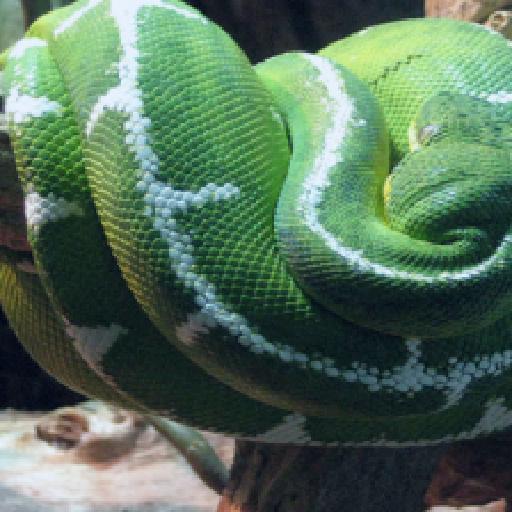}} & {\includegraphics[scale=0.093]{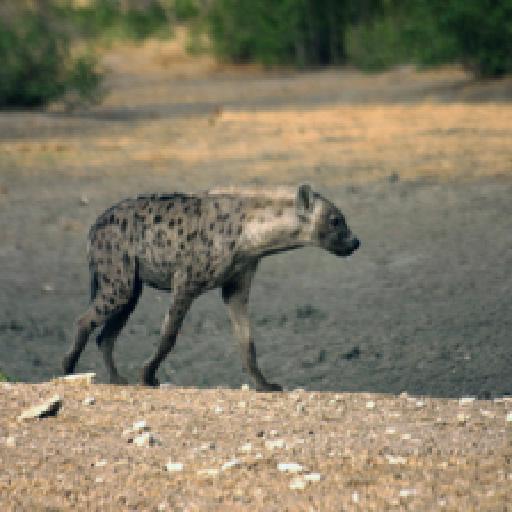}} \\
{\tiny Augmented} & {\includegraphics[scale=0.093]{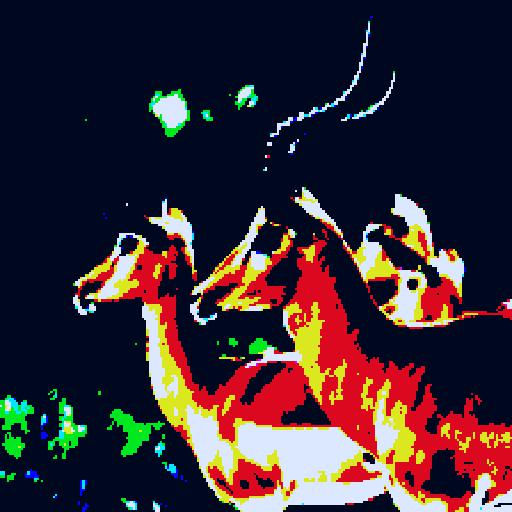}} & {\includegraphics[scale=0.093]{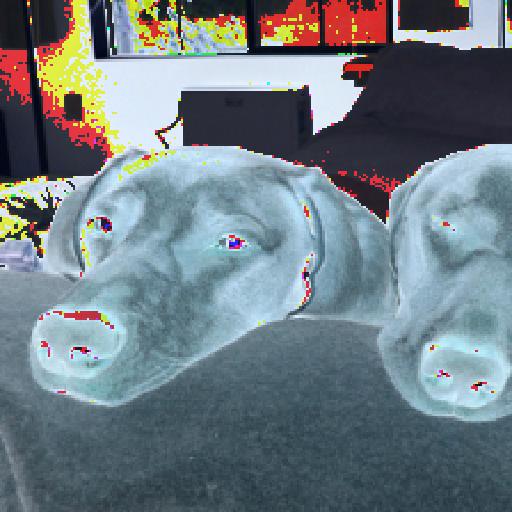}} & {\includegraphics[scale=0.093]{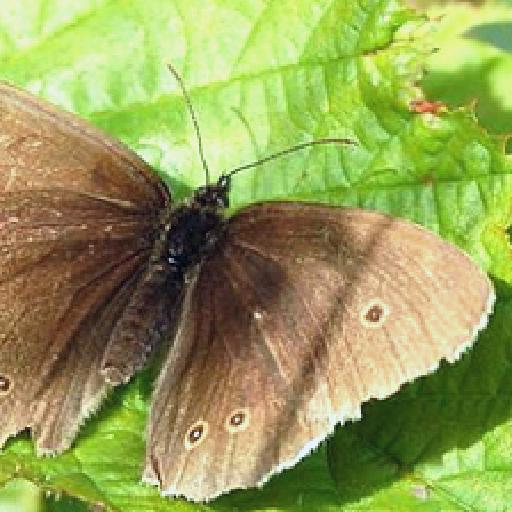}} & {\includegraphics[scale=0.093]{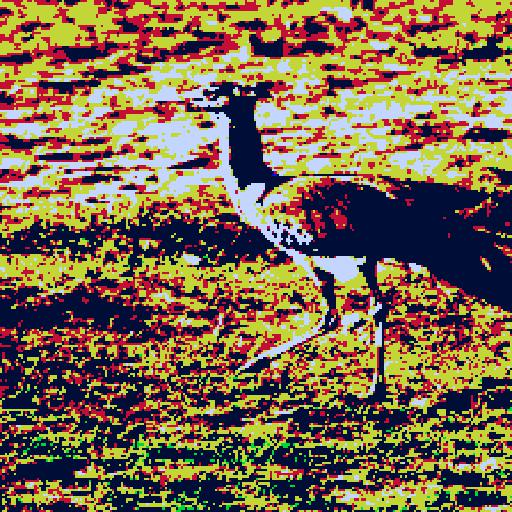}} & {\includegraphics[scale=0.093]{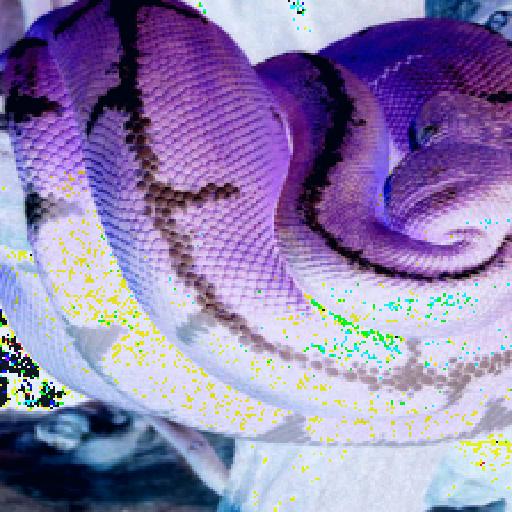}} & {\includegraphics[scale=0.093]{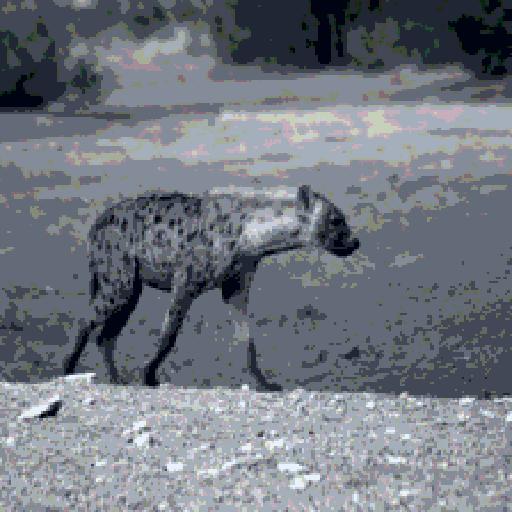}} \\
& \multicolumn{1}{c}{\tiny AutoContrast}  & \multicolumn{1}{c}{\tiny AutoContrast}  & \multicolumn{1}{c}{\tiny Brightness (m=7.3)} & \multicolumn{1}{c}{\tiny Brightness (m=1.8)} & \multicolumn{1}{c}{\tiny Brightness (m=5.6)} & \multicolumn{1}{c}{\tiny Color (m=0.2)} \\
& \multicolumn{1}{c}{\tiny Posterize (m=0.4)}  & \multicolumn{1}{c}{\tiny Solarize (m=1.8)}  & \multicolumn{1}{c}{\tiny Identity} & \multicolumn{1}{c}{\tiny Posterize (m=6.1)} & \multicolumn{1}{c}{\tiny Solarize (m=0.9)} & \multicolumn{1}{c}{\tiny Posterize (m=7.6))} \\
& \multicolumn{1}{c}{\tiny Weight=0.45}  & \multicolumn{1}{c}{\tiny Weight=0.44}  & \multicolumn{1}{c}{\tiny Weight=0.44} & \multicolumn{1}{c}{\tiny Weight=0.46} & \multicolumn{1}{c}{\tiny Weight=0.46} & \multicolumn{1}{c}{\tiny Weight=0.44} \\
{\tiny Original} & {\includegraphics[scale=0.093]{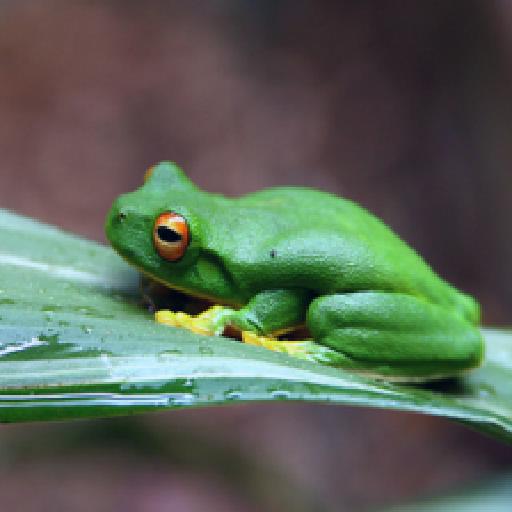}} & {\includegraphics[scale=0.093]{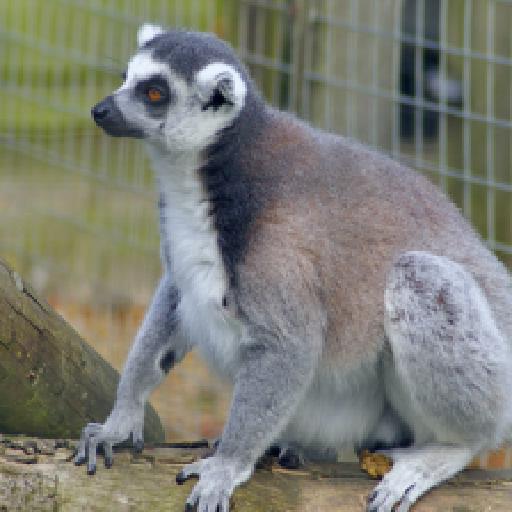}} & {\includegraphics[scale=0.093]{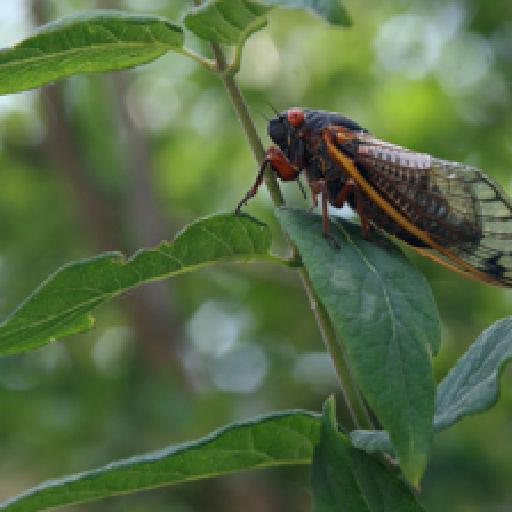}} & {\includegraphics[scale=0.093]{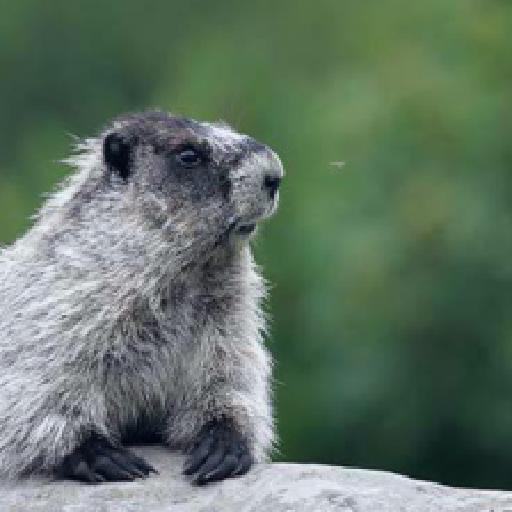}} & {\includegraphics[scale=0.093]{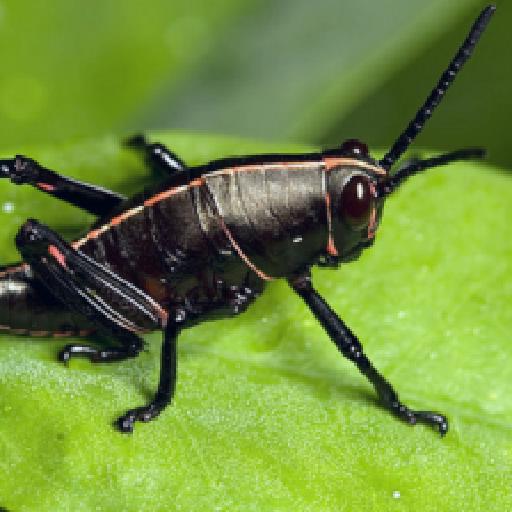}} & {\includegraphics[scale=0.093]{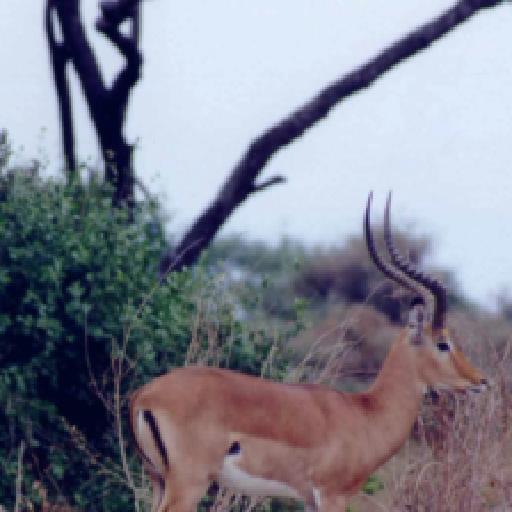}} \\
{\tiny Augmented} & {\includegraphics[scale=0.093]{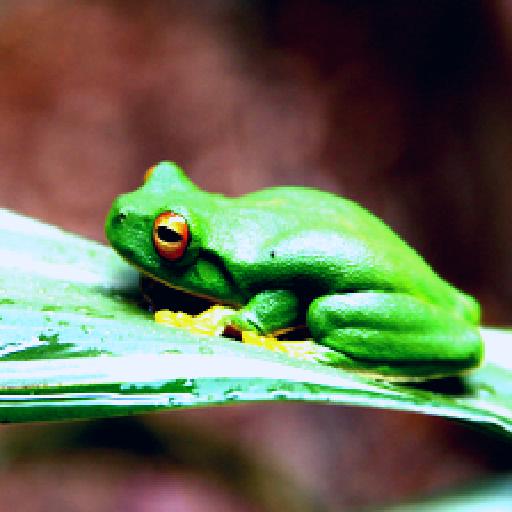}} & {\includegraphics[scale=0.093]{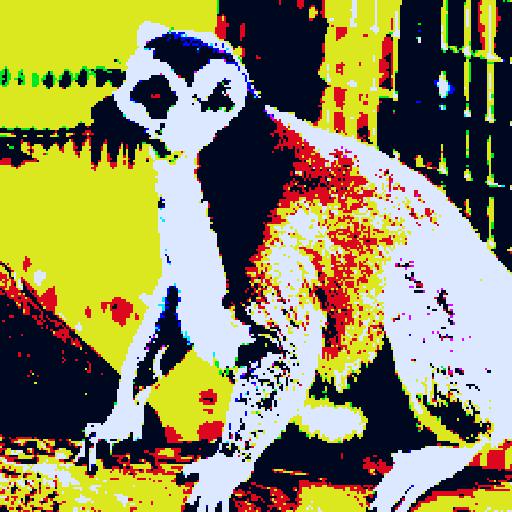}} & {\includegraphics[scale=0.093]{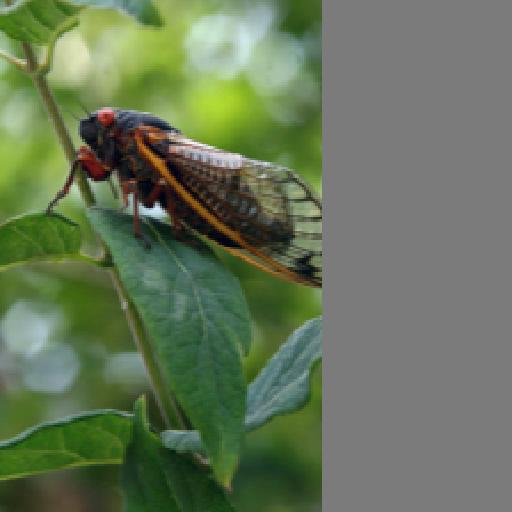}} & {\includegraphics[scale=0.093]{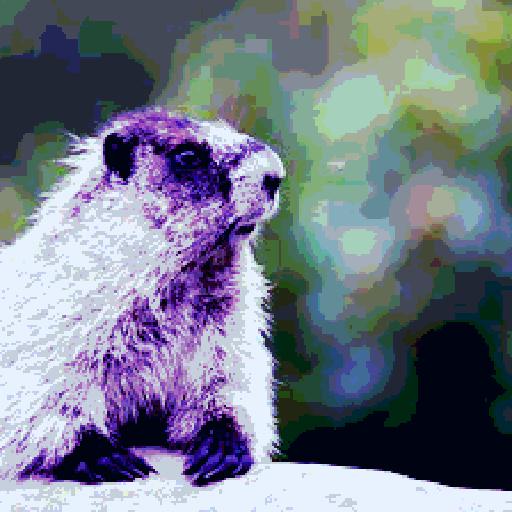}} & {\includegraphics[scale=0.093]{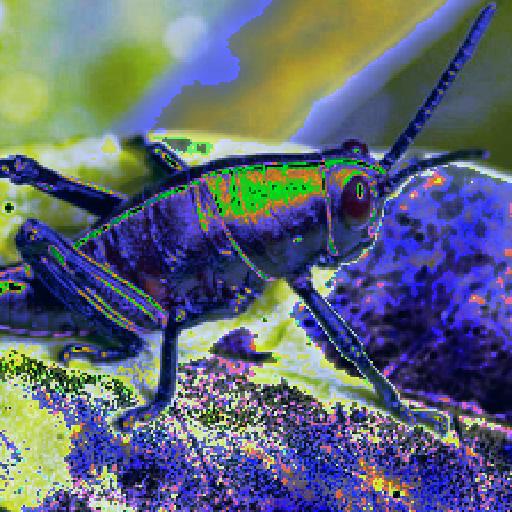}} & {\includegraphics[scale=0.093]{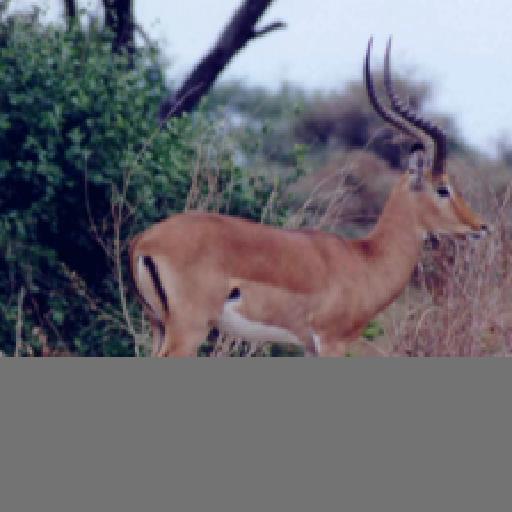}} \\
& \multicolumn{1}{c}{\tiny Contrast (m=9.2)}  & \multicolumn{1}{c}{\tiny Contrast (m=5.7)}  & \multicolumn{1}{c}{\tiny Contrast (m=6.2)} & \multicolumn{1}{c}{\tiny Equalize} & \multicolumn{1}{c}{\tiny Equalize} & \multicolumn{1}{c}{\tiny Identity} \\
& \multicolumn{1}{c}{\tiny Brightness (m=5.7)}  & \multicolumn{1}{c}{\tiny Posterize (m=2.0)}  & \multicolumn{1}{c}{\tiny TranslateX (m=8.2)} & \multicolumn{1}{c}{\tiny Posterize (m=7.9)} & \multicolumn{1}{c}{\tiny Solarize (m=6.8)} & \multicolumn{1}{c}{\tiny TranslateY (m=6.7))} \\
& \multicolumn{1}{c}{\tiny Weight=0.42}  & \multicolumn{1}{c}{\tiny Weight=0.44}  & \multicolumn{1}{c}{\tiny Weight=0.43} & \multicolumn{1}{c}{\tiny Weight=0.44} & \multicolumn{1}{c}{\tiny Weight=0.39} & \multicolumn{1}{c}{\tiny Weight=0.45} \\
{\tiny Original} & {\includegraphics[scale=0.093]{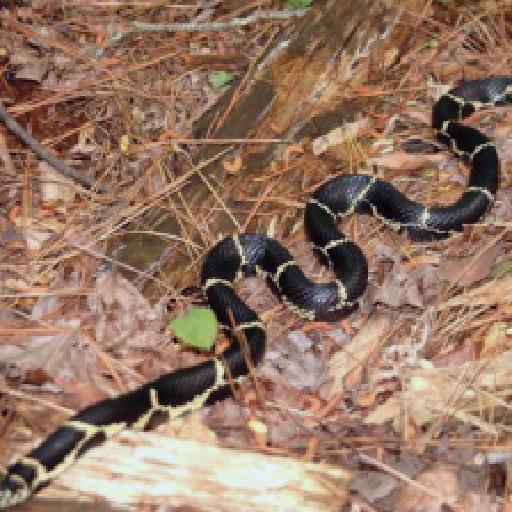}} & {\includegraphics[scale=0.093]{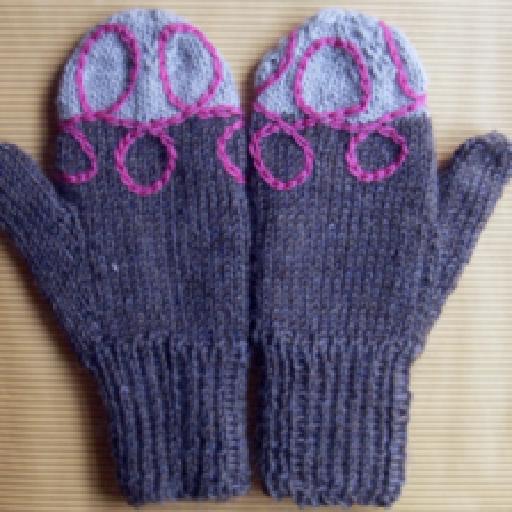}} & {\includegraphics[scale=0.093]{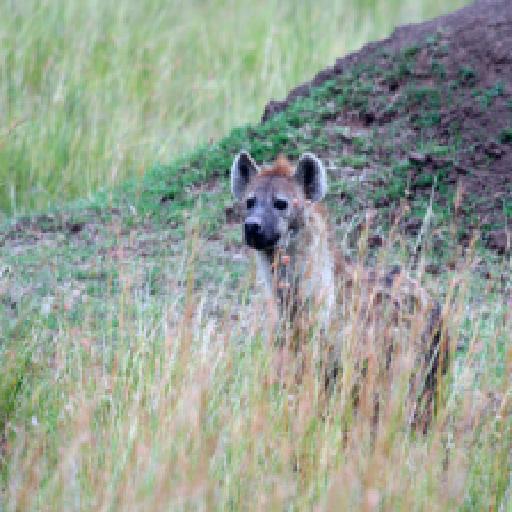}} & {\includegraphics[scale=0.093]{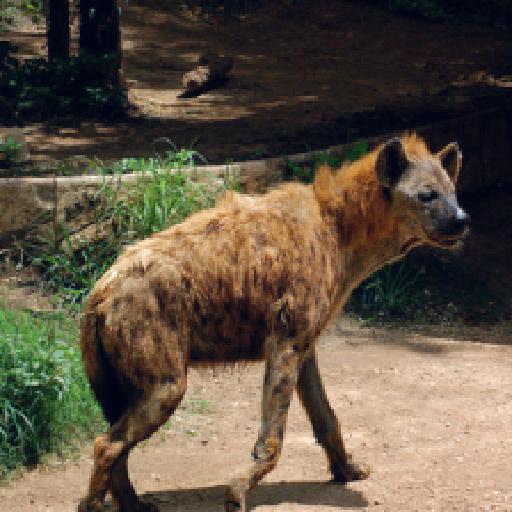}} & {\includegraphics[scale=0.093]{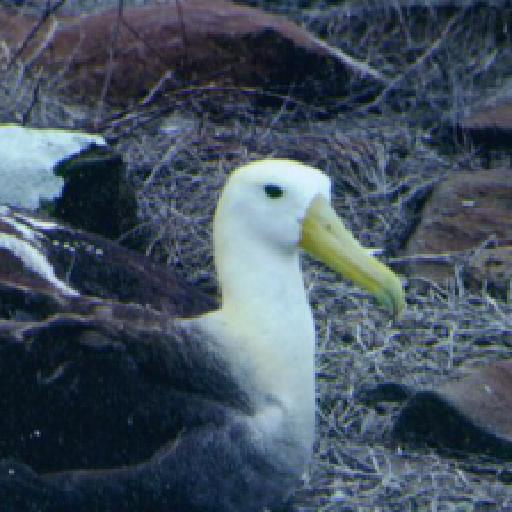}} & {\includegraphics[scale=0.093]{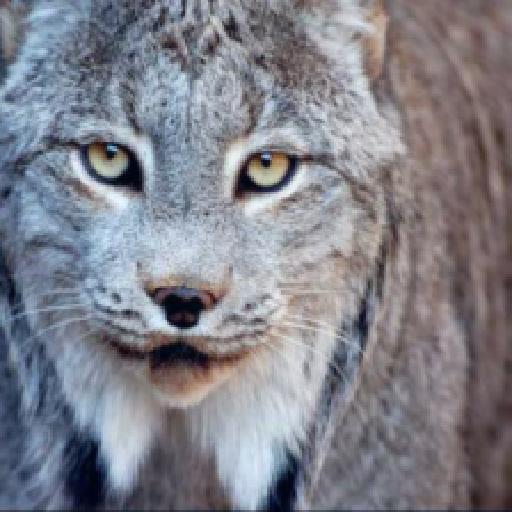}} \\
{\tiny Augmented} & {\includegraphics[scale=0.093]{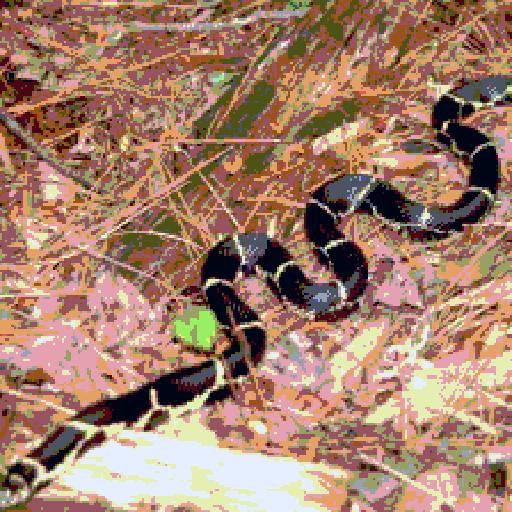}} & {\includegraphics[scale=0.093]{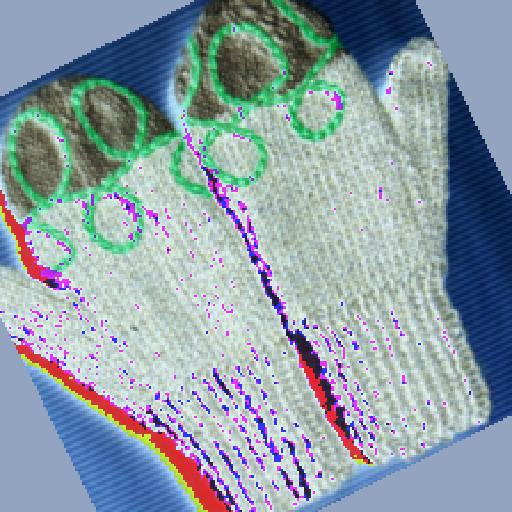}} & {\includegraphics[scale=0.093]{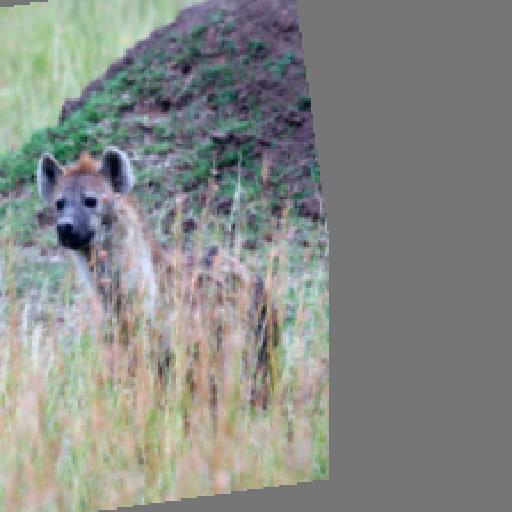}} & {\includegraphics[scale=0.093]{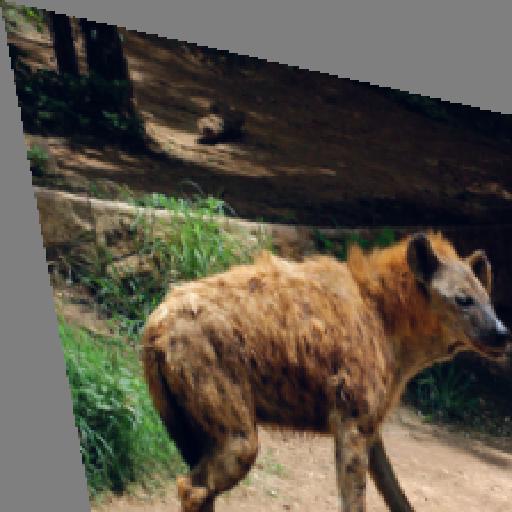}} & {\includegraphics[scale=0.093]{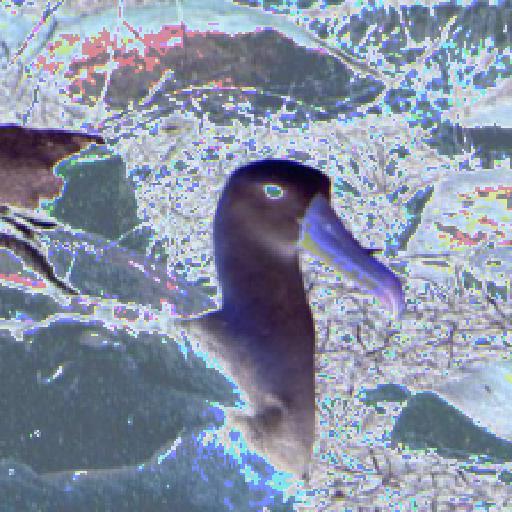}} & {\includegraphics[scale=0.093]{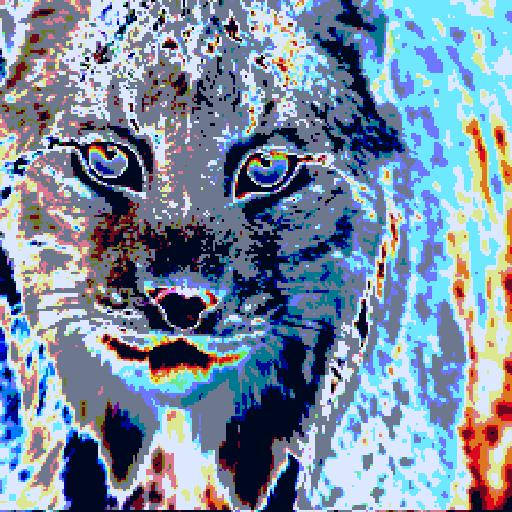}} \\
& \multicolumn{1}{c}{\tiny Posterize (m=6.1)}  & \multicolumn{1}{c}{\tiny Rotate (m=8.8)}  & \multicolumn{1}{c}{\tiny Rotate (m=2.6)} & \multicolumn{1}{c}{\tiny ShearX (m=6.2)} & \multicolumn{1}{c}{\tiny Solarize (m=3.9)} & \multicolumn{1}{c}{\tiny Solarize (m=3.1)} \\
& \multicolumn{1}{c}{\tiny Identity}  & \multicolumn{1}{c}{\tiny Solarize (m=1.9)}  & \multicolumn{1}{c}{\tiny TranslateX (m=7.9)} & \multicolumn{1}{c}{\tiny ShearY (m=7.5)} & \multicolumn{1}{c}{\tiny Identity} & \multicolumn{1}{c}{\tiny Posterize (m=5.4))} \\
& \multicolumn{1}{c}{\tiny Weight=0.45}  & \multicolumn{1}{c}{\tiny Weight=0.44}  & \multicolumn{1}{c}{\tiny Weight=0.41} & \multicolumn{1}{c}{\tiny Weight=0.43} & \multicolumn{1}{c}{\tiny Weight=0.46} & \multicolumn{1}{c}{\tiny Weight=0.44} \\
\end{tabular}
\caption{Examples of augmented samples with high weights.}
\label{fig:high_weights_appendix}
\end{figure}

\begin{figure}[t]
\centering
\begin{tabular}{cm{1.67cm}m{1.67cm}m{1.67cm}m{1.67cm}m{1.67cm}m{1.67cm}}
{\tiny Original} & {\includegraphics[scale=0.093]{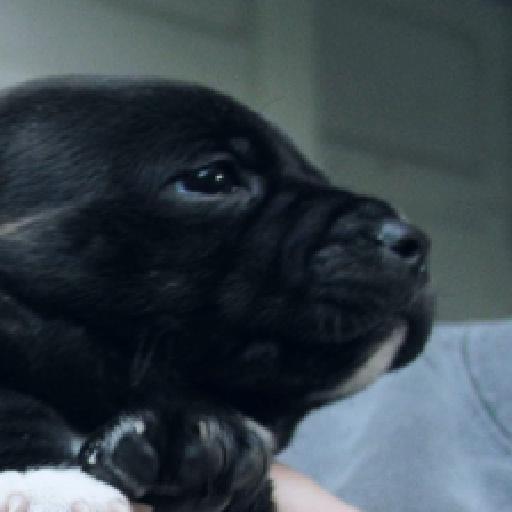}} & {\includegraphics[scale=0.093]{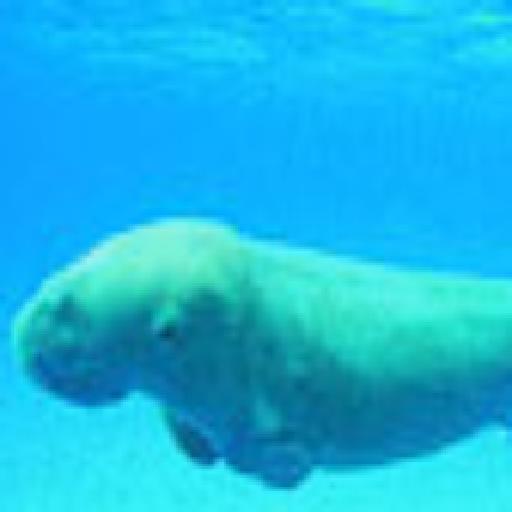}} & {\includegraphics[scale=0.093]{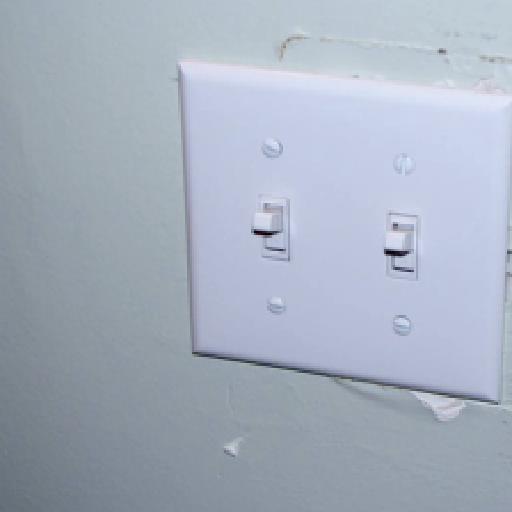}} & {\includegraphics[scale=0.093]{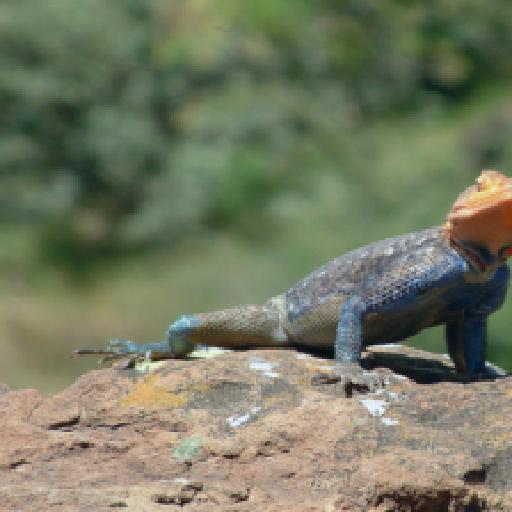}} & {\includegraphics[scale=0.093]{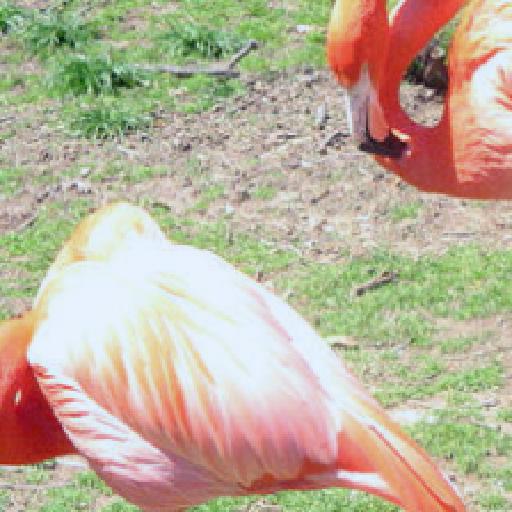}} & {\includegraphics[scale=0.093]{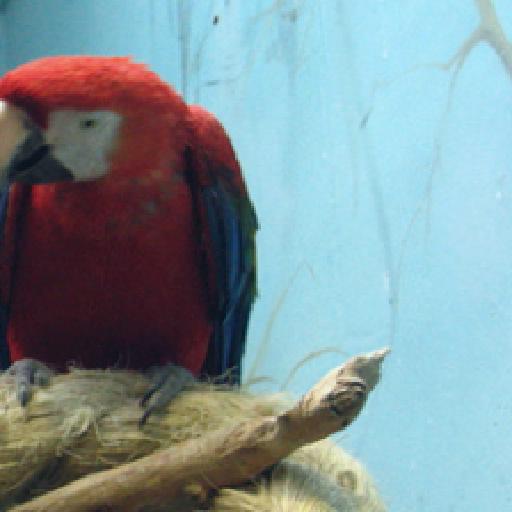}} \\
{\tiny Augmented} & {\includegraphics[scale=0.093]{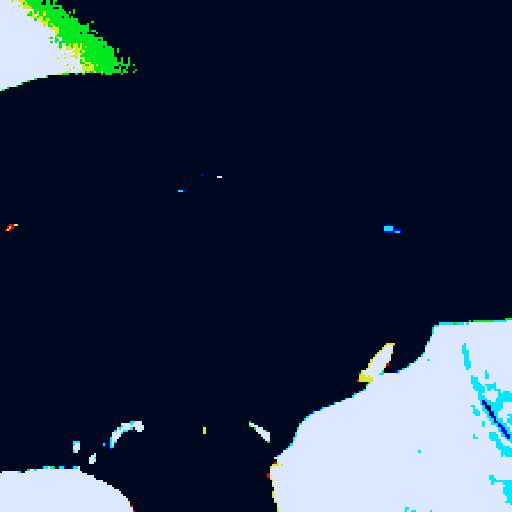}} & {\includegraphics[scale=0.093]{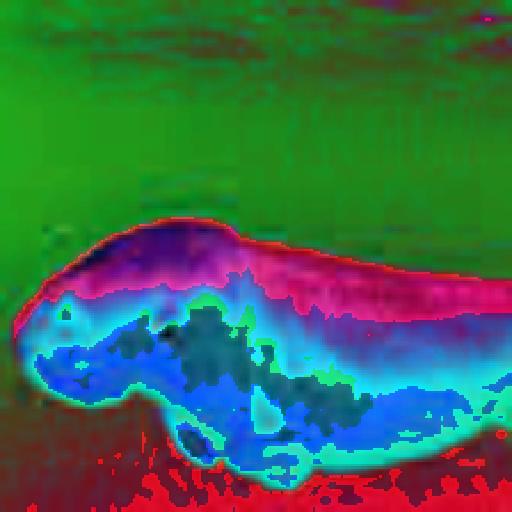}} & {\includegraphics[scale=0.093]{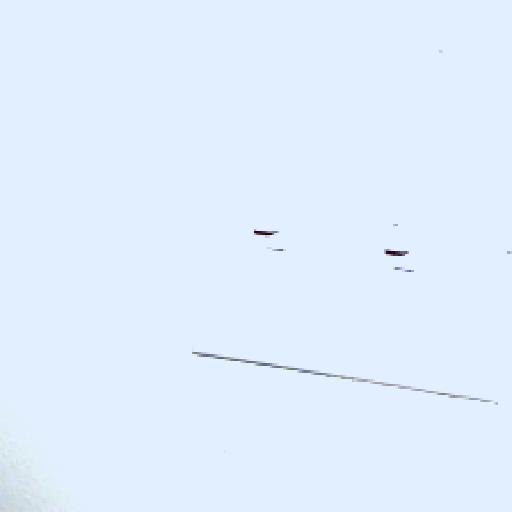}} & {\includegraphics[scale=0.093]{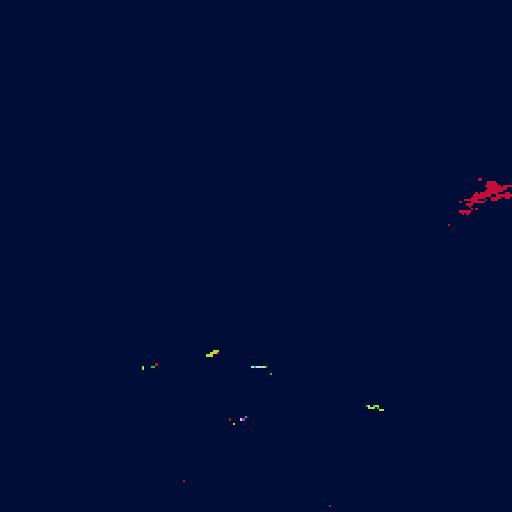}} & {\includegraphics[scale=0.093]{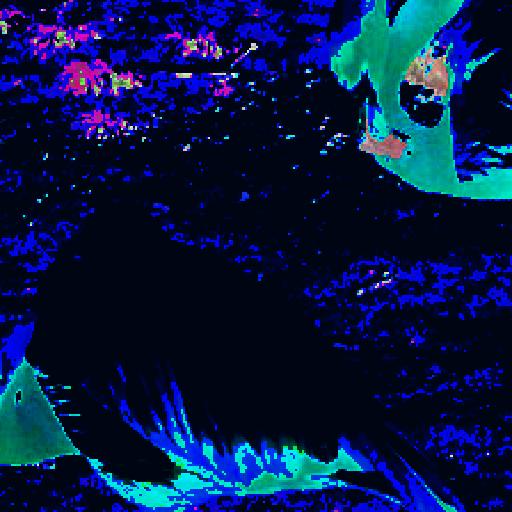}} & {\includegraphics[scale=0.093]{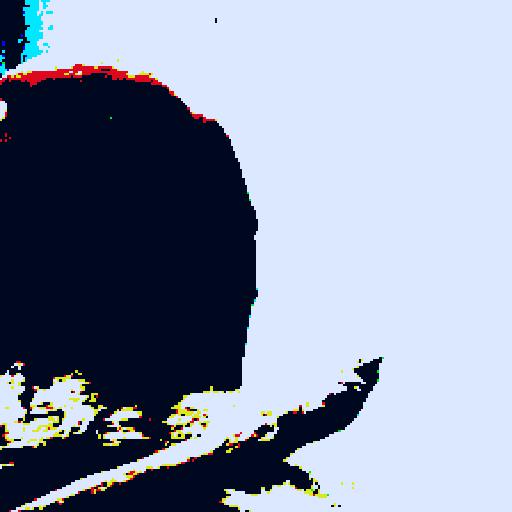}} \\
& \multicolumn{1}{c}{\tiny AutoContrast}  & \multicolumn{1}{c}{\tiny AutoContrast}  & \multicolumn{1}{c}{\tiny Brightness (m=9.0)} & \multicolumn{1}{c}{\tiny Brightness (m=0.9)} & \multicolumn{1}{c}{\tiny Brightness (m=8.8)} & \multicolumn{1}{c}{\tiny Color (m=0.4)} \\
& \multicolumn{1}{c}{\tiny Posterize (m=1.9)}  & \multicolumn{1}{c}{\tiny Solarize (m=3.4)}  & \multicolumn{1}{c}{\tiny Identity} & \multicolumn{1}{c}{\tiny Posterize (m=4.6)} & \multicolumn{1}{c}{\tiny Solarize (m=8.1)} & \multicolumn{1}{c}{\tiny Posterize (m=2.8))} \\
& \multicolumn{1}{c}{\tiny Weight=0.11}  & \multicolumn{1}{c}{\tiny Weight=0.17}  & \multicolumn{1}{c}{\tiny Weight=0.16} & \multicolumn{1}{c}{\tiny Weight=0.12} & \multicolumn{1}{c}{\tiny Weight=0.13} & \multicolumn{1}{c}{\tiny Weight=0.12} \\
{\tiny Original} & {\includegraphics[scale=0.093]{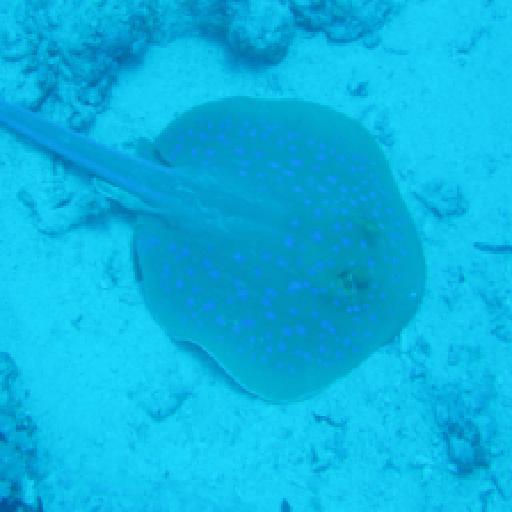}} & {\includegraphics[scale=0.093]{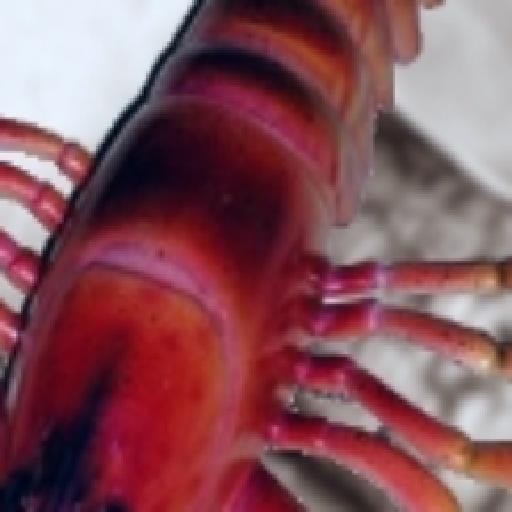}} & {\includegraphics[scale=0.093]{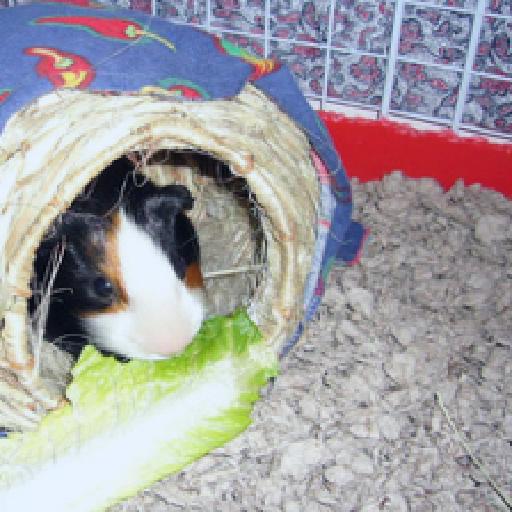}} & {\includegraphics[scale=0.093]{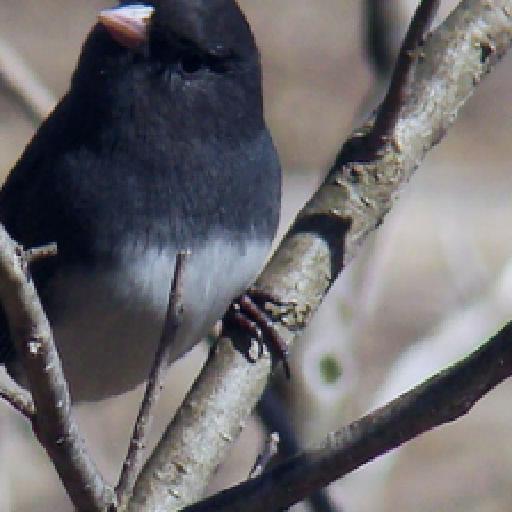}} & {\includegraphics[scale=0.093]{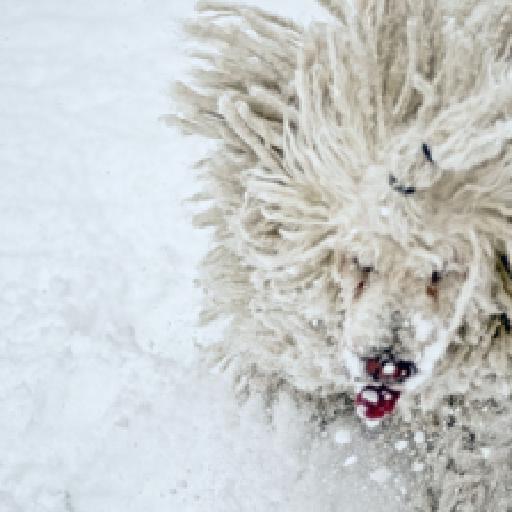}} & {\includegraphics[scale=0.093]{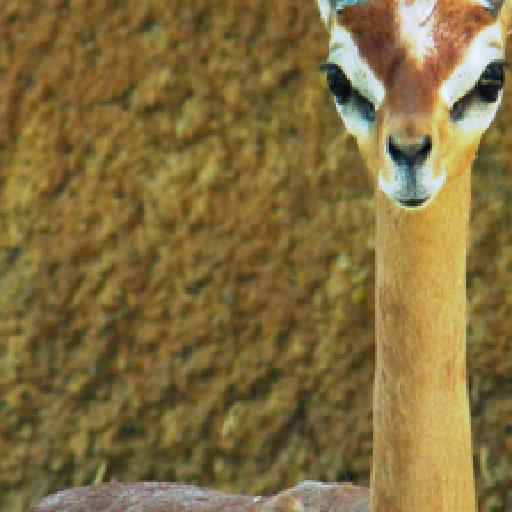}} \\
{\tiny Augmented} & {\includegraphics[scale=0.093]{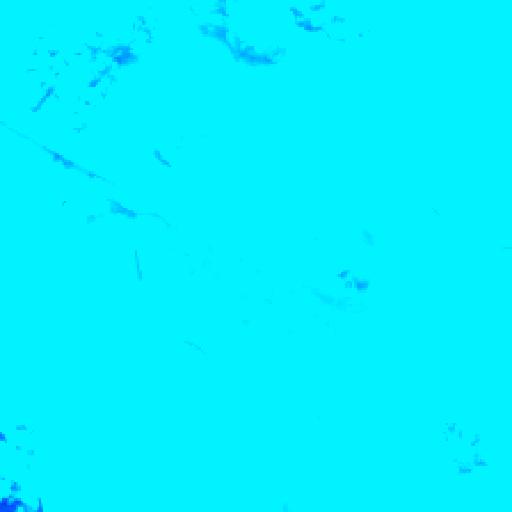}} & {\includegraphics[scale=0.093]{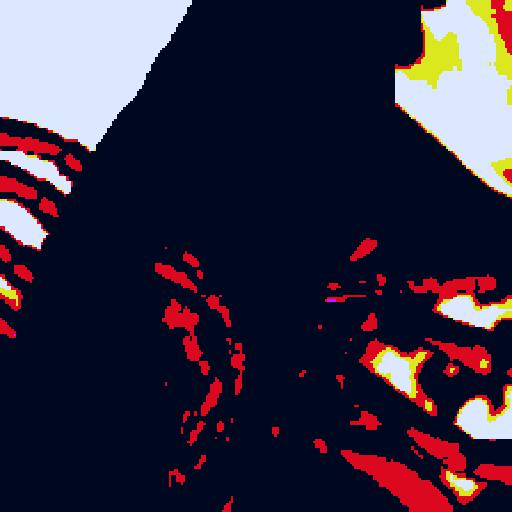}} & {\includegraphics[scale=0.093]{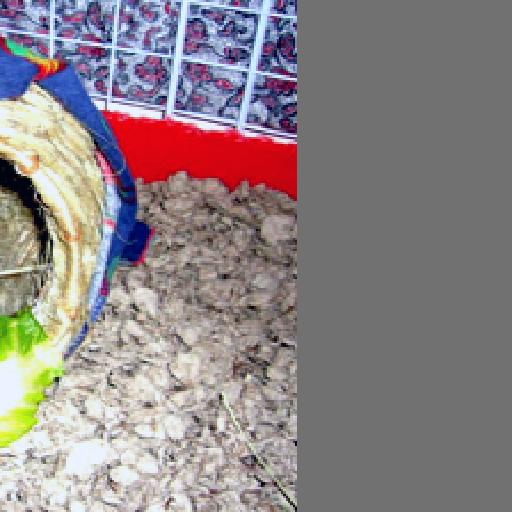}} & {\includegraphics[scale=0.093]{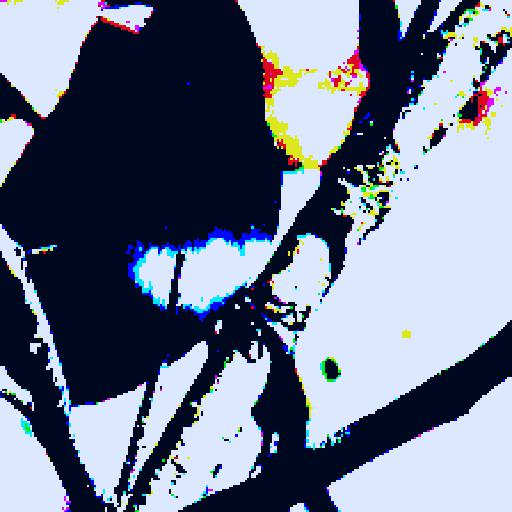}} & {\includegraphics[scale=0.093]{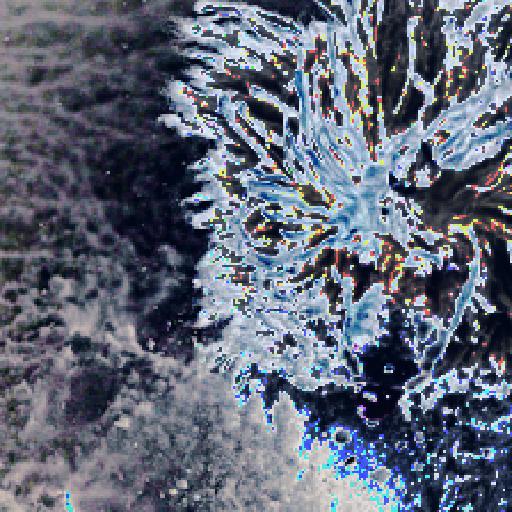}} & {\includegraphics[scale=0.093]{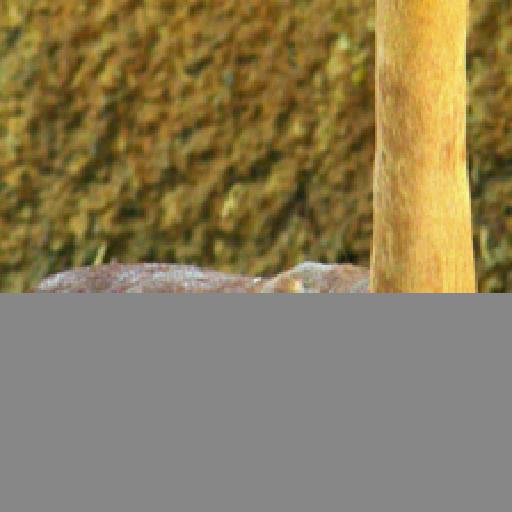}} \\
& \multicolumn{1}{c}{\tiny Contrast (m=9.3)}  & \multicolumn{1}{c}{\tiny Contrast (m=1.4)}  & \multicolumn{1}{c}{\tiny Contrast (m=8.8)} & \multicolumn{1}{c}{\tiny Equalize} & \multicolumn{1}{c}{\tiny Equalize} & \multicolumn{1}{c}{\tiny Identity} \\
& \multicolumn{1}{c}{\tiny Brightness (m=10.0)}  & \multicolumn{1}{c}{\tiny Posterize (m=3.2)}  & \multicolumn{1}{c}{\tiny TranslateX (m=9.3)} & \multicolumn{1}{c}{\tiny Posterize (m=0.1)} & \multicolumn{1}{c}{\tiny Solarize (m=2.8)} & \multicolumn{1}{c}{\tiny TranslateY (m=9.5))} \\
& \multicolumn{1}{c}{\tiny Weight=0.15}  & \multicolumn{1}{c}{\tiny Weight=0.10}  & \multicolumn{1}{c}{\tiny Weight=0.15} & \multicolumn{1}{c}{\tiny Weight=0.09} & \multicolumn{1}{c}{\tiny Weight=0.10} & \multicolumn{1}{c}{\tiny Weight=0.16} \\
{\tiny Original} & {\includegraphics[scale=0.093]{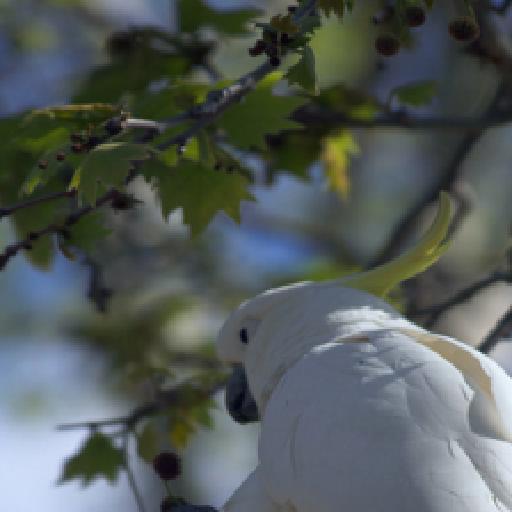}} & {\includegraphics[scale=0.093]{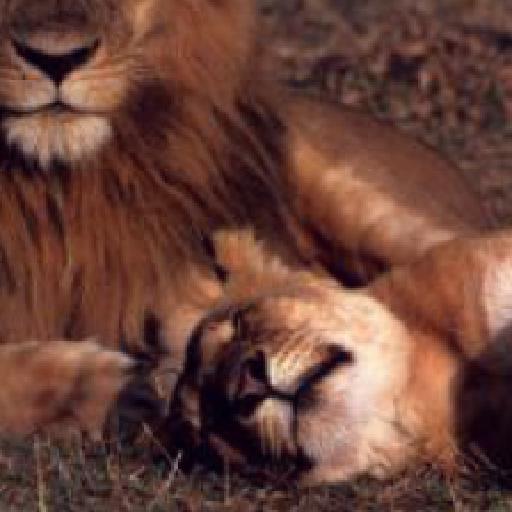}} & {\includegraphics[scale=0.093]{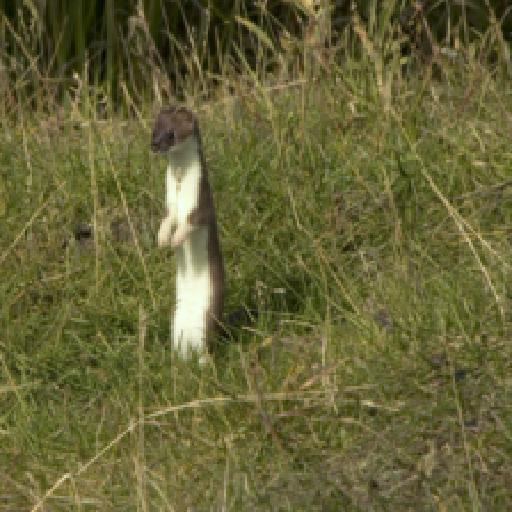}} & {\includegraphics[scale=0.093]{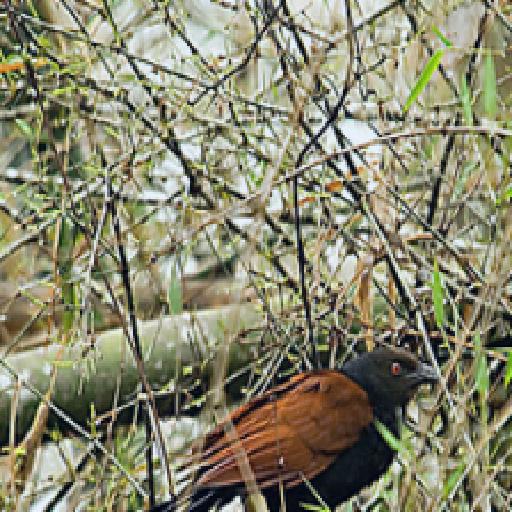}} & {\includegraphics[scale=0.093]{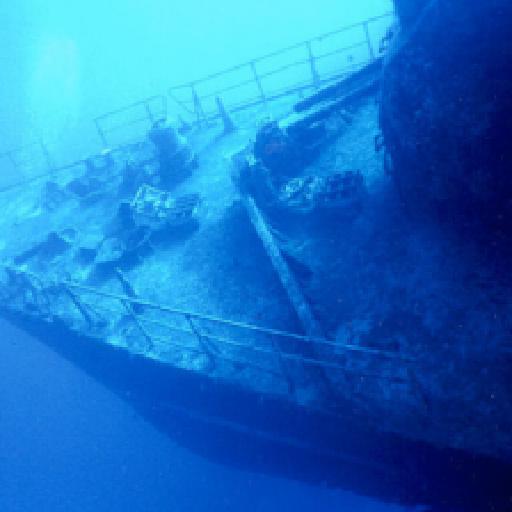}} & {\includegraphics[scale=0.093]{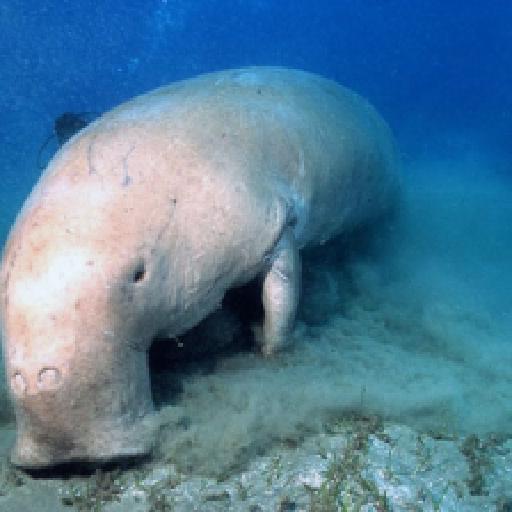}} \\
{\tiny Augmented} & {\includegraphics[scale=0.093]{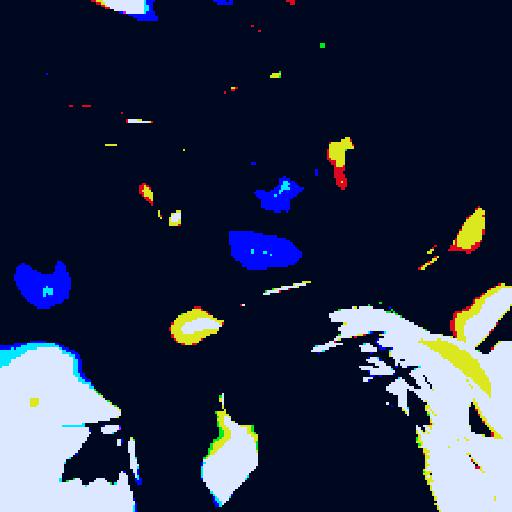}} & {\includegraphics[scale=0.093]{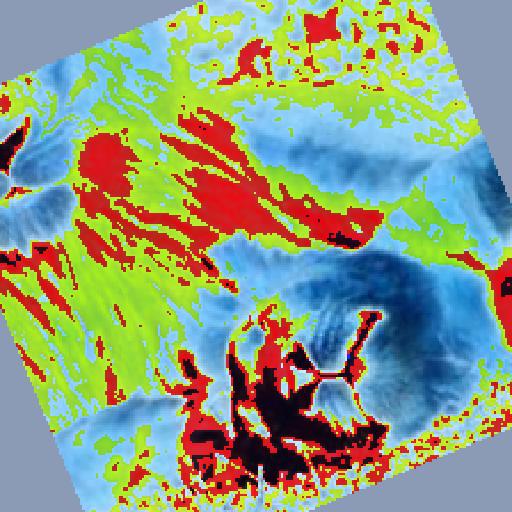}} & {\includegraphics[scale=0.093]{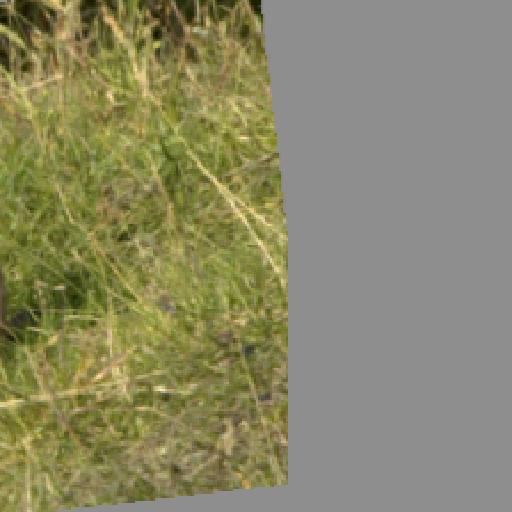}} & {\includegraphics[scale=0.093]{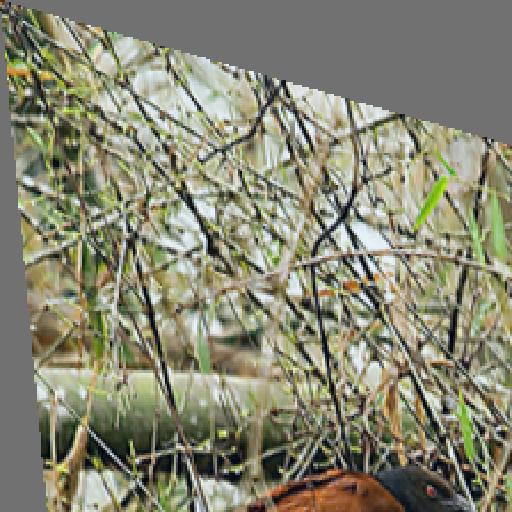}} & {\includegraphics[scale=0.093]{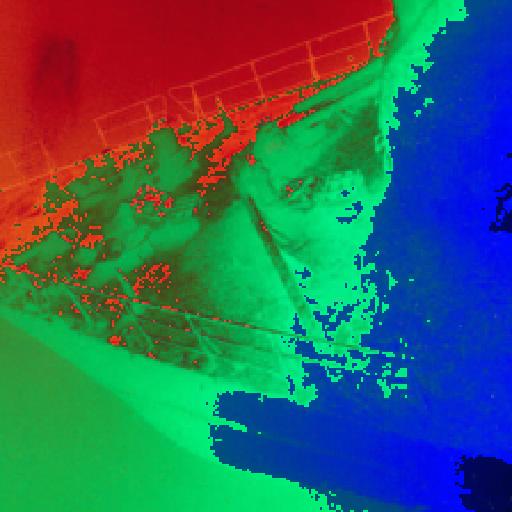}} & {\includegraphics[scale=0.093]{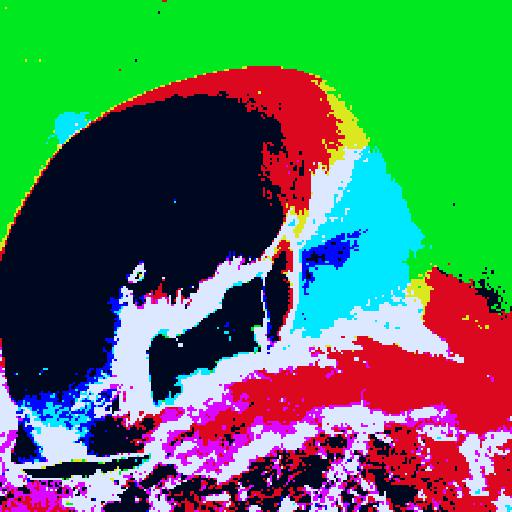}} \\
& \multicolumn{1}{c}{\tiny Posterize (m=2.2)}  & \multicolumn{1}{c}{\tiny Rotate (m=7.4)}  & \multicolumn{1}{c}{\tiny Rotate (m=2.2)} & \multicolumn{1}{c}{\tiny ShearX (m=3.2)} & \multicolumn{1}{c}{\tiny Solarize (m=2.5)} & \multicolumn{1}{c}{\tiny Solarize (m=1.5)} \\
& \multicolumn{1}{c}{\tiny Identity}  & \multicolumn{1}{c}{\tiny Solarize (m=1.5)}  & \multicolumn{1}{c}{\tiny TranslateX (m=9.7)} & \multicolumn{1}{c}{\tiny ShearY (m=9.5)} & \multicolumn{1}{c}{\tiny Identity} & \multicolumn{1}{c}{\tiny Posterize (m=3.3))} \\
& \multicolumn{1}{c}{\tiny Weight=0.12}  & \multicolumn{1}{c}{\tiny Weight=0.12}  & \multicolumn{1}{c}{\tiny Weight=0.13} & \multicolumn{1}{c}{\tiny Weight=0.15} & \multicolumn{1}{c}{\tiny Weight=0.15} & \multicolumn{1}{c}{\tiny Weight=0.12} \\
\end{tabular}
\caption{Examples of augmented samples with low weights.}
\label{fig:low_weights_appendix}
\end{figure}

\section{Analysis on MetaAugment}

In this section, we reformulate the problem of sample-aware data augmentation into a general form and investigate the theoretical properties of the general MetaAugment algorithm. 

\subsection{Problem}

Consider an image recognition task with the training set 
$\cD^{tr} = \{(x_i, y_i)\}_{i=1}^{N^{tr}}$, where $x_i$ denotes the $i$-th image, 
$y_i \in \{0,1\}^c$ is the label vector over $c$ classes, and $N^{tr}$ is the sample size. Let $op_j(x; \xi_k), 1 \leq j \leq M,$ be an augmentation operator applied on an image $x$, 
where $\xi_k\sim Q_j$ is the tuning parameter of the operator $op_j$ and $Q_j$ is a distribution of $\xi_k.$
Here $op_j(x; \xi_k)$ represents the transformation $\cT_{j,k}^{m_1, m_2}(x)$ in Section Methodology. 
Let $f(\cdot\,;\bw)$ be the task network with parameters $\bw$ and $\cV(\cdot,\cdot\,;\bth)$ be the policy network with parameters $\bth$. Let $\ell(f(x;\bw), y)$ denote the loss function. The task network is trained to get $\bw^*$ by minimizing the following weighted training loss:
\begin{linenomath}
\benrr
\cL^{tr}(\bw,\bth)=\frac{1}{N^{tr}} \sum_{i=1}^{N^{tr}} \frac{1}{M} \sum_{j=1}^M 
\bbE_{\xi_k\sim Q_j}\Big[\cV_{ijk}(\bth)L_{ijk}(\bw)\Big],
\eenrr
\end{linenomath}
where 
\begin{linenomath}
\[
\cV_{ijk}(\bth) = \cV\big(h(op_j(x_i;\xi_k)), e(op_j(\cdot\,;\xi_k));\bth\big)
\]
\end{linenomath}
and 
\begin{linenomath}
\[
L_{ijk}(\bw)=\ell\big(f(op_j(x_i;\xi_k); \bw), y_i\big).
\]
\end{linenomath}
Here $h$ is the feature extractor of the policy network and $e(op_j(\cdot\,;\xi_k))$ is the embedding of the augmentation operator $op_j(\cdot\,;\xi_k)$. 
In Section~\ref{proof_thm2}, we will further investigate the convergence of our algorithm when $h(op_j(x_i;\xi_k)) = f(op_j(x_i;\xi_k); \bw)$, i.e., the augmented sample feature is extracted by the task network. 
Notice that $\bw^*$ is a function of $\bth.$
Hence we denote $\bw^*(\bth) = \arg\min_{\bw} \cL^{tr}(\bw, \bth).$ 
Assume that we have a validation set $\cD^{val} = \{(x_{i'}^{val}, y_{i'}^{val})\}_{i'=1}^{N^{val}}$, where $N^{val}$ is the sample size of the validation data. 
The objective of the policy network is to minimize the following validation loss:
\begin{linenomath}
\benrr
\cL^{val}(\bw^*(\bth)) = \frac{1}{N^{val}}\sum_{i'=1}^{N^{val}} L_{i'}^{val}(\bw^*(\bth)),
\eenrr
\end{linenomath}
where $L_{i'}^{val}(\bw^*(\bth))=\ell(f(x_{i'}^{val};\bw^*(\bth)), y_{i'}^{val})$.

\subsection{MetaAugment Algorithm}

The policy network and the task network are trained alternately. 
For each iteration, a mini-batch of training samples $\cD^{tr}_{mi} = \{(x_i, y_i)\}_{i=1}^{n^{tr}}$ with batch size $n^{tr}$ and a mini-batch of augmentation operators $\{op_j(\cdot\,;\xi_k) \mid 1 \leq j \leq m^{op},\ 1 \leq k \leq m^{\xi}\}$ with batch size $m^{op}\cdot m^{\xi}$ are sampled. 
Then the inner loop update of $\bw$ in iteration $t+1$ is
\benr\label{A1}
\hat \bw^{(t)}(\bth) = \bw^{(t)} - \alpha
\frac{1}{n^{tr}}\sum_{i=1}^{n^{tr}} \frac{1}{m^{op}} \sum_{j=1}^{m^{op}} \frac{1}{m^{\xi}}\sum_{k=1}^{m^{\xi}} \cV_{ijk}(\bth) \nabla_{\bw} L_{ijk}(\bw^{(t)}),
\eenr
where $\alpha$ is the learning rate and
$
\nabla_{\bw} L_{ijk}(\bw^{(t)}) = \nabla_{\bw} L_{ijk}(\bw) \big|_{\bw^{(t)}}.
$
The formulation $\hat \bw^{(t)}(\bth)$ is regarded as a function of $\bth$, and then $\bth$ can be updated via the validation loss computed by $\hat \bw^{(t)}(\bth)$ on a mini-batch of validation samples $\cD^{val}_{mi}= \{(x_{i'}^{val}, y_{i'}^{val})\}_{i'=1}^{n^{val}}$ with batch size $n^{val}$. 
The outer loop update of $\bth$ is formulated by 
\benr\label{A2}
\bth^{(t+1)} = \bth^{(t)} - \beta \frac{1}{n^{val}} \sum_{i'=1}^{n^{val}} \nabla_{\bth} L^{val}_{i'}(\hat\bw^{(t)}(\bth^{(t)})),
\eenr
where $\beta$ is the learning rate and
$
\nabla_{\bth} L^{val}_{i'}(\hat\bw^{(t)}(\bth^{(t)})) = \nabla_{\bth} L^{val}_{i'}(\hat\bw^{(t)}(\bth))\big|_{\bth^{(t)}}.
$
The third step in iteration $t+1$ is the outer loop update of $\bw^{(t)}$ with the updated $\bth^{(t+1)}$:
\benr\label{A3}
\bw^{(t+1)} = \bw^{(t)} - \gamma \frac{1}{n^{tr}} \sum_{i=1}^{n^{tr}} \frac{1}{m^{op}} \sum_{j=1}^{m^{op}} \frac{1}{m^{\xi}}\sum_{k=1}^{m^{\xi}} \cV_{ijk}(\bth^{(t+1)}) \nabla_{\bw} L_{ijk}(\bw^{(t)}),
\eenr
where $\gamma$ is the learning rate.

\subsection{Analysis on Augmentation Policy Network}

According to the chain rule, the update of $\bth$ in Eq.~\eqref{A2} can be rewritten as:
\benr\label{outer11}
\bth^{(t+1)} &=& \bth^{(t)}+ \alpha \beta \frac{1}{n^{tr}} \sum_{i=1}^{n^{tr}}  \frac{1}{m^{op}} \sum_{j=1}^{m^{op}} \\
&& \times \frac{1}{m^{\xi}}\sum_{k=1}^{m^{\xi}}\Big(\frac{1}{n^{val}} \sum_{i'=1}^{n^{val}} R_{ii'}(op_j, \xi_k)\Big)\nabla_{\bth} \cV_{ijk}(\bth^{(t)}), \nn
\eenr
where $R_{ii'}(op_j, \xi_k) = \nabla_{\bw} L_{ijk}(\bw^{(t)})^\T \nabla_{\hat\bw} L_{i'}^{val}(\hat \bw^{(t)})$
and 
$\nabla_{\bth} \cV_{ijk}(\bth^{(t)}) = \nabla_{\bth} \cV_{ijk}(\bth)\big|_{\bth^{(t)}}.$
Notice that the update vector of $\bth^{(t)}$ is a weighted sum of the gradient of $\cV_{ijk}(\bth^{(t)})$ with respect to $\bth^{(t)}.$
What's more, the update direction is the gradient ascend direction of the weighted sum of $\cV_{ijk}(\bth^{(t)}).$
The weight of $\nabla_{\bth} \cV_{ijk}(\bth^{(t)})$ is formulated by 
\benrr
\frac{1}{n^{val}} \sum_{i'=1}^{n^{val}} R_{ii'}(op_j, \xi_k) &=& \frac{1}{n^{val}} \sum_{i'=1}^{n^{val}} \nabla_{\bw} L_{ijk}(\bw^{(t)})^\T \nabla_{\hat\bw} L_{i'}^{val}(\hat \bw^{(t)}) \\
&=& \Big\langle \nabla_{\bw} L_{ijk}(\bw^{(t)}), \frac{1}{n^{val}} \sum_{i'=1}^{n^{val}} \nabla_{\hat\bw} L_{i'}^{val}(\hat \bw^{(t)})  \Big\rangle.
\eenrr
The inner product measures the similarity between the gradient of an augmented sample loss and the average gradient of the losses computed on a mini-batch of validation data. 
If the gradient of an augmented sample loss is similar to that of the validation loss, 
this augmented sample is likely to improve the performance of the task network on the validation data and its weight will be increased after the update of the policy network.

\subsection{Useful Lemma}

{\bf Lemma 1.} {\it Suppose that:
\vs 0.05cm
\noi ($\bA$1) The loss function $\ell$ have $\rho_1$-bounded gradients with respect to $\bw$ under both (augmented) training data and validation data, and the loss function $\ell$ is Lipschitz smooth with constant $\rho_2$;\\
\noi ($\bA$2) The policy network $\cV$ is differential with a $\delta_1$-bounded gradient and twice differential with its Hessian bounded by $\delta_2$ with respect to $\bth$. 
\vs 0.05cm
Then the validation loss has $\rho'_1$-bounded gradients with respect to $\bth$ and is Lipschitz continuous with $\rho'_2$, where $\rho'_1 = \alpha  \rho_1^2 \delta_1$ and $\rho'_2 = \alpha \rho_1^2 (\alpha \delta_1^2 \rho_2 + \delta_2).$}
%\vs 0.2cm

\begin{proof}
The gradient of the validation loss $L^{val}_{i'}(\hat\bw^{(t)}(\bth^{(t)}))$ with respect $\bth$ can be written as:
\benrr
&&\nabla_{\bth} L^{val}_{i'}(\hat\bw^{(t)}(\bth^{(t)}))\\
&=& -\alpha \frac{1}{n^{tr}}\sum_{i=1}^{n^{tr}} \frac{1}{m^{op}} \sum_{j=1}^{m^{op}} \frac{1}{m^{\xi}}\sum_{k=1}^{m^{\xi}}   \nabla_{\bth}\cV_{ijk}(\bth^{(t)})\nabla_{\bw} L_{ijk}(\bw^{(t)})^\T \nabla_{\hat\bw} L^{val}_{i'}(\hat\bw^{(t)}) \\
&=& -\alpha \frac{1}{n^{tr}}\sum_{i=1}^{n^{tr}} \frac{1}{m^{op}} \sum_{j=1}^{m^{op}} \frac{1}{m^{\xi}}\sum_{k=1}^{m^{\xi}}   \nabla_{\bth}\cV_{ijk}(\bth^{(t)}) R_{ii'}(op_j, \xi_k).
\eenrr
By the assumptions {\it ($\bA$1)} and {\it ($\bA$2)},
\benrr
&&\big\| \nabla_{\bth} L^{val}_{i'}(\hat\bw^{(t)}(\bth^{(t)})) \big\| \\
&\leq& \alpha \frac{1}{n^{tr}}\sum_{i=1}^{n^{tr}} \frac{1}{m^{op}} \sum_{j=1}^{m^{op}} \frac{1}{m^{\xi}}\sum_{k=1}^{m^{\xi}}   \big\| \nabla_{\bth}\cV_{ijk}(\bth^{(t)}) \big\| 
\big\|\nabla_{\bw} L_{ijk}(\bw^{(t)})\big\| \big\| \nabla_{\hat\bw} L^{val}_{i'}(\hat\bw^{(t)})\big\|\\
&\leq & \alpha \rho_1^2 \delta_1 = \rho'_1.
\eenrr
Further the Hessian of the validation loss $L^{val}_{i'}(\hat\bw^{(t)}(\bth^{(t)}))$ with respect $\bth$ is
\benrr
\nabla^2_{\bth} L^{val}_{i'}(\hat\bw^{(t)}(\bth^{(t)})) = -\alpha \frac{1}{n^{tr}}\sum_{i=1}^{n^{tr}} \frac{1}{m^{op}} \sum_{j=1}^{m^{op}} \frac{1}{m^{\xi}}\sum_{k=1}^{m^{\xi}}(I_{1, i'ijk} + I_{2,i'ijk}),
\eenrr
where 
\benrr
I_{1, i'ijk} &=& \nabla_{\bth}\cV_{ijk}(\bth^{(t)}) \nabla_{\bth}R_{ii'}(op_j,\xi_k)^\T, \\
I_{2, i'ijk} &=&  \nabla^2_{\bth}\cV_{ijk}(\bth^{(t)}) R_{ii'}(op_j,\xi_k).
\eenrr
According to the assumption {\it ($\bA$2)}, 
\benrr
\|I_{1, i'ijk}\| &\leq& \|\nabla_{\bth}\cV_{ijk}(\bth^{(t)})\| \| \nabla_{\bth}R_{ii'}(op_j,\xi_k)\| \\
&\leq& \delta_1 \| \nabla_{\bth}R_{ii'}(op_j,\xi_k)\|.
\eenrr
Furthermore,
\benrr
&&\| \nabla_{\bth}R_{ii'}(op_j,\xi_k)\| \\
&=& \Big\| \nabla_{\hat \bw}\Big(\nabla_{\bth} L_{i'}^{val}\big(\hat \bw^{(t)}(\bth^{(t)})\big) \Big)\Big|_{\hat \bw^{(t)}} \nabla_{\bw} L_{ijk}(\bw^{(t)}) \Big\| \\
&=& \Big\| \nabla_{\hat \bw} \Big( -\alpha \frac{1}{n^{tr}}\sum_{i=1}^{n^{tr}} \frac{1}{m^{op}} \sum_{j=1}^{m^{op}} \frac{1}{m^{\xi}}\sum_{k=1}^{m^{\xi}}
\nabla_{\bth}\cV_{ijk}(\bth^{(t)})\nabla_{\bw} L_{ijk}(\bw^{(t)})^\T \\
&& \times  \nabla_{\hat \bw} L_{i'}^{val}(\hat \bw)
\Big)\Big|_{\hat \bw^{(t)}} \nabla_{\bw} L_{ijk}(\bw^{(t)}) \Big\| \\
&=& \Big\| -\alpha \frac{1}{n^{tr}}\sum_{i=1}^{n^{tr}} \frac{1}{m^{op}} \sum_{j=1}^{m^{op}} \frac{1}{m^{\xi}}\sum_{k=1}^{m^{\xi}}
\nabla_{\bth}\cV_{ijk}(\bth^{(t)})\nabla_{\bw} L_{ijk}(\bw^{(t)})^\T \\
&& \times \nabla^2_{\hat \bw} L_{i'}^{val}(\hat \bw^{(t)}) \nabla_{\bw} L_{ijk}(\bw^{(t)}) \Big\|,
\eenrr
where
$
\nabla^2_{\hat \bw} L_{i'}^{val}(\hat \bw^{(t)}) = \nabla^2_{\hat \bw} L_{i'}^{val}(\hat \bw)\big|_{\hat \bw^{(t)}}.
$
By the assumptions {\it ($\bA$1)} and {\it ($\bA$2)}, 
\benrr
\|I_{1, i'ijk}\| &\leq& \delta_1 \| \nabla_{\bth}R_{ii'}(op_j,\xi_k)\| \leq \alpha \delta_1^2\rho_1^2 \rho_2.
\eenrr
For the second term $I_{2, i'ijk}$, we have
\benrr
\|I_{2, i'ijk}\| &=& \|\nabla^2_{\bth}\cV_{ijk}(\bth^{(t)})\| |R_{ii'}(op_j,\xi_k)| \\
& \leq & \delta_2 \|\nabla_{\bw} L_{ijk}(\bw^{(t)})\| \|\nabla_{\hat\bw} L_{i'}^{val}(\hat \bw^{(t)})\| \leq \delta_2 \rho_1^2,
\eenrr
where the first inequality holds by the assumption {\it ($\bA$2)} and the second inequality holds by the assumption {\it ($\bA$1)}.
Combining the upper bound of $\|I_{1, i'ijk}\|$ and $\|I_{2, i'ijk}\|$, we have
\benrr
\|\nabla^2_{\bth} L^{val}_{i'}(\hat\bw^{(t)}(\bth^{(t)})) \| &\leq&  \frac{\alpha}{n^{tr}}\sum_{i=1}^{n^{tr}} \frac{1}{m^{op}} \sum_{j=1}^{m^{op}} \frac{1}{m^{\xi}}\sum_{k=1}^{m^{\xi}}\big(\|I_{1, i'ijk}\| + \|I_{2,i'ijk}\|\big)\\
&\leq&  \alpha \rho_1^2 (\alpha \delta_1^2 \rho_2 + \delta_2). 
\eenrr
By Lagrange mean value theorem,
\benrr
\|\nabla_{\bth} \cL^{val}(\hat\bw^{(t)}(\bth_1)) - \nabla_{\bth} \cL^{val}(\hat\bw^{(t)}(\bth_2))\| \leq \rho'_{2} \|\bth_1 - \bth_2\|
\eenrr
for all $\bth_1$ and $\bth_2.$
%\bbox
\end{proof}

\subsection{Proof of Theorem~\ref{theorem1}}

\begin{theorem}
Suppose the following assumptions hold: 
\vs 0.05cm
\noi ($\bA$1) The loss function $\ell$ have $\rho_1$-bounded gradients with respect to $\bw$ under both (augmented) training data and validation data, and the loss function $\ell$ is Lipschitz smooth with constant $\rho_2$;\\
%\vs 0.05cm
\noi ($\bA$2) The policy network $\cV$ is differential with a $\delta_1$-bounded gradient and twice differential with its Hessian bounded by $\delta_2$ with respect to $\bth$; \\
%\vs 0.1cm
\noi ($\bA$3) The absolute values of the policy network $\cV$ and the loss function $\ell$ are bounded above by $C_1$ and $C_2$, respectively; \\
%\vs 0.05cm
\noi ($\bA$4) For any iteration $0 \leq t \leq T-1$, the variance of the weighted training loss (validation loss) gradient on a mini-batch of training (validation) samples is bounded above; \\
%\vs 0.05cm
\noi ($\bA$5) Let
\benrr
\alpha = \frac{c \log T}{T}, \quad \beta = \sqrt{\frac{c' \log\log T}{T}}, \quad \gamma =  \frac{c''\log T}{T},
\eenrr
for some positive constants $c$, $c'$ and $c''$; \\
%\vs 0.05cm
\noi ($\bA$6) The number of iterations $T$ is sufficiently large such that $\alpha \beta \rho_1^2(\alpha \delta_1^2 \rho_2 + \delta_2)<1$ and $\gamma C_1 \rho_2 <1.$
\vs 0.05cm
If the policy network has its own feature extractor, we have
\benrr
&& \frac{1}{T} \sum_{t=0}^{T-1}\bbE \Big[  \big\|\nabla_{\bth} \cL^{val}(\hat \bw^{(t)}(\bth^{(t)}))\big\|^2 \Big] \leq O(\frac{\log T}{\sqrt{T\log\log T}}), \\
&& \lim_{T \rightarrow \infty} 
\frac{1}{T}\sum_{t=0}^{T-1} \bbE \Big[\|\nabla_{\bw} \cL^{tr}(\bw^{(t)}, \bth^{(t+1)})\|^2 \Big] = 0.
\eenrr
%}
\end{theorem}
%\vs 0.2cm

\begin{proof}
We start with $\cL^{val}(\hat \bw^{(t+1)}(\bth^{(t+1)})) - \cL^{val}(\hat \bw^{(t)}(\bth^{(t)})).$ 
Decompose it into $I_1 + I_2$, where 
\benrr
I_1 &=& \cL^{val}(\hat \bw^{(t+1)}(\bth^{(t+1)})) - \cL^{val}(\hat \bw^{(t)}(\bth^{(t+1)})), \\
I_2 &=& \cL^{val}(\hat \bw^{(t)}(\bth^{(t+1)})) - \cL^{val}(\hat \bw^{(t)}(\bth^{(t)})).
\eenrr
By the assumption {\it ($\bA$1)},
\benrr
I_1 
&\leq& \big\|\nabla_{\hat \bw} \cL^{val}(\hat \bw^{(t)}(\bth^{(t+1)}))\big\| 
\big\|\hat \bw^{(t+1)}(\bth^{(t+1)}) - \hat \bw^{(t)}(\bth^{(t+1)})\big\| \\
&& + \frac{\rho_2}{2} \big\|\hat \bw^{(t+1)}(\bth^{(t+1)}) - \hat \bw^{(t)}(\bth^{(t+1)})\big\|^2.
\eenrr
To proceed further, denote
\benrr
\nabla_{\bw} \cL_{t}^{tr}(\bw,\bth)) = \frac{1}{n^{tr}} \sum_{i=1}^{n^{tr}} \frac{1}{m^{op}} \sum_{j=1}^{m^{op}} \frac{1}{m^{\xi}}\sum_{k=1}^{m^{\xi}}
\cV_{ijk}(\bth)\nabla_{\bw} L_{ijk}(\bw), 
\eenrr
where the batches of training samples and augmentation operators are sampled at time $t.$ 
Note that 
\benrr
&&\hat \bw^{(t+1)}(\bth^{(t+1)}) - \hat \bw^{(t)}(\bth^{(t+1)}) \\
&=& (\bw^{(t+1)} - \alpha \nabla_{\bw} \cL_{t+1}^{tr}(\bw^{(t+1)},\bth^{(t+1)})) - (\bw^{(t)} - \alpha \nabla_{\bw} \cL_{t}^{tr}(\bw^{(t)},\bth^{(t+1)})) \\
&=& (\bw^{(t+1)}-\bw^{(t)}) - \alpha \nabla_{\bw} \cL_{t+1}^{tr}(\bw^{(t+1)},\bth^{(t+1)}) + \alpha \nabla_{\bw} \cL_{t}^{tr}(\bw^{(t)},\bth^{(t+1)})\\
&=& (\alpha - \gamma)\nabla_{\bw} \cL_{t}^{tr}(\bw^{(t)},\bth^{(t+1)}) - \alpha \nabla_{\bw} \cL_{t+1}^{tr}(\bw^{(t+1)},\bth^{(t+1)}).
\eenrr
Thus, by the assumptions {\it ($\bA$1)} and {\it ($\bA$3)}, 
\benrr
&&\|\hat \bw^{(t+1)}(\bth^{(t+1)}) - \hat \bw^{(t)}(\bth^{(t+1)})\| \\
&\leq& |\gamma - \alpha| \|\nabla_{\bw} \cL_{t}^{tr}(\bw^{(t)},\bth^{(t+1)})\| + \alpha \|\nabla_{\bw} \cL_{t+1}^{tr}(\bw^{(t+1)},\bth^{(t+1)})\| \\
&\leq& |\gamma - \alpha| \rho_1 C_1 + \alpha \rho_1 C_1 \leq 2 \max\{\gamma, \alpha\} \rho_1 C_1.
\eenrr
According to the assumption {\it ($\bA$5)}, we rewrite $2 \max\{\gamma, \alpha\} \rho_1 C_1$ as $ \gamma \rho_1 \tilde C_1.$ 
Then the upper bound of $|I_1|$ can be written as
\benrr
|I_1| \leq \gamma \rho_1^2 \tilde C_1 + \frac{1}{2} \gamma^2 \rho_1^2 \rho_2 \tilde C_1^2 = \gamma \rho_1^2 \tilde C_1(1 + \frac{1}{2}\gamma \rho_2 \tilde C_1).
\eenrr
Next we deal with $I_2.$ By {\bf Lemma~1}, we have
\benrr
I_2 \leq \nabla_{\bth} \cL^{val}(\hat \bw^{(t)}(\bth^{(t)}))^\T (\bth^{(t+1)} - \bth^{(t)}) + \frac{\rho'_2}{2}\|\bth^{(t+1)} - \bth^{(t)}\|^2.
\eenrr
Note that
\benrr
\bth^{(t+1)} - \bth^{(t)} =  - \beta \frac{1}{n^{val}} \sum_{i'=1}^{n^{val}} \nabla_{\bth} L^{val}_{i'}(\hat\bw^{(t)}(\bth^{(t)})),
\eenrr
where $\{(x_{i'}^{val}, y_{i'}^{val})\}_{i'=1}^{n^{val}}$ is a mini-batch randomly sampled from all validation data.
To proceed further, we denote
\benrr
\nabla_{\bth} \cL^{val}_t(\hat \bw^{(t)}(\bth^{(t)})) = \frac{1}{n^{val}} \sum_{i'=1}^{n^{val}}  \nabla_{\bth} L^{val}_{i'}(\hat\bw^{(t)}(\bth^{(t)}))
\eenrr
and $\vep^{(t)} = \nabla_{\bth} \cL_{t}^{val}(\hat \bw^{(t)}(\bth^{(t)})) - \nabla_{\bth} \cL^{val}(\hat \bw^{(t)}(\bth^{(t)})).$ 
By the assumption {\it ($\bA$4)}, 
$
\bbE[\|\vep^{(t)}\|^2] \leq \sigma_1^2
$
for some positive scalar $\sigma_1.$ 
Rewrite the update vector of $\bth$ as $\nabla_{\bth} \cL^{val}(\hat \bw^{(t)}(\bth^{(t)})) + \vep^{(t)}$, and plug it into the upper bound of $I_2$. Then 
\benrr
I_2 &\leq& -\beta \nabla_{\bth} \cL^{val}(\hat \bw^{(t)}(\bth^{(t)}))^\T \big(\nabla_{\bth} \cL^{val}(\hat \bw^{(t)}(\bth^{(t)}))+ \vep^{(t)}\big) \\
&& + \frac{\rho'_2\beta^2}{2}\big\|\nabla_{\bth} \cL^{val}(\hat \bw^{(t)}(\bth^{(t)}))+ \vep^{(t)}\big\|^2\\
&=& -(\beta-\frac{\rho'_2\beta^2}{2})\big\|\nabla_{\bth} \cL^{val}(\hat \bw^{(t)}(\bth^{(t)}))\big\|^2 + \frac{\rho'_2\beta^2}{2} \|\vep^{(t)}\|^2\\
&& - (\beta-\rho'_2\beta^2)\nabla_{\bth} \cL^{val}(\hat \bw^{(t)}(\bth^{(t)}))^\T\vep^{(t)}.
\eenrr
Combining the upper bound of $I_1$ and $I_2$, we have
\benrr
&&\cL^{val}(\hat \bw^{(t+1)}(\bth^{(t+1)})) - \cL^{val}(\hat \bw^{(t)}(\bth^{(t)}))\\
&\leq& \gamma \rho_1^2 \tilde C_1(1 + \frac{1}{2}\gamma \rho_2 \tilde C_1) -(\beta-\frac{\rho'_2\beta^2}{2})\big\|\nabla_{\bth} \cL^{val}(\hat \bw^{(t)}(\bth^{(t)}))\big\|^2 \\
&& + \frac{\rho'_2\beta^2}{2} \|\vep^{(t)}\|^2 - (\beta-\rho'_2\beta^2)\nabla_{\bth} \cL^{val}(\hat \bw^{(t)}(\bth^{(t)}))^\T\vep^{(t)}.
\eenrr
Rearranging the terms, we can obtain that
\benrr
&& (\beta-\frac{\rho'_2\beta^2}{2})\big\|\nabla_{\bth} \cL^{val}(\hat \bw^{(t)}(\bth^{(t)}))\big\|^2 \\
&\leq& \gamma \rho_1^2 \tilde C_1(1 + \frac{1}{2}\gamma \rho_2 \tilde C_1) +  \cL^{val}(\hat \bw^{(t)}(\bth^{(t)})) - \cL^{val}(\hat \bw^{(t+1)}(\bth^{(t+1)}))\\
&& + \frac{\rho'_2\beta^2}{2} \|\vep^{(t)}\|^2 - (\beta-\rho'_2\beta^2)\nabla_{\bth} \cL^{val}(\hat \bw^{(t)}(\bth^{(t)}))^\T\vep^{(t)}.
\eenrr
By taking the mean of $t$ from $0$ to $T-1$ and taking the expectation with respect to the mini-batch of samples, 
\benrr
&&\frac{1}{T}\sum_{t=0}^{T-1}(\beta-\frac{\rho'_2\beta^2}{2})\bbE\Big[\big\|\nabla_{\bth} \cL^{val}(\hat \bw^{(t)}(\bth^{(t)}))\big\|^2\Big]\\
&\leq& \gamma \rho_1^2 \tilde C_1(1 + \frac{1}{2}\gamma \rho_2 \tilde C_1) + \frac{1}{T}\bbE\Big[\big( \cL^{val}(\hat \bw^{(0)}(\bth^{(0)})) - \cL^{val}(\hat \bw^{(T)}(\bth^{(T)})) \big)\Big] \\
&& + \frac{1}{T} \sum_{t=0}^{T-1}\frac{\rho'_2\beta^2}{2} \bbE\Big[\|\vep^{(t)}\|^2\Big] \\
&\leq&\gamma \rho_1^2 \tilde C_1(1 + \frac{1}{2}\gamma \rho_2 \tilde C_1) + \frac{1}{T}\Delta_{\cL}^{val} + \frac{\rho'_2\sigma_1^2}{2} \beta^2,
\eenrr
where 
$\Delta_{\cL}^{val} = \sup_{\bw}\cL^{val}(\bw) - \inf_{\bw}\cL^{val}(\bw).$ 
Furthermore, we have
\benrr
&&\frac{1}{T} \sum_{t=0}^{T-1} \bbE\Big[  \big\|\nabla_{\bth} \cL^{val}(\hat \bw^{(t)}(\bth^{(t)}))\big\|^2\Big]\\
&\leq& \frac{ \gamma \rho_1^2 \tilde C_1(2 + \gamma \rho_2 \tilde C_1) + 2 \Delta_{\cL}^{val}/T + \alpha\beta^2 \rho_1^2\sigma_1^2  (\alpha \delta_1^2 \rho_2 + \delta_2) }{ 2\beta-\rho'_2\beta^2}.
\eenrr
Note that $(\beta-\rho'_2\beta^2) > 0.$ 
Thus,
\benrr
&&\frac{1}{T} \sum_{t=0}^{T-1} \bbE\Big[ \big\|\nabla_{\bth} \cL^{val}(\hat \bw^{(t)}(\bth^{(t)}))\big\|^2\Big]\\
&\leq& \frac{\gamma \rho_1^2 \tilde C_1(2 + \gamma \rho_2 \tilde C_1)}{\beta}
+ \frac{2 \Delta_{\cL}^{val}}{T\beta} + \frac{1}{2}\alpha\beta \rho_1^2\sigma_1^2  (\alpha \delta_1^2 \rho_2 + \delta_2) \\
&=& O(\frac{\log T}{\sqrt{T \log\log T}}) + O(\frac{1}{\sqrt{T \log \log T}}) + O(\sqrt{\frac{(\log T)^2\log\log T}{T^3}}).
\eenrr
Hence, 
\benrr
\frac{1}{T} \sum_{t=0}^{T-1} \bbE \Big[  \big\|\nabla_{\bth} \cL^{val}(\hat \bw^{(t)}(\bth^{(t)}))\big\|^2 \Big] \leq O(\frac{\log T}{\sqrt{T \log\log T}}).
\eenrr
Next we prove the convergence of the training loss. 
We start with the decomposition that  
\benrr
\cL^{tr}(\bw^{(t+1)}, \bth^{(t+2)}) - \cL^{tr}(\bw^{(t)}, \bth^{(t+1)}) = I_3 + I_4,
\eenrr
where
\benrr
I_3 &=& \cL^{tr}(\bw^{(t+1)}, \bth^{(t+2)}) - \cL^{tr}(\bw^{(t+1)}, \bth^{(t+1)}), \\
I_4 &=& \cL^{tr}(\bw^{(t+1)}, \bth^{(t+1)}) - \cL^{tr}(\bw^{(t)}, \bth^{(t+1)}).
\eenrr
For the term $I_3$, we have
\benrr
I_3= \frac{1}{N^{tr}} \sum_{i=1}^{N^{tr}} \frac{1}{M} \sum_{j=1}^M \bbE_{j}\Big\{ \big[
\cV_{ijk}(\bth^{(t+2)}) - \cV_{ijk}(\bth^{(t+1)})\big] L_{ijk}(\bw^{(t+1)})\Big\},
\eenrr
where $\bbE_j$ stands for $\bbE_{\xi_k\sim Q_j}.$
According to the assumption {\it ($\bA$2)},
\benrr
&& \cV_{ijk}(\bth^{(t+2)}) - \cV_{ijk}(\bth^{(t+1)})   \\
&\leq& \nabla_{\bth} \cV_{ijk}(\bth^{(t+1)})^\T(\bth^{(t+2)}-\bth^{(t+1)}) + \frac{\delta_2}{2}\|\bth^{(t+2)}-\bth^{(t+1)}\|^2.  \nn
\eenrr
Notice that
\benrr
\bth^{(t+2)}-\bth^{(t+1)} &=& - \beta \nabla_{\bth} \cL^{val}_{t+1}(\hat \bw^{(t+1)}(\bth^{(t+1)})) \\
&=& - \beta \big(\nabla_{\bth} \cL^{val}(\hat \bw^{(t+1)}(\bth^{(t+1)})) + \vep^{(t+1)}\big).
\eenrr
Thus we have
\benrr
&& \cV_{ijk}(\bth^{(t+2)}) - \cV_{ijk}(\bth^{(t+1)}) \nn \\
&\leq& -\beta \nabla_{\bth} \cV_{ijk}(\bth^{(t+1)})^\T\big(\nabla_{\bth} \cL^{val}(\hat \bw^{(t+1)}(\bth^{(t+1)})) + \vep^{(t+1)}\big) \nn \\
&& + \frac{\delta_2\beta^2}{2}\big\|\nabla_{\bth} \cL^{val}(\hat \bw^{(t+1)}(\bth^{(t+1)})) + \vep^{(t+1)}\big\|^2 \nn \\
& \leq & -\beta \nabla_{\bth} \cV_{ijk}(\bth^{(t+1)})^\T \nabla_{\bth} \cL^{val}(\hat \bw^{(t+1)}(\bth^{(t+1)})) - \beta \nabla_{\bth} \cV_{ijk}(\bth^{(t+1)})^\T \vep^{(t+1)} \nn \\
&& + \frac{\delta_2 \beta^2}{2} \big\|\nabla_{\bth} \cL^{val}(\hat \bw^{(t+1)}(\bth^{(t+1)})) \big\|^2 + \frac{\delta_2 \beta^2}{2} \|\vep^{(t+1)}\|^2 \nn \\
&& + \delta_2 \beta^2 \nabla_{\bth} \cL^{val}(\hat \bw^{(t+1)}(\bth^{(t+1)}))^\T \vep^{(t+1)}.
\eenrr
Then the upper bound of $I_3$ can be written as
\benrr
I_3 &\leq& -\beta \frac{1}{N^{tr}} \sum_{i=1}^{N^{tr}} \frac{1}{M} \sum_{j=1}^M \bbE_{j} \Big\{\nabla_{\bth} \cV_{ijk}(\bth^{(t+1)})^\T \nabla_{\bth} \cL^{val}(\hat \bw^{(t+1)}(\bth^{(t+1)})) \\
&& \times L_{ijk}(\bw^{(t+1)}) \Big\}\\
&& - \beta \frac{1}{N^{tr}} \sum_{i=1}^{N^{tr}} \frac{1}{M} \sum_{j=1}^M \bbE_{j} \Big\{ \nabla_{\bth} \cV_{ijk}(\bth^{(t+1)})^\T \vep^{(t+1)}L_{ijk}(\bw^{(t+1)}) \Big\}\\
&& + \Big\{\frac{\delta_2 \beta^2}{2} \big\|\nabla_{\bth} \cL^{val}(\hat \bw^{(t+1)}(\bth^{(t+1)})) \big\|^2 + \frac{\delta_2 \beta^2}{2} \|\vep^{(t+1)}\|^2\\
&& + \delta_2 \beta^2 \nabla_{\bth} \cL^{val}(\hat \bw^{(t+1)}(\bth^{(t+1)}))^\T \vep^{(t+1)} \Big\} \cL^{tr}_0(\bw^{(t+1)})\\
&=& -\beta \nabla_{\bth} \cL^{tr}(\bw^{(t+1)}, \bth^{(t+1)})^\T \nabla_{\bth} \cL^{val}(\hat \bw^{(t+1)}(\bth^{(t+1)})) \\ 
&& - \beta \nabla_{\bth} \cL^{tr}(\bw^{(t+1)}, \bth^{(t+1)})^\T \vep^{(t+1)} \\
&& + \Big\{\frac{\delta_2 \beta^2}{2} \big\|\nabla_{\bth} \cL^{val}(\hat \bw^{(t+1)}(\bth^{(t+1)})) \big\|^2 + \frac{\delta_2 \beta^2}{2} \|\vep^{(t+1)}\|^2\\
&& + \delta_2 \beta^2 \nabla_{\bth} \cL^{val}(\hat \bw^{(t+1)}(\bth^{(t+1)}))^\T \vep^{(t+1)} \Big\} \cL^{tr}_0(\bw^{(t+1)}),
\eenrr
where 
\benrr
\cL^{tr}_0(\bw^{(t+1)}) = \frac{1}{N^{tr}} \sum_{i=1}^{N^{tr}} \frac{1}{M} \sum_{j=1}^M \bbE_{j}\big[ L_{ijk}(\bw^{(t+1)})\big].
\eenrr
Next we consider $I_4.$ According to the assumptions {\it ($\bA$1)} and {\it ($\bA$3)},
\benrr
I_4 \leq \nabla_{\bw} \cL^{tr}(\bw^{(t)}, \bth^{(t+1)})^\T (\bw^{(t+1)}-\bw^{(t)}) + \frac{C_1\rho_2}{2}\|\bw^{(t+1)}-\bw^{(t)}\|^2.
\eenrr
To proceed further, we denote $\eta^{(t)} = \nabla_{\bw} \cL^{tr}_t( \bw^{(t)}, \bth^{(t+1)}) - \nabla_{\bw} \cL^{tr}( \bw^{(t)}, \bth^{(t+1)}).$ 
By the assumption {\it ($\bA$4)}, 
$
\bbE[\|\eta^{(t)}\|^2] \leq \sigma_2^2
$
for some positive scalar $\sigma_2.$ 
Note that 
\benrr
\bw^{(t+1)}-\bw^{(t)} &=& - \gamma \nabla_{\bw} \cL^{tr}_t( \bw^{(t)}, \bth^{(t+1)})  \nn \\
&=& - \gamma \big(\nabla_{\bw} \cL^{tr}( \bw^{(t)}, \bth^{(t+1)}) + \eta^{(t)}\big).
\eenrr
Then the upper bound of $I_4$ can be written as
\benrr
I_4 &\leq& -(\gamma - \frac{\gamma^2 C_1 \rho_2}{2}) \|\nabla_{\bw} \cL^{tr}(\bw^{(t)}, \bth^{(t+1)})\|^2\\
&& - (\gamma-\gamma^2 C_1 \rho_2) \nabla_{\bw} \cL^{tr}(\bw^{(t)}, \bth^{(t+1)})^\T \eta^{(t)} + \frac{\gamma^2 C_1 \rho_2}{2} \|\eta^{(t)}\|^2.
\eenrr
Combining the results of $I_3$ and $I_4$, we can obtain that
\benrr
&&(\gamma - \frac{\gamma^2 C_1 \rho_2}{2}) \|\nabla_{\bw} \cL^{tr}(\bw^{(t)}, \bth^{(t+1)})\|^2 \\
&\leq& \cL^{tr}(\bw^{(t)}, \bth^{(t+1)}) - \cL^{tr}(\bw^{(t+1)}, \bth^{(t+2)})\\
&& -\beta \nabla_{\bth} \cL^{tr}(\bw^{(t+1)}, \bth^{(t+1)})^\T \nabla_{\bth} \cL^{val}(\hat \bw^{(t+1)}(\bth^{(t+1)})) \\ 
&& - \beta \nabla_{\bth} \cL^{tr}(\bw^{(t+1)}, \bth^{(t+1)})^\T \vep^{(t+1)} \\
&& + \Big\{\frac{\delta_2 \beta^2}{2} \big\|\nabla_{\bth} \cL^{val}(\hat \bw^{(t+1)}(\bth^{(t+1)})) \big\|^2 + \frac{\delta_2 \beta^2}{2} \|\vep^{(t+1)}\|^2\\
&& + \delta_2 \beta^2 \nabla_{\bth} \cL^{val}(\hat \bw^{(t+1)}(\bth^{(t+1)}))^\T \vep^{(t+1)} \Big\} \cL^{tr}_0(\bw^{(t+1)}) \\
&& - (\gamma-\gamma^2 C_1 \rho_2) \nabla_{\bw} \cL^{tr}(\bw^{(t)}, \bth^{(t+1)})^\T \eta^{(t)} + \frac{\gamma^2 C_1 \rho_2}{2} \|\eta^{(t)}\|^2.
\eenrr
By taking the mean of $t$ from $0$ to $T-1$ and taking the expectation with respect to the mini-batch of samples, we have 
\benrr
&&\frac{1}{T}\sum_{t=0}^{T-1}(\gamma - \frac{\gamma^2 C_1 \rho_2}{2}) \bbE\big[\|\nabla_{\bw} \cL^{tr}(\bw^{(t)}, \bth^{(t+1)})\|^2 \big] \\
&\leq& \frac{1}{T}\bbE\big[\big( \cL^{tr}(\bw^{(0)},\bth^{(1)}) - \cL^{tr}(\bw^{(T)},\bth^{(T+1)}) \big)\big]\\
&& + \frac{1}{T}\sum_{t=0}^{T-1}\beta \bbE \big[\|\nabla_{\bth} \cL^{tr}(\bw^{(t+1)}, \bth^{(t+1)})\| \|\nabla_{\bth} \cL^{val}(\hat \bw^{(t+1)}(\bth^{(t+1)}))\| \big] \\
&& + \frac{1}{T}\sum_{t=0}^{T-1} \frac{\delta_2 \beta^2}{2} \bbE\big[\|\nabla_{\bth} \cL^{val}(\hat \bw^{(t+1)}(\bth^{(t+1)}))\|^2 \cL^{tr}_0(\bw^{(t+1)})\big]\\
&& + \frac{1}{T}\sum_{t=0}^{T-1} \frac{\delta_2 \beta^2}{2} \bbE\big[ \|\vep^{(t+1)}\|^2 \cL^{tr}_0(\bw^{(t+1)})\big] + \frac{1}{T}\sum_{t=0}^{T-1} \frac{\gamma^2 C_1 \rho_2}{2} \bbE \big[ \|\eta^{(t)}\|^2 \big] \\
&\leq& \frac{1}{T} \Delta_{\cL}^{tr} + \beta \delta_1 C_2 \rho'_1 + \frac{1}{2}\delta_2 \beta^2 \rho'_1 C_2 + \frac{1}{2} \delta_2 \beta^2 \sigma_1^2 C_2 + \frac{1}{2}\gamma^2 C_1 \rho_2 \sigma_2^2, 
\eenrr
where 
$\Delta_{\cL}^{tr} = \sup_{(\bw,\bth)}\cL^{tr}(\bw,\bth) - \inf_{(\bw,\bth)}\cL^{tr}(\bw,\bth).$ 
Since $\gamma - \gamma^2 C_1 \rho_2 > 0$, 
\benrr
&&\frac{1}{T}\sum_{t=0}^{T-1} \bbE \big[\|\nabla_{\bw} \cL^{tr}(\bw^{(t)}, \bth^{(t+1)})\|^2 \big] \\
& \leq & \frac{2\Delta_{\cL}^{tr}/T + 2\beta \delta_1 C_2 \rho'_1 + \delta_2 \beta^2 \rho'_1 C_2 + \delta_2 \beta^2 \sigma_1^2 C_2 + \gamma^2 C_1 \rho_2 \sigma_2^2}{ 2\gamma - \gamma^2 C_1 \rho_2} \\
& \leq & \frac{2\Delta_{\cL}^{tr}}{T \gamma} + \frac{2\beta \delta_1 C_2 \rho'_1}{\gamma} + \frac{\delta_2 \beta^2 \rho'_1 C_2}{\gamma} + \frac{\delta_2 \beta^2 \sigma_1^2 C_2}{\gamma}
+ \gamma C_1 \rho_2 \sigma_2^2\\
&=& O(\frac{1}{T \gamma}) + O(\frac{\alpha \beta}{\gamma}) + O(\frac{\alpha \beta^2}{\gamma}) + O(\frac{\beta^2}{\gamma}) + O(\gamma) \\
&=& O(\frac{\beta^2}{\gamma})  = O(\frac{\log \log T}{\log T}).
\eenrr
Hence, 
\benrr
\lim_{T \rightarrow \infty} 
\frac{1}{T}\sum_{t=0}^{T-1} \bbE \big[\|\nabla_{\bw} \cL^{tr}(\bw^{(t)}, \bth^{(t+1)})\|^2 \big] = 0.
\eenrr
%\bbox
\end{proof}

\subsection{Proof of Theorem~\ref{theorem2}}\label{proof_thm2}

\begin{theorem}
Suppose the following assumptions hold: 
\vs 0.05cm
\noi ($\bA$1) The loss function $\ell$ have $\rho_1$-bounded gradients with respect to $\bw$ under both (augmented) training data and validation data, and the loss function $\ell$ is Lipschitz smooth with constant $\rho_2$;\\
%\vs 0.05cm
\noi ($\bA$2) The policy network $\cV$ is differential with a $\delta_1$-bounded gradient and twice differential with its Hessian bounded by $\delta_2$ with respect to $\bth$; \\
\noi ($\bA$2')  Further assume that the policy network $\cV$ depends on $\bw$ and is differential with a $\tilde \delta_1$-bounded gradient 
with respect to $\bw$; \\
%\vs 0.1cm
\noi ($\bA$3) The absolute values of the policy network $\cV$ and the loss function $\ell$ are bounded above by $C_1$ and $C_2$, respectively; \\
%\vs 0.05cm
\noi ($\bA$4) For any iteration $0 \leq t \leq T-1$, the variance of the weighted training loss (validation loss) gradient on a mini-batch of training (validation) samples is bounded above; \\
%\vs 0.05cm
\noi ($\bA$5) Let
\benrr
\alpha = \frac{c \log T}{T}, \quad \beta = \sqrt{\frac{c' \log\log T}{T}}, \quad \gamma =  \frac{c''\log T}{T},
\eenrr
for some positive constants $c$, $c'$ and $c''$; \\
%\vs 0.05cm
\noi ($\bA$6) The number of iterations $T$ is sufficiently large such that $\alpha \beta \rho_1^2(\alpha \delta_1^2 \rho_2 + \delta_2)<1$ and $\gamma C_1 \rho_2 <1.$

Then we have
\benrr
&& \frac{1}{T} \sum_{t=0}^{T-1}\bbE \Big[  \big\|\nabla_{\bth} \cL^{val}(\hat \bw^{(t)}(\bth^{(t)}))\big\|^2 \Big] \leq O(\frac{\log T}{\sqrt{T\log\log T}}), \\
&&  
\frac{1}{T}\sum_{t=0}^{T-1} \bbE \Big[\|\nabla_{\bw} \cL^{tr}(\bw^{(t)}, \bth^{(t+1)})\|^2 \Big] - 2\rho_1 {\tilde \delta_1} C_1 C_2 \leq o(1).
\eenrr
\end{theorem}
%}
%\vs 0.2cm

%{\bf Proof}:
\begin{proof}
we denote 
\benrr
\cV_{ijk}(\bw,\bth) = \cV\big(f(op_j(x_i;\xi_k); \bw), e(op_j(\cdot\,;\xi_k));\bth\big),
\eenrr
where $f(\cdot\,;\bw)$ is the feature extractor of the task network and $e(op_j(\cdot\,;\xi_k))$ is the embedding of the augmentation operator $op_j(\cdot\,;\xi_k)$. 
To proceed further, we let
\benrr
\cL^{tr}(\bw,(\bw',\bth)) &=& \frac{1}{N^{tr}} \sum_{i=1}^{N^{tr}} \frac{1}{M} \sum_{j=1}^M 
\bbE_{j}\Big[\cV_{ijk}(\bw',\bth)L_{ijk}(\bw)\Big], \\
\cL_{t}^{tr}(\bw,(\bw',\bth)) &=& \frac{1}{n^{tr}} \sum_{i=1}^{n^{tr}} \frac{1}{m^{op}} \sum_{j=1}^{m^{op}}\frac{1}{m^{\xi}}\sum_{k=1}^{m^{\xi}}\cV_{ijk}(\bw',\bth) L_{ijk}(\bw).
\eenrr
For the case that $\cV$ depends on $\bw$, the arguments about $I_1$, $I_2$, and $I_3$ are similar to those of Theorem~\ref{theorem1}. 
We decompose the term $I_4$ into $I_{41} + I_{42}$, where
\benrr
I_{41} &=& \cL^{tr}(\bw^{(t+1)},(\bw^{(t+1)},\bth^{(t+1)})) - \cL^{tr}(\bw^{(t+1)},(\bw^{(t)},\bth^{(t+1)})), \\
I_{42} &=& \cL^{tr}(\bw^{(t+1)},(\bw^{(t)},\bth^{(t+1)})) - \cL^{tr}(\bw^{(t)},(\bw^{(t)},\bth^{(t+1)})).
\eenrr
The argument about $I_{42}$ is similar to that of $I_4$ in Theorem~\ref{theorem1}. 
We only deal with the term $I_{41}$, which represents the discontinuous change of the training loss between two iterations. 
Plugging the expression of $\cL^{tr}(\bw,(\bw',\bth))$ into $I_{41}$, we have
\benrr
I_{41} =\frac{1}{N^{tr}} \sum_{i=1}^{N^{tr}} \frac{1}{M} \sum_{j=1}^M 
\bbE_{j}\Big[(\cV_{ijk}(\bw^{(t+1)},\bth^{(t+1)}) - \cV_{ijk}(\bw^{(t)},\bth^{(t+1)}))L_{ijk}(\bw^{(t+1)})\Big].
\eenrr

To proceed further, we denote 
\benrr
\nabla_{\bw} \cL^{tr}(\bw^{(t)},(\bw^{(t)},\bth^{(t+1)})) &=& \nabla_{\bw} \cL^{tr}(\bw,(\bw^{(t)},\bth^{(t+1)}))\Big|_{\bw=\bw^{(t)}}, \\
\nabla_{\bw'} \cL^{tr}(\bw^{(t)},(\bw^{(t)},\bth^{(t+1)})) &=& \nabla_{\bw} \cL^{tr}(\bw^{(t)},(\bw,\bth^{(t+1)}))\Big|_{\bw=\bw^{(t)}}.
\eenrr
By the mean value theorem, there exists $\bw^{(t)*} = \bw^{(t)} + c (\bw^{(t+1)} - \bw^{(t)})$ with $0<c<1$ such that
\benrr
I_{41} = \frac{1}{N^{tr}} \sum_{i=1}^{N^{tr}} \frac{1}{M} \sum_{j=1}^M 
\bbE_{j}\Big[\nabla_\bw \cV_{ijk}(\bw^{(t)*},\bth^{(t+1)})^\T(\bw^{(t+1)} -\bw^{(t)})L_{ijk}(\bw^{(t+1)})\Big].
\eenrr
Note that 
\benrr
\bw^{(t+1)}-\bw^{(t)} = -\gamma (\nabla_{\bw} \cL^{tr}(\bw^{(t)},(\bw^{(t)},\bth^{(t+1)})) + \eta^{(t)}).
\eenrr
Hence we have
\benrr
I_{41} &=& -\gamma \frac{1}{N^{tr}} \sum_{i=1}^{N^{tr}} \frac{1}{M} \sum_{j=1}^M 
\bbE_{j}\Big[\nabla_\bw \cV_{ijk}(\bw^{(t)*},\bth^{(t+1)})^\T L_{ijk}(\bw^{(t+1)})\Big] \\
&& \times \big(\nabla_{\bw} \cL^{tr}(\bw^{(t)},(\bw^{(t)},\bth^{(t+1)})) + \eta^{(t)}\big) \\
&=& - \gamma \nabla_{\bw'} \cL^{tr}(\bw^{(t+1)},(\bw^{(t)*},\bth^{(t+1)}))^\T \nabla_{\bw} \cL^{tr}(\bw^{(t)},(\bw^{(t)},\bth^{(t+1)})) \\
&& - \gamma \nabla_{\bw'} \cL^{tr}(\bw^{(t+1)},(\bw^{(t)*},\bth^{(t+1)}))^\T \eta^{(t)}.
\eenrr
Combining the decomposition of $I_{41}$ and the upper bound of $I_{42}$, we know
\benrr
I_4 &\leq& - \gamma \nabla_{\bw'} \cL^{tr}(\bw^{(t+1)},(\bw^{(t)*},\bth^{(t+1)}))^\T \nabla_{\bw} \cL^{tr}(\bw^{(t)},(\bw^{(t)},\bth^{(t+1)})) \\
&& - \gamma \nabla_{\bw'} \cL^{tr}(\bw^{(t+1)},(\bw^{(t)*},\bth^{(t+1)}))^\T \eta^{(t)} \\
&& -(\gamma - \frac{\gamma^2 C_1 \rho_2}{2}) \|\nabla_{\bw} \cL^{tr}(\bw^{(t)}, (\bw^{(t)},\bth^{(t+1)}))\|^2 \\
&& - (\gamma-\gamma^2 C_1 \rho_2) \nabla_{\bw} \cL^{tr}(\bw^{(t)}, (\bw^{(t)},\bth^{(t+1)}))^\T \eta^{(t)} + \frac{\gamma^2 C_1 \rho_2}{2} \|\eta^{(t)}\|^2.
\eenrr
Further, combining the upper bounds of $I_3$ and $I_4$, we obtain that
\benrr
&&(\gamma - \frac{\gamma^2 C_1 \rho_2}{2}) \|\nabla_{\bw} \cL^{tr}(\bw^{(t)}, (\bw^{(t)},\bth^{(t+1)}))\|^2 \\
&\leq& \cL^{tr}(\bw^{(t)}, (\bw^{(t)},\bth^{(t+1)})) - \cL^{tr}(\bw^{(t+1)}, (\bw^{(t+1)},\bth^{(t+2)}))\\
&& -\beta \nabla_{\bth} \cL^{tr}(\bw^{(t+1)}, (\bw^{(t+1)},\bth^{(t+1)}))^\T \nabla_{\bth} \cL^{val}(\hat \bw^{(t+1)}(\bth^{(t+1)})) \\ 
&& - \beta \nabla_{\bth} \cL^{tr}(\bw^{(t+1)}, (\bw^{(t+1)},\bth^{(t+1)}))^\T \vep^{(t+1)} \\
&& + \Big\{\frac{\delta_2 \beta^2}{2} \big\|\nabla_{\bth} \cL^{val}(\hat \bw^{(t+1)}(\bth^{(t+1)})) \big\|^2 + \frac{\delta_2 \beta^2}{2} \|\vep^{(t+1)}\|^2\\
&& + \delta_2 \beta^2 \nabla_{\bth} \cL^{val}(\hat \bw^{(t+1)}(\bth^{(t+1)}))^\T \vep^{(t+1)} \Big\} \cL^{tr}_0(\bw^{(t+1)}) \\
&&  - \gamma \nabla_{\bw'} \cL^{tr}(\bw^{(t+1)},(\bw^{(t)*},\bth^{(t+1)}))^\T \nabla_{\bw} \cL^{tr}(\bw^{(t)},(\bw^{(t)},\bth^{(t+1)})) \\
&& - \gamma \nabla_{\bw'} \cL^{tr}(\bw^{(t+1)},(\bw^{(t)*},\bth^{(t+1)}))^\T \eta^{(t)} \\
&& - (\gamma-\gamma^2 C_1 \rho_2) \nabla_{\bw} \cL^{tr}(\bw^{(t)}, (\bw^{(t)},\bth^{(t+1)}))^\T \eta^{(t)} + \frac{\gamma^2 C_1 \rho_2}{2} \|\eta^{(t)}\|^2.
\eenrr
By taking the mean of $t$ from $0$ to $T-1$ and taking the expectation with respect to the mini-batch of samples, we have 
\benr\label{A7}
&&\frac{1}{T}\sum_{t=0}^{T-1} (\gamma - \frac{\gamma^2 C_1 \rho_2}{2}) \bbE\big[\|\nabla_{\bw} \cL^{tr}(\bw^{(t)}, (\bw^{(t)},\bth^{(t+1)}))\|^2\big] \nn \\
&\leq& \frac{1}{T}\bbE \big[\big( \cL^{tr}(\bw^{(0)},(\bw^{(0)},\bth^{(1)})) - \cL^{tr}(\bw^{(T)}, (\bw^{(T)},\bth^{(T+1)})) \big)\big] \nn \\
&& - \frac{1}{T}\sum_{t=0}^{T-1}\beta \bbE \big[\nabla_{\bth} \cL^{tr}(\bw^{(t+1)}, (\bw^{(t+1)},\bth^{(t+1)}))^\T \nabla_{\bth} \cL^{val}(\hat \bw^{(t+1)}(\bth^{(t+1)})) \big] \nn \\
&& + \frac{1}{T}\sum_{t=0}^{T-1} \frac{\delta_2 \beta^2}{2} \bbE\big[\|\nabla_{\bth} \cL^{val}(\hat \bw^{(t+1)}(\bth^{(t+1)}))\|^2 \cL^{tr}_0(\bw^{(t+1)})\big] \nn \\
&& + \frac{1}{T}\sum_{t=0}^{T-1} \frac{\delta_2 \beta^2}{2} \bbE\big[ \|\vep^{(t+1)}\|^2 \cL^{tr}_0(\bw^{(t+1)})\big] + \frac{1}{T}\sum_{t=0}^{T-1} \frac{\gamma^2 C_1 \rho_2}{2} \bbE \big[ \|\eta^{(t)}\|^2 \big] \nn \\
&& - \frac{1}{T}\sum_{t=0}^{T-1} \gamma \bbE \big[\nabla_{\bw'} \cL^{tr}(\bw^{(t+1)},(\bw^{(t)*},\bth^{(t+1)}))^\T \nabla_{\bw} \cL^{tr}(\bw^{(t)},(\bw^{(t)},\bth^{(t+1)}))\big].
\eenr
Then,
\benrr
&&\frac{1}{T}\sum_{t=0}^{T-1} (\gamma - \frac{\gamma^2 C_1 \rho_2}{2}) \bbE\big[\|\nabla_{\bw} \cL^{tr}(\bw^{(t)}, (\bw^{(t)},\bth^{(t+1)}))\|^2\big] \\
&\leq& \frac{1}{T} \Delta_{\cL}^{tr} + \beta \rho'_1 \delta_1 C_2  + \frac{1}{2} \beta^2 \rho'_1 \delta_2 C_2 + \frac{1}{2} \beta^2 \sigma_1^2  \delta_2 C_2 + \frac{1}{2}\gamma^2 \sigma_2^2 \rho_2 C_1  + \gamma \rho_1 \tilde \delta_1 C_1 C_2.
\eenrr
Further,
\benrr
&&\frac{1}{T}\sum_{t=0}^{T-1} \bbE \big[\|\nabla_{\bw} \cL^{tr}(\bw^{(t)}, \bth^{(t+1)})\|^2 \big] = \frac{1}{T}\sum_{t=0}^{T-1} \bbE \big[\|\nabla_{\bw} \cL^{tr}(\bw^{(t)}, (\bw^{(t)},\bth^{(t+1)}))\|^2 \big] \\
& \leq & \frac{2\Delta_{\cL}^{tr}}{T \gamma} + \frac{2\beta \rho'_1 \delta_1 C_2 }{\gamma} + \frac{\beta^2 \rho'_1 \delta_2 C_2}{\gamma} + \frac{\beta^2 \sigma_1^2 \delta_2  C_2}{\gamma}
+ \gamma C_1 \rho_2 \sigma_2^2 + 2 \rho_1 \tilde \delta_1 C_1 C_2\\
&=& O(\frac{1}{T \gamma}) + O(\frac{\alpha \beta}{\gamma}) + O(\frac{\alpha \beta^2}{\gamma}) + O(\frac{\beta^2}{\gamma}) + O(\gamma) + 2 \rho_1 \tilde \delta_1 C_1 C_2\\
&=& O(\frac{\beta^2}{\gamma}) + 2 \rho_1 \tilde \delta_1 C_1 C_2.
\eenrr
Hence, 
\benrr
\frac{1}{T}\sum_{t=0}^{T-1} \bbE \big[\|\nabla_{\bw} \cL^{tr}(\bw^{(t)}, \bth^{(t+1)})\|^2 \big] - 2 \rho_1 \tilde \delta_1 C_1 C_2 \leq o(1).
\eenrr
The proof of Theorem~\ref{theorem2} is finished.
Note that, if 
\benrr
\frac{1}{T}\sum_{t=0}^{T-1} \gamma \bbE \big[\nabla_{\bw'} \cL^{tr}(\bw^{(t+1)},(\bw^{(t)*},\bth^{(t+1)}))^\T \nabla_{\bw} \cL^{tr}(\bw^{(t)},(\bw^{(t)},\bth^{(t+1)}))\big] > 0,
\eenrr
the inequality~\eqref{A7} implies that
\benrr
&&\frac{1}{T}\sum_{t=0}^{T-1} (\gamma - \frac{\gamma^2 C_1 \rho_2}{2}) \bbE\big[\|\nabla_{\bw} \cL^{tr}(\bw^{(t)}, (\bw^{(t)},\bth^{(t+1)}))\|^2\big] \\
&\leq& \frac{1}{T} \Delta_{\cL}^{tr} + \beta \rho'_1 \delta_1 C_2  + \frac{1}{2} \beta^2 \rho'_1 \delta_2 C_2 + \frac{1}{2} \beta^2 \sigma_1^2  \delta_2 C_2 + \frac{1}{2}\gamma^2 \sigma_2^2 \rho_2 C_1.
\eenrr
Then, 
\benrr
\frac{1}{T}\sum_{t=0}^{T-1} \bbE\big[\|\nabla_{\bw} \cL^{tr}(\bw^{(t)}, (\bw^{(t)},\bth^{(t+1)}))\|^2\big] \leq o(1).
\eenrr
This implies that under certain conditions, the convergence results of Theorem~\ref{theorem1} still hold though the policy network $\cV$ depends on $\bw$.
\end{proof}
%\bbox

\end{document}